\documentclass{article}

    \PassOptionsToPackage{numbers, compress}{natbib}

     \usepackage[preprint]{neurips_2020}

\usepackage[utf8]{inputenc} 
\usepackage[T1]{fontenc}    
\usepackage{hyperref}       
\usepackage{url}            
\usepackage{booktabs}       
\usepackage{amsfonts}       
\usepackage{nicefrac}       
\usepackage{microtype}      
\usepackage{mathtools}
\usepackage{amsthm}
\usepackage{algorithm}
\usepackage{algpseudocode}

\newif\ifconf
\newif\ifproofs
\newif\ifmatexp

\newtheorem{theorem}{Theorem}
\newtheorem{lemma}{Lemma}

\newtheorem{prop}[theorem]{Proposition}

\DeclareMathOperator{\tr}{tr}
\DeclareMathOperator{\sign}{sign}
\DeclareMathOperator*{\argmin}{arg\,min}
\DeclarePairedDelimiter{\abs}{\lvert}{\rvert}
\DeclarePairedDelimiter{\norm}{\lVert}{\rVert}

\newcommand\given[1][]{\:#1\vert\:}

\newcommand{\EE}[1]{\mathbb{E}\left[#1\right]}
\newcommand{\cG}{\mathcal{G}}
\newcommand{\cV}{\mathcal{V}}
\newcommand{\cE}{\mathcal{E}}
\newcommand{\cP}{\mathcal{P}}
\newcommand{\cZ}{\mathcal{Z}}
\newcommand{\cX}{\mathcal{X}}
\newcommand{\cA}{\mathcal{A}}

\newcommand{\bX}{\mathbf{X}}
\newcommand{\ones}{\boldsymbol{1}}

\title{DAGs with No Fears: A Closer Look at Continuous Optimization for Learning Bayesian Networks}

\author{%
  Dennis Wei\\
  IBM Research\\
  \texttt{dwei@us.ibm.com} \\
  \And
  Tian Gao\\
  IBM Research\\
  \texttt{tgao@us.ibm.com} \\
  \And
  Yue Yu\\
  Lehigh University\\
  \texttt{yuy214@lehigh.edu} \\
}

\begin{document}

\maketitle

\begin{abstract}
  This paper re-examines a continuous optimization framework dubbed NOTEARS for learning Bayesian networks. We first generalize existing algebraic characterizations of acyclicity to a class of matrix polynomials. Next, focusing on a one-parameter-per-edge setting, it is shown that the Karush-Kuhn-Tucker (KKT) optimality conditions for the NOTEARS formulation cannot be satisfied except in a trivial case, which explains a behavior of the associated algorithm. We then derive the KKT conditions for an equivalent reformulation, show that they are indeed necessary, and relate them to explicit constraints that certain edges be absent from the graph. If the score function is convex, these KKT conditions are also sufficient for local minimality despite the non-convexity of the constraint. Informed by the KKT conditions, a local search post-processing algorithm is proposed and shown to substantially and universally improve the structural Hamming distance of all tested algorithms, typically by a factor of 2 or more. Some combinations with local search are both more accurate and more efficient than the original NOTEARS.
\end{abstract}

\section{Introduction}
\label{sec:intro}

Bayesian networks are directed probabilistic graphical models used to model joint probability distributions of data in many applications \cite{ott2004finding,glymour1999computation}. Automatic discovery of their directed acyclic graph (DAG) structure is important to research areas from causal inference to biology. However, DAG structure learning is in general an NP-hard problem \cite{chickering12}. Many learning algorithms have been proposed to circumvent exhaustive search in the discrete space of DAGs, including those for discrete variables \cite{chickering02,acid2005learning, silander06,  jaakkola2010, cussens11, Yuan13learning, gao2018parallel} and continuous variables \cite{buhlmann2014cam,MMPCcor}.  

Recently, \citet{zheng2018dags} proposed a \emph{continuous} optimization formulation, referred to as NOTEARS, in which acyclicity of the graph is enforced by a trace of matrix exponential constraint on a weighted adjacency matrix. Several works have since successfully extended the formulation to nonlinear and nonparametric models \cite{yu2019dag,lachapelle2020gradient,Kalainathan2018,zheng2020learning}.

This paper takes further steps toward fulfilling the promise of \cite{zheng2018dags} in opening the door to 
continuous optimization techniques for score-based structure learning. We contribute in particular to theoretical understanding of this framework, leading to significant algorithmic improvements. 

First, in Section~\ref{sec:chars}, the matrix exponential constraint of \cite{zheng2018dags} and the matrix polynomial constraint of \cite{yu2019dag}
are generalized to a class of matrix polynomials with positive coefficients whose traces characterize acyclicity. We also provide a characterization involving the gradient of functions in this class, which is not only essential to proving later results but also has an intuitive graphical interpretation.

In Section~\ref{sec:theory:sq}, we revisit the NOTEARS formulation of \cite{zheng2018dags} in which a weighted adjacency matrix is obtained by element-wise squaring of the parameter matrix. It is shown that the Karush-Kuhn-Tucker (KKT) optimality conditions for this constrained optimization cannot be satisfied except in a trivial case. This negative result is somewhat surprising given the empirical success of the augmented Lagrangian algorithm of \cite{zheng2018dags}, and we use the result to explain why the algorithm does not converge to an exactly acyclic solution even when the penalty parameters are very high.

In Section~\ref{sec:theory:abs}, we consider an equivalent reformulation in which the adjacency matrix is given by 
the absolute value of the parameter matrix, motivated in part by the failure to satisfy KKT conditions in Section~\ref{sec:theory:sq}, and 
in part by the connection between the $\ell_1$ norm and sparsity. We show that the KKT conditions for this reformulation are indeed necessary conditions of optimality, i.e.~they are satisfied by all local minima, although even here common constraint qualification methods turn out to fail. If the score function is convex, then the KKT conditions are also sufficient for local minimality, despite the non-convexity of the constraint. We then relate the KKT conditions to the optimality conditions for score optimization subject to explicit edge absence constraints. The KKT conditions can thus be understood through edge absences: together these must be sufficient to ensure acyclicity, but each absence must also be necessary in preventing the completion of 
a cycle. 

The theoretical development of Section~\ref{sec:theory:abs} naturally suggests two algorithms: an augmented Lagrangian algorithm as in \cite{zheng2018dags} with an absolute value adjacency matrix instead of quadratic, and a local search algorithm, KKTS, informed by the KKT conditions and proven to satisfy them. KKTS (a) adds edge absence constraints to break cycles, (b) removes constraints that are unnecessary, and (c) swaps constraints (reverses edges) to combat non-convexity.
We find in Section~\ref{sec:expt} that neither of these two algorithms yields state-of-the-art accuracy by itself. However, 
when combined with other algorithms, KKTS substantially reduces structural Hamming distance (SHD) with respect to the true graph, typically by a factor of at least 2. Moreover, this improvement is consistent across dimensions and base algorithms. In the case of NOTEARS, new state-of-the-art accuracy is obtained, while other combinations can outperform NOTEARS and take less time.

\paragraph{More on related work} Bayesian network structure learning has long been an active research area. Constraint- and score-based methods utilize independence tests and graph scores respectively to learn the DAG structure. Optimization methods such as greedy search \cite{chickering02}, dynamic programming \cite{koivisto2004exact}, branch and bound \cite{campos2009}, A* search \cite{Yuan13learning,xiang2013A*}, local-to-global search 
\cite{gao2017local} as well as approximation methods \cite{scanagatta2015learning} have all been proposed.  As mentioned, this paper is most closely related to  
the continuous framework of \cite{zheng2018dags} and subsequent works \cite{yu2019dag,zheng2020learning}. Regression-based methods for DAG learning, without the matrix exponential constraint, have also been carefully studied  \cite{L1MB,buhlmann2014cam,aragam2015concave,gu2019penalized}.

\section{Characterizations of acyclicity}
\label{sec:chars}

In this first section, we provide algebraic characterizations of acyclicity for a directed graph in terms of its adjacency matrix. For a directed graph $\cG = (\cV, \cE)$ with vertices $\cV = \{1,\dots,d\}$ and directed edges $(i,j) \in \cE$, a non-negative matrix $A$ is a (weighted) adjacency matrix for $\cG$ if $A_{ij} > 0$ for $(i,j) \in \cE$ and $A_{ij} = 0$ otherwise. 

We consider a class of functions $h(A)$ corresponding to matrix polynomials of degree $d$ with positive coefficients,
\[
P(A) = c_0 I + c_1 A + \dots + c_d A^d = \sum_{p=0}^d c_p A^p, \qquad c_p > 0, \quad p = 1,\dots,d,
\]
from which we define 
\begin{equation}\label{eqn:h}
h(A) = \tr(P(A)) - c_0 d = \sum_{p=1}^d c_p \tr(A^p).
\end{equation}
This class includes the function $h(A) = \tr\bigl((I + A/d)^d\bigr) - d$ from \cite{yu2019dag}, which corresponds to $c_p = \binom{d}{p} / d^p$, and the trace of matrix exponential from \cite{zheng2018dags},
\begin{equation}\label{eqn:hExp}
h(A) = \tr(e^A) - d = \sum_{p=1}^{\infty} \frac{\tr(A^p)}{p!}.
\end{equation}
Although \eqref{eqn:hExp} appears to be an infinite power series, it can be rewritten as a finite series with no powers higher than $d$ using the Cayley-Hamilton theorem \cite{horn_johnson_2012}, which equates $A^d$ to a linear combination of $I, A, \dots, A^{d-1}$, and similarly for all higher powers of $A$.

Any function $h(A)$ of the form 
in \eqref{eqn:h} can characterize acyclicity, as stated below.
We defer all proofs to Appendix~\ref{sec:proofs}.
\begin{theorem}\label{thm:h}
A directed graph $\cG$ is acyclic if and only if its (weighted) adjacency matrix satisfies $h(A) = 0$ for any $h$ defined by \eqref{eqn:h}.
\end{theorem}
The proof of Theorem~\ref{thm:h} is facilitated by Lemma~\ref{lem:nilpotent} below. We recall that a matrix $B$ is said to be \emph{nilpotent} if $B^p = 0$ for some power $p \in \mathbb{N}$ 
(and consequently all higher powers). Equivalent characterizations are that all eigenvalues of $B$ are zero, and most usefully here, that $\tr(B^p) = 0$ for all $p \in \mathbb{N}$. We call attention to the lemma as there may be independent interest in alternative ways of enforcing nilpotency.
\begin{lemma}\label{lem:nilpotent}
A directed graph $\cG$ is acyclic if and only if its (weighted) adjacency matrix $A$ is nilpotent.
\end{lemma}

The gradient of $h(A)$ in \eqref{eqn:h} is a matrix-valued function given by 
\begin{equation}\label{eqn:gradh}
    \nabla h(A) = \sum_{p=1}^d p c_p \left(A^{p-1}\right)^T.
\end{equation}
We make the following elementary observation for later reference.
\begin{lemma}\label{lem:gradhNonNeg}
For non-negative matrices $A$, $\nabla h(A)$ is non-negative and $h(A)$ is therefore a non-decreasing function in the sense that $h(A) \geq h(B)$ if $A - B \geq 0$.
\end{lemma}
Off-diagonal elements $(\nabla h(A))_{ij}$ have an intuitive interpretation 
in terms of \emph{directed walks} from $j$ to $i$, i.e.~a sequence of edges $(j, i_1), (i_1, i_2), \dots, (i_{l-1}, i) \in \cE$. If there is a directed walk from $j$ to $i$, then there is also a \emph{directed path}, i.e.~a directed walk in which all vertices $j, i_1, \dots, i_{l-1}, i$ are distinct \cite{bondy_murty2008}.
\begin{lemma}\label{lem:gradh_ij}
For any $h(A)$ defined by \eqref{eqn:h} and $i \neq j$, $(\nabla h(A))_{ij} > 0$ if and only if there exists a directed walk from $j$ to $i$ in $\cG$.
\end{lemma}

The gradient $\nabla h(A)$ can also be used to characterize acyclicity, which will prove useful in the sequel. 
\begin{lemma}\label{lem:Agradh}
A directed graph $\cG$ is acyclic if and only if the Hadamard product $A \circ \nabla h(A) = 0$ for any $h$ defined by \eqref{eqn:h}.
\end{lemma}

With the help of Lemma~\ref{lem:gradh_ij}, we can give a simple graphical interpretation of Lemma~\ref{lem:Agradh}: If a directed graph is acyclic, then for every pair $(i,j)$, we must either not have an edge from $i$ to $j$, i.e.~$A_{ij} = 0$, or not have a return path from $j$ to $i$, i.e.~$(\nabla h(A))_{ij} = 0$.

\section{Analysis of continuous acyclicity-constrained optimization}
\label{sec:theory}

In the remainder of the paper, we address the problem of learning a Bayesian network (a probabilistic directed graphical model) for the joint distribution of a $d$-dimensional random vector $X$, given a data matrix of $n$ samples $\bX \in \mathbb{R}^{n\times d}$. We assume that the Bayesian network is parametrized by a matrix $W \in \mathbb{R}^{d\times d}$ 
such that the
sparsity pattern of $W$ corresponds to the adjacency pattern of the graph: $W_{ij} \neq 0$ if and only if $(i,j) \in \cE$. In other words, each edge is associated with a single parameter $W_{ij}$. The most straightforward instance of this setting is a linear structural equation model (SEM) given by $X_j = W_{\cdot j}^T X + z_j$, where $W_{\cdot j}$ is the $j$th column of $W$ and $z_j$ is random noise. More general models such as generalized linear models $\EE{X_j \given X} = g\left(W_{\cdot j}^T X\right)$ are also included. While we experiment only with continuous variables in Section~\ref{sec:expt}, it is straightforward to accommodate binary variables as well: in a generalized linear structural equation, a single parameter $W_{ij}$ can account for the effect of a binary input variable $X_i$, while a suitable link function $g$ (e.g.~logistic) can be used for a binary output $X_j$.

This section analyzes the continuous optimization problem of minimizing a score function $F(W)$ subject to the acyclicity constraint $h(A) = 0$ for any $h$ defined by \eqref{eqn:h} (thanks to Theorem~\ref{thm:h}). For simplicity, it is assumed in this section that $F(W)$ is continuously differentiable, although it is not hard to extend the analysis to account for an $\ell_1$ penalty as in \eqref{eqn:scoreL1}. We consider two ways of defining a weighted adjacency matrix $A$ from $W$. Section~\ref{sec:theory:sq} re-examines the quadratic case $A = W \circ W$ proposed in \cite{zheng2018dags} and sheds light on their augmented Lagrangian algorithm. We then propose and  
study the absolute value case $A = \abs{W}$ in Section~\ref{sec:theory:abs}.

\subsection{Quadratic adjacency matrix}
\label{sec:theory:sq}

With $A = W \circ W$ as the element-wise square of $W$, the optimization problem is  
\begin{equation}\label{eqn:probSq}
    \min_W \quad F(W) \quad \text{s.t.} \quad h(W \circ W) \leq 0.
\end{equation}
The constraint $h(W \circ W) \leq 0$ is equivalent to $h(W \circ W) = 0$ because $h(A) \geq 0$ for non-negative $A$, as seen from \eqref{eqn:h}. The matrix exponential case of \eqref{eqn:probSq} with $h(A)$ as in \eqref{eqn:hExp} was proposed in \cite{zheng2018dags}.

Applying Lemma~\ref{lem:Agradh} yields the following consequence.
\begin{lemma}\label{lem:zeroGrad}
Let $W$ be a feasible solution to problem \eqref{eqn:probSq}. Then $\nabla_W (h(W \circ W)) = 0$.
\end{lemma}

The vanishing gradient in Lemma~\ref{lem:zeroGrad} has theoretical and practical implications. First, the Karush-Kuhn-Tucker (KKT) conditions of optimality \cite{bertsekas1999} for problem \eqref{eqn:probSq}, namely
\begin{equation}\label{eqn:KKTsq}
\nabla F(W) + \lambda \nabla_W (h(W \circ W)) = 0
\end{equation}
with Lagrange multiplier $\lambda \geq 0$, are not satisfied for any feasible solution, let alone a local minimum, except in a trivial case.
\begin{prop}\label{prop:noKKTsq}
Let $W$ be a feasible solution to problem \eqref{eqn:probSq}. Then unless $W$ is an unconstrained stationary point of $F(W)$, i.e.~$\nabla F(W) = 0$, the KKT condition \eqref{eqn:KKTsq} cannot hold.
\end{prop}
In particular if $F(W)$ is convex, the condition $\nabla F(W) = 0$ holds only for unconstrained minimizers of $F(W)$, so if these solutions are already acyclic, there is nothing more to be done.

On the practical side, Lemma~\ref{lem:zeroGrad} sheds light on the augmented Lagrangian algorithm proposed in \cite{zheng2018dags}. The augmented Lagrangian corresponding to \eqref{eqn:probSq} with penalty parameters $\alpha$ and $\rho$ is 
\begin{equation}\label{eqn:ALsq}
F(W) + \alpha h(W \circ W) + \frac{\rho}{2} h(W \circ W)^2,
\end{equation}
with gradient 
\[
\nabla F(W) + (\alpha + \rho h(W \circ W)) \nabla_W (h(W \circ W)).
\]
\begin{prop}\label{prop:ALsq}
Let $W$ be a feasible solution to problem \eqref{eqn:probSq}. Then unless $W$ is an unconstrained stationary point of $F(W)$, 
i.e.~$\nabla F(W) = 0$, 
$W$ cannot be a stationary point of the augmented Lagrangian \eqref{eqn:ALsq}.
\end{prop}

Proposition~\ref{prop:ALsq} explains the following observed behavior of the augmented Lagrangian algorithm, namely that it does not converge to an exactly (or within machine precision) feasible solution of \eqref{eqn:probSq} even when the penalty parameters $\alpha$, $\rho$ are very high ($\rho \sim 10^{16}$). The reason is that a minimizer of the augmented Lagrangian \eqref{eqn:ALsq} cannot be a feasible solution to \eqref{eqn:probSq} except in the trivial case discussed above. However, when $\alpha$ and $\rho$ are very large, minimizers of \eqref{eqn:ALsq} do tend to have gradients $\nabla_W (h(W \circ W)) \approx 0$, and accordingly $h(W \circ W) \approx 0$ by continuity. Thus as $\alpha$ and $\rho$ increase, the augmented Lagrangian algorithm yields solutions that are closer and closer to being feasible.

\subsection{Absolute value adjacency matrix}
\label{sec:theory:abs}

As an alternative, we 
consider defining adjacency matrix $A$ as the absolute value of $W$, $A = \abs{W}$, leading to the following constrained optimization:
\begin{equation}\label{eqn:probOrig}
\min_W \quad F(W) \quad \text{s.t.} \quad h(\abs{W}) \leq 0.
\end{equation}
Formulation~\eqref{eqn:probOrig} is motivated in part by the failure to satisfy KKT conditions in Section~\ref{sec:theory:sq} and in part by the connection between the absolute value function/$\ell_1$ norm and sparsity, which is needed for acyclicity. While it will be seen that \eqref{eqn:probOrig} has different theoretical and numerical properties from \eqref{eqn:probSq}, the two formulations are equivalent in a sense because acyclicity depends only on the sparsity pattern of $W$, which is clearly the same regardless of whether $\abs{W}$ or $W \circ W$ is used.

\subsubsection{An equivalent smooth optimization}

Problem~\eqref{eqn:probOrig} is not a smooth optimization because of the absolute value function. To avoid any issues with continuous differentiability, we make use of the following alternative formulation, which we show in Appendix~\ref{sec:proofs:equiv} to be equivalent to \eqref{eqn:probOrig}:
\begin{equation}\label{eqn:prob+-}
\min_{W^+, W^-} \; F\left(W^+ - W^-\right) \quad \text{s.t.} \quad h\left(W^+ + W^-\right) \leq 0, \quad W^+, W^- \geq 0.
\end{equation}
Given any solution $(W^+, W^-)$ to \eqref{eqn:prob+-}, a solution to \eqref{eqn:probOrig} is obtained simply as $W = W^+ - W^-$.

\subsubsection{KKT conditions and constraint qualification}
\label{sec:theory:abs:KKT}

We proceed to analyze the KKT conditions for the smooth reformulation \eqref{eqn:prob+-}, which are as follows: 
\begin{subequations}\label{eqn:KKT}
\begin{align}
\pm \nabla F\left(W^+ - W^-\right) + \lambda \nabla h\left(W^+ + W^-\right) &= M^\pm \geq 0\label{eqn:KKTgrad}\\
W^\pm \circ M^\pm &= 0,\label{eqn:KKTcomplementary}
\end{align}
\end{subequations}
in addition to the feasibility conditions in \eqref{eqn:prob+-}. The $\pm$ versions of \eqref{eqn:KKTgrad} result from taking gradients with respect to $W^+$ and $W^-$ respectively, where $\lambda \geq 0$ is a Lagrange multiplier. $M^+$, $M^-$ are non-negative matrices of Lagrange multipliers corresponding to the non-negativity constraints in \eqref{eqn:prob+-}, with complementary slackness conditions \eqref{eqn:KKTcomplementary}.

As in Section~\ref{sec:theory:sq}, we must consider whether the KKT conditions are \emph{necessary} conditions of optimality, i.e.~whether a local minimum must satisfy them. Theorem~\ref{thm:necessary} gives an affirmative answer; however, it turns out that common \emph{constraint qualifications} used to establish necessity do not hold. 
To begin, we recall that a feasible solution to an inequality-constrained problem such as \eqref{eqn:prob+-} is said to be \emph{regular} if the gradients of the active (i.e.~tight) constraints are linearly independent. If a local minimum is regular, then the KKT conditions necessarily hold. 
\begin{prop}\label{prop:regular}
A feasible solution $(W^+, W^-)$ to problem \eqref{eqn:prob+-} cannot be regular.
\end{prop}

Beyond regularity, we refer to the hierarchy of constraint qualifications presented in \cite{bertsekas1999} and show that feasible solutions to \eqref{eqn:prob+-} do not satisfy a weaker constraint qualification called \emph{quasinormality}.
\begin{prop}\label{prop:quasinormal}
A feasible solution $(W^+, W^-)$ to problem \eqref{eqn:prob+-} cannot be quasinormal.
\end{prop}

In spite of these negative results, Appendix~\ref{sec:proofs:abs} provides a direct proof that KKT conditions \eqref{eqn:KKT} are satisfied at a local minimum of \eqref{eqn:prob+-}. The proof uses the following lemma, which we highlight because of its graphical interpretation in terms of directed paths not being created/destroyed by the addition/removal of certain edges.
\begin{lemma}\label{lem:simplePath}
For a non-negative matrix $A$, if $(\nabla h(A))_{ij} > 0$, changing the values of $A_{kj}$ for any $k$ cannot make $(\nabla h(A))_{ij} = 0$. Similarly if $(\nabla h(A))_{ij} = 0$, changing the values of $A_{kj}$ for any $k$ cannot make $(\nabla h(A))_{ij} > 0$.
\end{lemma}
\begin{theorem}\label{thm:necessary}
Let $(W^+, W^-)$ be a local minimum of problem \eqref{eqn:prob+-}. Then there exist a Lagrange multiplier $\lambda \geq 0$ and matrices $M^+ \geq 0$, $M^- \geq 0$ satisfying the KKT conditions in \eqref{eqn:KKT}.
\end{theorem}

\subsubsection{Relationships with explicit edge absence constraints}
\label{sec:theory:abs:zeros}

We now discuss relationships between the KKT conditions \eqref{eqn:KKT} and the optimality conditions for score optimization problems with explicit edge absence constraints, which correspond to zero-value constraints on the matrix $W$. Given a set $\cZ$ of such constraints, we consider the problem
\begin{equation}\label{eqn:probZeros}
    \min_{W} \; F(W) \quad \text{s.t.} \quad W_{ij} = 0, \quad (i,j) \in \cZ
\end{equation}
and denote by $W^*(\cZ)$ an optimal solution. The necessary conditions of optimality for \eqref{eqn:probZeros} are
\begin{subequations}\label{eqn:optCondZeros}
\begin{align}
    (\nabla F(W))_{ij} &= 0, \quad (i,j) \notin \cZ,\\
    W_{ij} &= 0, \quad (i,j) \in \cZ.
\end{align}
\end{subequations}

In one direction, given a KKT point $(W^+, W^-)$, we define the set 
\begin{equation}\label{eqn:setP}
\cP := \{(i,j): (\nabla h(W^+ + W^-))_{ij} > 0\},
\end{equation}
i.e.~the set of $(i,j)$ with directed walks from $j$ to $i$, according to Lemma~\ref{lem:gradh_ij}. 
\begin{lemma}\label{lem:KKTzeros}
If $(W^+, W^-)$ satisfies the KKT conditions in \eqref{eqn:KKT}, then $W^* = W^+ - W^-$ satisfies the optimality conditions in \eqref{eqn:optCondZeros} for $\cZ = \cP$. If in addition $F(W)$ is convex, then $W^*$ is a minimizer of \eqref{eqn:probZeros} for $\cZ = \cP$.
\end{lemma}

Under the assumption that $F$ is convex, we can use Lemma~\ref{lem:KKTzeros} to show that the KKT conditions \eqref{eqn:KKT} are \emph{sufficient} for local minimality in \eqref{eqn:probOrig}, despite the constraint $h(\abs{W}) \leq 0$ not being convex.
\begin{theorem}\label{thm:KKTsuff}
Assume that $F(W)$ is convex. Then if $(W^+, W^-)$ satisfies the KKT conditions in \eqref{eqn:KKT}, $W^* = W^+ - W^-$ is a local minimum for problem \eqref{eqn:probOrig}.
\end{theorem}

In the opposite direction of Lemma~\ref{lem:KKTzeros}, 
we focus on the case in which a minimizer $W^*(\cZ)$ of \eqref{eqn:probZeros} is feasible, i.e.~$h(A^*(\cZ)) = 0$ for $A^*(\cZ) = \abs{W^*(\cZ)}$. Then by Lemma~\ref{lem:Agradh}, we must have $(W^*(\cZ))_{ij} = 0$ wherever $\bigl(\nabla h(A^*(\cZ))\bigr)_{ij} > 0$. If $\cZ$ does not include such a pair $(i,j)$, we may add $(i,j)$ to $\cZ$ while preserving the optimality of the existing solution $W^*(\cZ)$ with respect to \eqref{eqn:probZeros} (since it already satisfies the new constraint $W_{ij} = 0$). Hence for feasible $W^*(\cZ)$, we adopt the convention that all $(i,j)$ with $(W^*(\cZ))_{ij} = 0$ and $\bigl(\nabla h(A^*(\cZ))\bigr)_{ij} > 0$ are included in $\cZ$. 

We call $\cZ$ \emph{irreducible} if it contains \emph{only} pairs $(i,j)$ for which $\bigl(\nabla h(A^*(\cZ))\bigr)_{ij} > 0$. 
\begin{theorem}\label{thm:zerosKKT}
If a minimizer $W^*(\cZ)$ of \eqref{eqn:probZeros} is feasible and $\cZ$ is irreducible, then $W^+ = (W^*(\cZ))_+$, $W^- = (W^*(\cZ))_-$ satisfy the KKT conditions in \eqref{eqn:KKT}.
\end{theorem}

If $W^*(\cZ)$ is feasible but $\cZ$ is not irreducible, then the following result guarantees that $\cZ$ may be reduced to an irreducible set without losing feasibility. We assume that $F(W)$ is separable (decomposable) as 
the following sum:
\[
F(W) = \sum_{j=1}^d F_{j}\bigl(W_{\cdot j}\bigr).
\]
\begin{lemma}\label{lem:reduceZ}
Assume that the score function $F(W)$ is separable. Suppose that $W^*(\cZ)$ in \eqref{eqn:probZeros} is feasible and $\cZ_0(j) = \{(i_1,j), \dots, (i_J,j)\} \subseteq \cZ$ is a subset for which $\bigl(\nabla h(A^*(\cZ)) \bigr)_{ij} = 0$, $(i,j) \in \cZ_0(j)$. Then $W^*(\cZ \backslash \cZ_0(j))$ is also feasible.
\end{lemma}
Since the removal of a constraint $(i,j) \in \cZ$ for which $\bigl(\nabla h(A^*(\cZ)) \bigr)_{ij} = 0$ does not affect feasibility, we call such a constraint \emph{unnecessary} as a somewhat colloquial shorthand.

Lemma~\ref{lem:reduceZ} removes a set of pairs $(i,j) \in \cZ_0(j)$ from $\cZ$ for which $\bigl(\nabla h(A^*(\cZ)) \bigr)_{ij} = 0$ while maintaining feasibility. The resulting set is then checked again for irreducibility. Since each application of Lemma~\ref{lem:reduceZ} removes at least one pair from $\cZ$, the re-optimization \eqref{eqn:reduceZ} has to be performed at most $\abs{\cZ}$ times to ensure a irreducible set.

The development in this subsection suggests the meta-algorithm in Algorithm~\ref{alg:KKTS}, which we refer to as KKT-informed local search. An instantiation is described in Section~\ref{sec:algs:KKTS}.
\begin{algorithm}
\caption{KKT-informed local search (KKTS)}
\label{alg:KKTS}
\begin{algorithmic}[1]
\Require Initial set $\cZ$ of edge absence constraints. Solve \eqref{eqn:probZeros}.
\While{$W^*(\cZ)$ infeasible}
    \State Select edge(s) in cycle ($(W^*(\cZ))_{ij} \neq 0$, $\bigl(\nabla h(A^*(\cZ)) \bigr)_{ij} > 0$). Add to $\cZ$. Re-solve \eqref{eqn:probZeros}.\label{alg:KKTS:remove}
\EndWhile
\While{$\cZ$ reducible}  
    \State Remove one or more unnecessary constraints $(i,j) \in \cZ$ for which $\bigl(\nabla h(A^*(\cZ)) \bigr)_{ij} = 0$ 
    (see Lemma~\ref{lem:reduceZ}). Re-solve \eqref{eqn:probZeros}.\label{alg:KKTS:restore}
\EndWhile
\end{algorithmic}
\end{algorithm}
\begin{theorem}\label{thm:algZC}
If $F(W)$ is separable, KKT-informed local search 
yields a solution satisfying 
the KKT conditions \eqref{eqn:KKT}.
\end{theorem}
When combined with Theorem~\ref{thm:KKTsuff} and a convex $F(W)$, Theorem~\ref{thm:algZC} guarantees that KKT-informed local search will result in local minima. However, due to the non-convex constraint, the quality of such local minima is highly dependent on the particular instantiation of the meta-algorithm. Section~\ref{sec:expt} shows for example that the choice of initialization plays a large role.

\section{Algorithms}
\label{sec:algs}

For the algorithms in this section, we let the score function $F(W)$ be the sum of a smooth loss function $\ell(W; \bX)$ with respect to the data $\bX$ and an $\ell_1$ penalty to promote overall sparsity, as in \cite{zheng2018dags}: 
\begin{equation}\label{eqn:scoreL1}
F(W) = \ell(W; \bX) + \tau \norm{W}_1.
\end{equation}

\subsection{Augmented Lagrangian with absolute value adjacency matrix}
\label{sec:algs:ALabs}

Formulation~\eqref{eqn:probOrig} naturally suggests an augmented Lagrangian algorithm as in \cite{zheng2018dags} but with $h(\abs{W})$ instead of $h(W \circ W)$. 
Using the $(W^+, W^-)$ representation as in \eqref{eqn:prob+-}, the augmented Lagrangian minimized in each iteration is 
\[
L(W^+, W^-, \alpha, \rho) = \ell\bigl(W^+ - W^-; \bX\bigr) + \tau \ones^T \bigl(W^+ + W^-\bigr) \ones + \alpha h\bigl(W^+ + W^-\bigr) + \frac{\rho}{2} h\bigl(W^+ + W^-\bigr)^2,
\]
subject to $W^+ \geq 0$ and $W^- \geq 0$, where $\ones$ is a vector of ones. The gradients 
with respect to $W^+, W^-$ 
are given by
\[
\nabla_{W^\pm} L(W^+, W^-, \alpha, \rho) = \pm \nabla\ell\bigl(W^+ - W^-; \bX\bigr) + \tau \ones \ones^T + \left(\alpha + \rho h\bigl(W^+ + W^-\bigr) \right) \nabla h\bigl(W^+ + W^-\bigr).
\]
We otherwise closely follow the algorithm in \cite{zheng2018dags}. 

\subsection{KKT-informed local search}
\label{sec:algs:KKTS}

We now describe an instantiation of the KKT-informed local search meta-algorithm in Algorithm~\ref{alg:KKTS}. 
This involves initializing the set $\cZ$ of edge absence constraints, selecting edges for removal (line~\ref{alg:KKTS:remove}), reducing unnecessary constraints (line~\ref{alg:KKTS:restore}), and re-solving \eqref{eqn:probZeros}. 
We also discuss an additional operation of reversing edges, which is not part of Algorithm~\ref{alg:KKTS} but helps in attaining better local minima. 

\paragraph{Initializing $\cZ$}
We allow any matrix $W$ to serve as an initial solution. To define the set $\cZ$, we set to zero elements in $W$ that are smaller than a threshold $\omega$ in absolute value. We then let $\cZ = \{(i,j): W_{ij} = 0\}$.

\paragraph{Selecting edges for removal (line~\ref{alg:KKTS:remove})} There are many possible ways of selecting edges to break cycles. Here we
consider an approach of minimizing the Lagrangian $F(W) + \alpha h(\abs{W})$ of \eqref{eqn:probOrig} subject to the existing constraints $W_{ij} = 0$ for $(i,j) \in \cZ$. 
The Lagrangian thus trades off minimizing the score function against reducing infeasibility. 
For $\alpha = 0$, the minimizer is the existing solution $W^*(\cZ)$, and as $\alpha$ increases, weights $W_{ij}$ will be set to zero to decrease the infeasibility penalty $h(\abs{W})$.

We implement a computationally simple version of the above idea. First, $h(A) = h(\abs{W})$ in the Lagrangian is linearized around $A^*(\cZ) = \abs{W^*(\cZ)}$ as 
\[
h(A) \approx h\bigl(A^*(\cZ)\bigr) + \left\langle \nabla h\bigl(A^*(\cZ)\bigr), A - A^*(\cZ) \right\rangle.
\]
After dropping constant terms and expanding the inner product, the constrained, linearized Lagrangian to be minimized is as follows: 
\begin{equation}\label{eqn:selectRemove}
    \min_W \;\; F(W) + \alpha \sum_{(i,j): i\neq j} \bigl(\nabla h(A^*(\cZ))\bigr)_{ij} \abs{W_{ij}} \quad \text{s.t.} \quad W_{ij} = 0, \quad (i,j) \in \cZ.
\end{equation}
Problem~\eqref{eqn:selectRemove} is a score minimization problem with a weighted $\ell_1$ penalty and 
zero-value constraints, i.e.~the corresponding parameters are simply absent. 
Furthermore, 
in the common case where 
$F(W)$ is separable column-wise, \eqref{eqn:selectRemove} is also separable.

Second, 
$\alpha$ is increased from zero only until a single existing edge $(i,j)$ (with $(W^*(\cZ))_{ij} \neq 0$) belonging to a cycle ($\bigl(\nabla h(A^*(\cZ))\bigr)_{ij} > 0$) is set to zero. This involves following the solution \emph{path} of \eqref{eqn:selectRemove} defined by $\alpha$ from $W^*(\cZ)$ at $\alpha = 0$ until the first additional edge is removed. 
If $\ell(W; \bX)$ in \eqref{eqn:scoreL1} is the least-squares loss, the solution path is piecewise linear and we have implemented a modified version of the LARS algorithm \cite{efron2004} to efficiently track the path. The modification accounts for the non-uniformity of the weights $\bigl(\nabla h(A^*(\cZ))\bigr)_{ij}$, some of which may be zero, in the $\ell_1$ penalty in \eqref{eqn:selectRemove}. It is described further in Appendix~\ref{sec:algsAdd:path}.

\paragraph{Reducing unnecessary constraints (line~\ref{alg:KKTS:restore})}
We also refer to this step as restoring edges (``restore'' because these edges were likely present in an earlier iteration when $W$ was denser), in analogy with the previous step which removes edges. When there are multiple unnecessary constraints, the order in which they 
are removed can matter because the removal of constraints and 
re-optimization of \eqref{eqn:probZeros}
can make previously unnecessary constraints necessary. Because of this, even though Lemma~\ref{lem:reduceZ} allows for multiple unnecessary constraints $(i_1,j), \dots, (i_J,j)$ to be removed at a time, we opt to do so more gradually, 
only one at a time. To decide among multiple unnecessary constraints $(i,j)$, we greedily choose one for which the absolute partial derivative of the loss function, $\abs{(\nabla \ell(W; \bX))_{ij}}$, is largest. Since $\abs{(\nabla \ell(W; \bX))_{ij}}$ is the marginal rate of decrease of the loss as the constraint $W_{ij} = 0$ is relaxed, this strategy gives the largest marginal rate of decrease. We note also that if $\abs{(\nabla \ell(W; \bX))_{ij}} \leq \tau$, relaxing the constraint $W_{ij} = 0$ does not change its value because $W_{ij} = 0$ already satisfies the optimality conditions for minimizing $F(W)$.

\paragraph{Reversing edges}
In addition to removing and restoring edges, 
we consider reversing edges, which involves two operations: adding $(i,j)$ to $\cZ$ to remove an existing edge $(W^*(\cZ))_{ij} \neq 0$, and removing $(j,i)$ from $\cZ$ (which must have been a necessary constraint if $W^*(\cZ)$ is feasible, to avoid a $2$-cycle) to introduce the opposite edge. 
In contrast to removing edges, which generally increases $F(W)$ but decreases $h(A)$, and restoring edges, which decreases $F(W)$ and is guaranteed by Lemma~\ref{lem:reduceZ} not to increase $h(A)$, reversing edges does not necessarily decrease $F(W)$ or $h(A)$. We therefore \emph{accept} an edge reversal only if it decreases one of $F(W)$, $h(A)$ relative to the original direction and does not increase the other, and otherwise \emph{reject} the reversal.

There are many possible variations in when to perform edge reversals within Algorithm~\ref{alg:KKTS}. In our implementation, we restrict reversals to the second while-loop and alternate between restoring one edge (reducing $\cZ$ by one) and attempting all possible reversals given the current state. When there are multiple reversal candidates, similar to restoring edges, we evaluate the loss partial derivatives $\abs{(\nabla \ell(W; \bX))_{ji}}$, this time associated with introducing the reverse edges $(j,i)$, and proceed in order of decreasing $\abs{(\nabla \ell(W; \bX))_{ji}}$. 

The edge reversal operation is made much more efficient by keeping a memory of previously attempted reversals that do not have to be attempted again for some time. When the reversal of edge $(i,j)$ is attempted, it is recorded in the memory, and if the reversal is accepted, reversal of $(j,i)$ is also added to the memory as it would revert to the previous inferior state. The memory for $(i,j)$ is cleared when either column $i$ or $j$ is updated since this may change the value of reversing $(i,j)$. 

\paragraph{Re-solving \eqref{eqn:probZeros} (lines \ref{alg:KKTS:remove}, \ref{alg:KKTS:restore})}
Removing, restoring, and reversing edges all involve re-solving \eqref{eqn:probZeros} after adding to $\cZ$, reducing $\cZ$, 
or both in the case of reversals. When $\ell(W; \bX)$ in \eqref{eqn:scoreL1} is the least-squares loss, these re-optimizations can be done efficiently using the LARS algorithm. In the case of adding $(i,j)$ to $\cZ$, an increasing penalty is imposed on $\abs{W_{ij}}$, while in the case of 
removing $(i,j)$ from $\cZ$, a penalty equivalent to the constraint $W_{ij} = 0$ is inferred and then decreased to zero. Further 
details are in Appendix~\ref{sec:algsAdd}.

\section{Experiments}
\label{sec:expt}

We compare the structure learning performance of the following base algorithms: NOTEARS \cite{zheng2018dags}, the FGS implementation \cite{ramsey2017million} of GES \cite{chickering02}, MMHC \cite{MMPCcor}, PC \cite{PCalgorithm}, augmented Lagrangian with absolute value adjacency matrix $A = \abs{W}$ (Section~\ref{sec:algs:ALabs}, abbreviated `Abs'), 
and KKT-informed local search (Section~\ref{sec:algs:KKTS}, KKTS) initialized with the unconstrained solution ($\cZ = \{(i,i)$, $i\in\cV\}$ just to avoid self-loops). 
We also experimented with CAM \cite{buhlmann2014cam} but defer those results to Appendix~\ref{sec:exptAdd:results} as we found them less competitive in the tested settings. In addition, we use each of the above base algorithms to initialize KKTS (denoted by appending `-KKTS' and excepting KKTS itself). Algorithm parameter settings are detailed in Appendix~\ref{sec:exptAdd:params}. Of note are the default termination tolerance on $h$, $\epsilon = 10^{-10}$, and the threshold on $W$, $\omega = 0.3$ following \cite{zheng2018dags}, applied after NOTEARS, Abs, and KKTS as well as to initialize $\cZ$ before KKTS.

\begin{figure}[t]
  \centering
  \includegraphics[width=\linewidth]{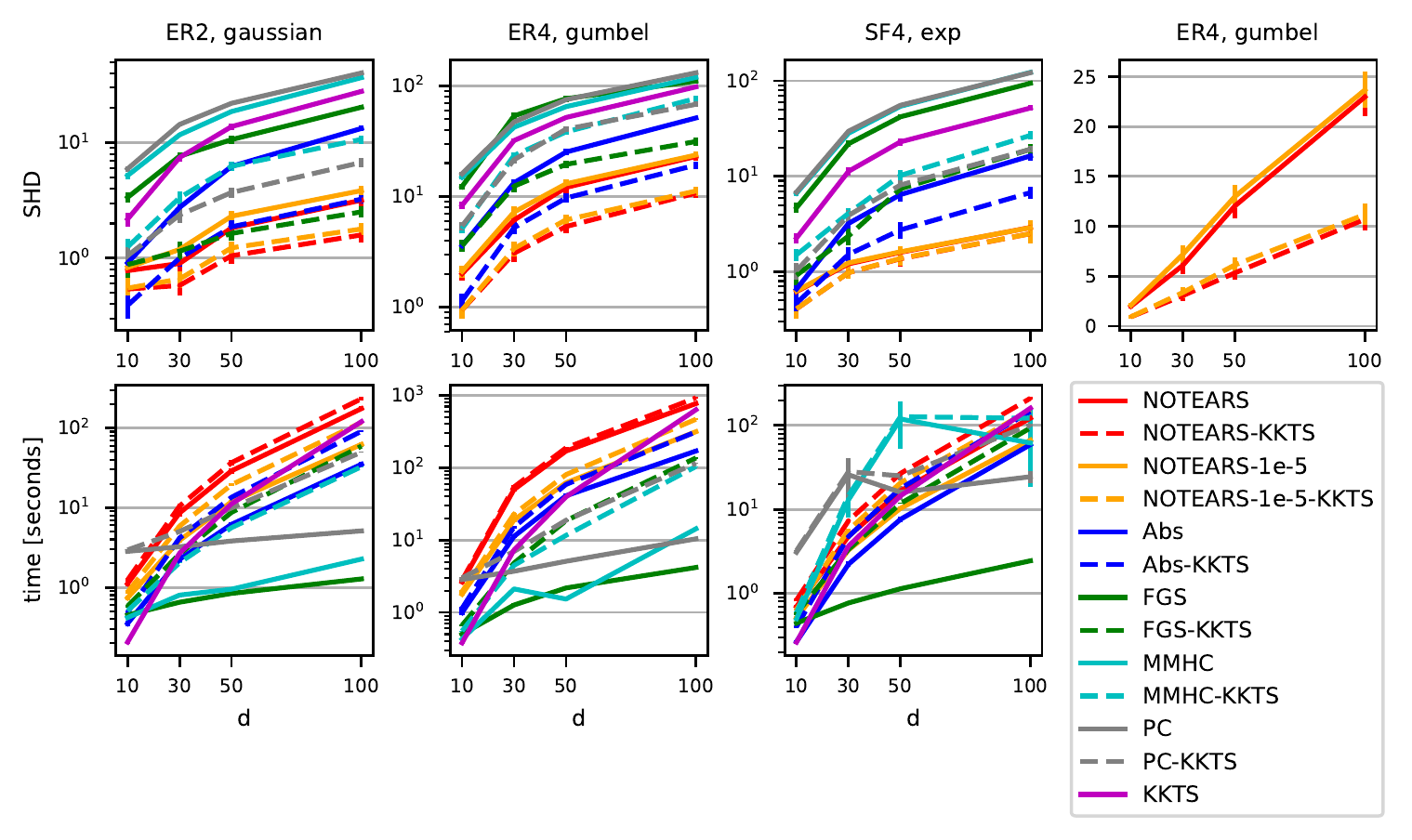}
  \caption{Structural Hamming distances (SHD) with respect to true graph and solution times for $n = 1000$. Error bars indicate standard errors over $100$ trials. Red lines overlap with orange in the SF4 SHD plot. The upper right panel focuses on combinations with NOTEARS using a linear vertical scale.}
  \label{fig:main}
\end{figure}

The experimental setup is similar to \cite{zheng2018dags}. In brief, random Erd\"{o}s-R\'{e}nyi or scale-free graphs are generated with $kd$ expected edges (denoted ER$k$ or SF$k$), and uniform random weights $W$ are assigned to the edges. Data $\bX \in \mathbb{R}^{n\times d}$ is then generated by taking $n$ i.i.d.~samples from the linear SEM $X = W^T X + z$, where $z$ is either Gaussian, Gumbel, or exponential noise. $100$ trials are performed for each graph type-noise type combination, which is an order of magnitude larger than in e.g.~\cite{zheng2018dags,yu2019dag} and reduces the standard errors of the estimated means.

Figure~\ref{fig:main} shows structural Hamming distances (SHD) with respect to the true graph and running times for three graph-noise combinations and $n = 1000$. Figure~\ref{fig:main_n_2d} shows the same metrics and combinations for the more challenging setting $n = 2d$, with largely similar patterns. Other graph-noise combinations, results in tabular form, and computing environment details are in Appendix~\ref{sec:exptAdd}. 

We focus first on the base algorithms (solid lines), of which NOTEARS is clearly the best in terms of SHD.\footnote{The SHDs for NOTEARS and FGS in Figure~\ref{fig:main} are much better than those reported in \cite{zheng2018dags}, by almost an order of magnitude in some cases. Part of the improvement is due to code updates for NOTEARS but the rest we cannot explain. We also show in Appendix~\ref{sec:exptAdd:meanSub} that subtracting the mean from $\bX$ improves the SHD by a noticeable factor for some noise types. All results in Figure~\ref{fig:main} are obtained with zero-mean $\bX$.} Abs is next and better than FGS, MMHC, and PC. We hypothesize that the smoothness of the quadratic adjacency $A = W \circ W$ used by NOTEARS is better able to overcome non-convexity than the non-smooth $A = \abs{W}$ of Abs, which tends to force parameters $W_{ij}$ to zero, perhaps too soon. 
The non-convexity is further reflected in the inferior performance of (pure) KKTS, which only takes local steps starting from the unconstrained solution.

\begin{figure}[t]
  \centering
  \includegraphics[width=\linewidth]{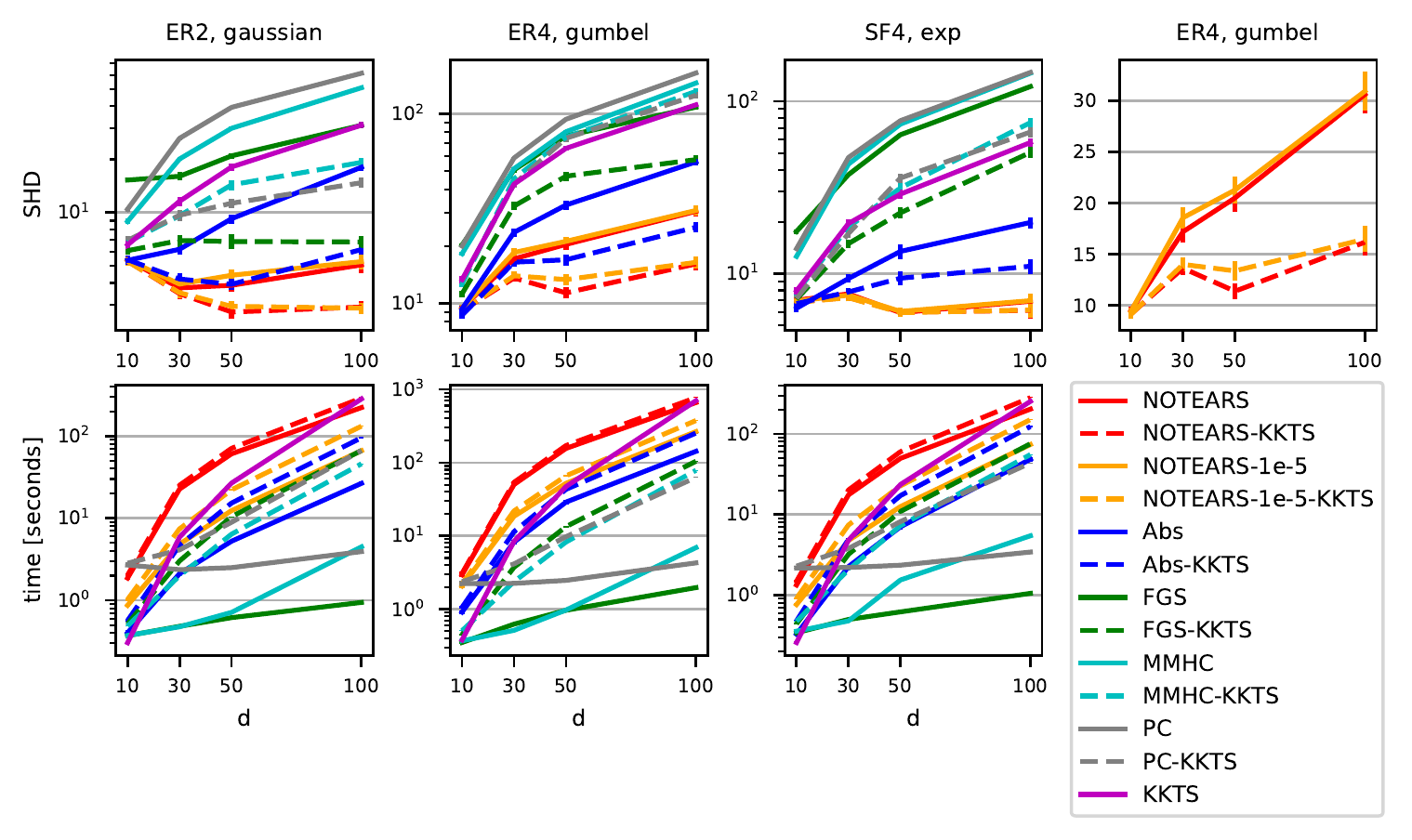}
  \caption{Structural Hamming distances (SHD) with respect to true graph and solution times for $n = 2d$. Red lines overlap with orange in the SF4 SHD plot. The upper right panel focuses on combinations with NOTEARS using a linear vertical scale.}
  \label{fig:main_n_2d}
\end{figure}

We now turn to the `-KKTS' combinations (dashed lines). 
It is seen that KKTS, and the theoretical understanding it embodies, improve the SHD of \emph{all} base algorithms (including CAM in Appendix~\ref{sec:exptAdd:results}). The improvement is by at least a factor of $2$, except when the SHD is already low (e.g.~NOTEARS on SF4), and moreover is consistent across dimensions $d$. An ablation study in Appendix~\ref{sec:exptAdd:ablation} shows that both reducing unnecessary constraints and reversing edges contribute to the improvement.

In the case of NOTEARS-KKTS, while Proposition~\ref{prop:ALsq} asserts that NOTEARS cannot yield an exactly feasible solution, let alone a KKT point, Figure~\ref{fig:main} confirms that it yields high-quality nearly feasible solutions. NOTEARS is therefore well-suited as an initialization for KKTS, and combining them apparently results in new state-of-the-art accuracy. Furthermore, in an attempt to achieve feasibility, NOTEARS uses more augmented Lagrangian iterations and very large penalty parameters $\alpha$ and $\rho$. The latter causes the augmented Lagrangian \eqref{eqn:ALsq} to be poorly conditioned and optimization solvers for it to take longer to converge. Thus, to reduce solution time as well as satisfy KKT conditions, we terminate NOTEARS early with a higher $h$ tolerance of $\epsilon = 10^{-5}$ before running KKTS. Figure~\ref{fig:main} shows that this results in nearly the same SHD improvement over NOTEARS while also taking considerably less time (except for SF4). Abs-KKTS similarly outperforms NOTEARS on ER graphs and takes even less time.

\section{Conclusion and future work}
\label{sec:concl}

We have re-examined a recently proposed continuous optimization framework for learning Bayesian networks. Our most important contributions are as follows: (1) better understanding of the NOTEARS formulation and algorithm of \cite{zheng2018dags}; (2) analysis and understanding of the KKT optimality conditions for an equivalent reformulation (for which they do indeed hold); (3) a local search algorithm informed by the KKT conditions that significantly and universally improves the accuracy of NOTEARS and other algorithms.

A clear next step is to generalize the theory and algorithms to the case in which each edge in the graph corresponds to multiple parameters. One motivation is to allow nonlinear models; a nonlinear extension of the absolute value case of Section~\ref{sec:theory:abs} could parallel the recent nonparametric extension \cite{zheng2020learning} 
for the quadratic case. Another reason for having multiple parameters is to accommodate non-binary categorical variables, which are typically encoded into multiple binary variables on the input side, or predicted using e.g.~multi-logit regression \cite{gu2019penalized} on the output side. Other future directions include improving the efficiency of algorithms for solving \eqref{eqn:probSq}, \eqref{eqn:probOrig} and exploring 
alternative acyclicity characterizations from Section~\ref{sec:chars}.

\section*{Broader Impact}

Bayesian networks are fundamentally about modeling the joint probability distribution of data, in a parsimonious and comprehensible manner. This work therefore contributes mostly to layer 0 (``foundational research'') in the ``Impact Stack'' of \cite{ashurst2020guide}, particularly with regard to the theoretical aspects. If one views Bayesian network structure learning as a ``ML technique'' rather than a ``foundational technique'', then the algorithmic contribution also falls into layer 1. We thus confine our discussion of broader impacts mostly to layers 0 and 1, i.e.~``tractable'' impacts according to \cite{ashurst2020guide}, as it is difficult and perhaps inappropriate to speculate further.

The predominant contribution of this work is to theoretical understanding of the optimization problem that is score-based structure learning, and specifically a 
continuous formulation thereof. This understanding has resulted in improvements in accuracy (as measured by structural Hamming distance), and we expect that further improvements will be made in future work. We also believe that this understanding may lead to advances in computational efficiency as well, beyond the simple measure of terminating the NOTEARS algorithm early when it has no hope of reaching feasibility, or observing that the absolute value version (Abs) converges more quickly. For example, new optimization algorithms may be proposed for problems \eqref{eqn:probSq} and/or \eqref{eqn:probOrig} that take better advantage of their properties.

As the accuracy and scalability of Bayesian network structure learning continue to increase, we hope that it becomes an even more commonly used technique for modeling data than it is now. We are particularly interested in its use as the first step in \emph{causal} structure discovery, which may then facilitate other causal inference tasks. We recognize however that errors in structure learning may compound into potentially more serious downstream errors. This is an issue calling for further study.

\section*{Acknowledgments}
Y. Yu is supported by the National Science Foundation under award DMS 1753031.

\bibliography{ref}
\bibliographystyle{plainnat}

\clearpage
\appendix

\section{Proofs}
\label{sec:proofs}

\subsection{Proofs for Section~\ref{sec:chars}}

\subsubsection{Proof of Lemma~\ref{lem:nilpotent}}
Given a weighted adjacency matrix $A$, we define the \emph{weight} of a directed walk from $i$ to $j$ to be the product $A_{i,i_1} A_{i_1,i_2} \dots A_{i_{l-1},j}$. It is well-known that $(A^p)_{ij}$ is the sum of the weights of all length-$p$ directed walks from $i$ to $j$ \cite{bondy_murty2008}. Therefore $\tr(A^p)$ is the sum of the weights of all length-$p$ directed circuits. If $\cG$ is acyclic, then all of these sums are zero, i.e.~$A$ is nilpotent according to the definition. The converse also holds.

\subsubsection{Proof of Theorem~\ref{thm:h}}
Using Lemma~\ref{lem:nilpotent}, we equivalently show that $A$ is nilpotent if and only if $h(A) = 0$. The ``only if'' direction is clearly true. 
    
If $h(A) = 0$, then because $c_p > 0$, $p = 1,\dots,d$, and $\tr(A^p) \geq 0$ due to the non-negativity of $A$, we must have $\tr(A^p) = 0$, $p = 1,\dots,d$. The extension to higher powers of $A$ can be shown by induction using the Cayley-Hamilton theorem. For the base case $d+1$, $A^{d+1}$ can be expressed as a linear combination of $A, \dots, A^{d}$, specifically by multiplying the characteristic polynomial of $A$ by another power of $A$. Therefore $\tr(A^{d+1}) = 0$. For the inductive step $p > d+1$, $A^p$ can similarly be expressed as a linear combination of $A^{p-d}, \dots, A^{p-1}$, the traces of which are all known to be zero. We conclude that $\tr(A^p) = 0$ for all $p \in \mathbb{N}$.

\subsubsection{Proof of Lemma~\ref{lem:gradh_ij}}
From the power series expression for $\nabla h(A)$,
\begin{equation}\label{eqn:gradh_ij}
(\nabla h(A))_{ij} = \sum_{p=1}^d p c_p \left(A^{p-1}\right)_{ji} = \sum_{p=1}^{d-1} (p+1) c_{p+1} \left(A^{p}\right)_{ji}
\end{equation}
for $i \neq j$. Thus if $(\nabla h(A))_{ij} > 0$, then $(A^p)_{ji} > 0$ for at least one $p$, i.e.~there exists a directed walk of length $p$ from $j$ to $i$.
    
Conversely, if there is a directed walk from $j$ to $i$, then there is also a directed path from $j$ to $i$. A directed path can have length at most $d-1$ since no vertices can be repeated. Therefore $(A^p)_{ji} > 0$ for at least one $p$ in $\{1,\dots,d-1\}$ and $(\nabla h(A))_{ij} > 0$ from \eqref{eqn:gradh_ij}.

\subsubsection{Proof of Lemma~\ref{lem:Agradh}}

Lemma~\ref{lem:Agradh} follows from Lemma~\ref{lem:Agradh2} below and rewriting $\tr\left((\nabla h(A))^T A\right)$ as the inner product 
\[
\tr\left((\nabla h(A))^T A\right) = \sum_{i,j} (\nabla h(A))_{ij} A_{ij}.
\]
Since $A$ is non-negative and $\nabla h(A)$ is also non-negative (Lemma~\ref{lem:gradhNonNeg}), $\tr\left((\nabla h(A))^T A\right) = 0$ if and only if $(\nabla h(A))_{ij} A_{ij} = 0$ for all $i,j$.
\begin{lemma}\label{lem:Agradh2}
A directed graph $\cG$ is acyclic if and only if $\tr\left((\nabla h(A))^T A\right) = 0$ for any $h$ defined by \eqref{eqn:h}.
\end{lemma}
\begin{proof}
Again from the power series expression $\nabla h(A) = \sum_{p=1}^d p c_p \bigl(A^{p-1}\bigr)^T$,
\[
\tr\left((\nabla h(A))^T A\right) = \sum_{p=1}^d p c_p \tr\left(A^p\right).
\]
Similar to \eqref{eqn:h}, this is a strictly positive linear combination of non-negative traces $\tr(A^p)$, $p = 1,\dots,d$. Thus $\tr\left((\nabla h(A))^T A\right) = 0$ if and only if $\tr(A^p) = 0$ for $p = 1,\dots,d$. Similarly from \eqref{eqn:h}, $h(A) = 0$ if and only if $\tr(A^p) = 0$ for $p = 1,\dots,d$. Theorem~\ref{thm:h} completes the chain of equivalences.
\end{proof}

\subsection{Proofs for Section~\ref{sec:theory:sq}}

Propositions~\ref{prop:noKKTsq} and \ref{prop:ALsq} are immediate consequences of Lemma~\ref{lem:zeroGrad}.

\subsubsection{Proof of Lemma~\ref{lem:zeroGrad}}
By the chain rule,
\[
\nabla_W (h(W \circ W)) = \nabla h(W \circ W) \circ 2W,
\]
where $\nabla h(W \circ W)$ refers to 
\[
\nabla h(A) = \sum_{p=1}^d p c_p \left(A^{p-1}\right)^T
\]
evaluated at $A = W \circ W$. (The above gradient expression generalizes eq.~(8) in \cite{zheng2018dags}.) If $W$ is feasible, i.e.~$h(W \circ W) = 0$, then Lemma~\ref{lem:Agradh} with $A = W \circ W$ implies that $\nabla h(W \circ W) \circ W \circ W = 0$. Since the latter is true if and only if $\nabla h(W \circ W) \circ W = 0$, we have $\nabla_W (h(W \circ W)) = 0$.

\subsection{Equivalence of problems \eqref{eqn:probOrig} and \eqref{eqn:prob+-}}
\label{sec:proofs:equiv}

We map between solutions to \eqref{eqn:probOrig} and \eqref{eqn:prob+-} as follows: 
\begin{subequations}\label{eqn:W+-map}
\begin{align}
    W &\mapsto \left(W^+, W^-\right) = \left((W)_+, (W)_-\right),\label{eqn:WtoW+-}\\
    \left(W^+, W^-\right) &\mapsto W = W^+ - W^-,\label{eqn:W+-toW}
\end{align}
\end{subequations}
where 
\[
(W)_+ := \max\{W, 0\}, \quad (W)_- := -\min\{W, 0\},
\]
and the maximum and minimum are taken element-wise. $(W)_+$ and $(W)_-$ are therefore the positive and negative parts of $W$, motivating the $W^+$, $W^-$ notation. 

To establish the equivalence, we introduce the following intermediate formulation with the additional constraint $W^+ \circ W^- = 0$: 
\begin{equation}\label{eqn:prob+-Cons}
\begin{split}
\min_{W^+, W^-} \quad &F\left(W^+ - W^-\right)\\
\text{s.t.} \quad &h\left(W^+ + W^-\right) \leq 0, \quad W^+, W^- \geq 0, \quad W^+ \circ W^- = 0.
\end{split}
\end{equation}
The mappings in \eqref{eqn:W+-map} define a one-to-one correspondence between $\mathbb{R}^{d\times d}$ and non-negative pairs $(W^+, W^-)$ satisfying $W^+ \circ W^- = 0$. Thus we have the following.
\begin{lemma}\label{lem:prob+-Cons}
If $W$ is a feasible solution to problem \eqref{eqn:probOrig}, then applying mapping \eqref{eqn:WtoW+-} to $W$ yields a feasible solution to \eqref{eqn:prob+-Cons} with the same objective value. Conversely if $(W^{+}, W^{-})$ is a feasible solution to \eqref{eqn:prob+-Cons}, then $W = W^{+} - W^{-}$ is a feasible solution to \eqref{eqn:probOrig} with the same objective value.
\end{lemma}
\begin{proof}
Mapping \eqref{eqn:WtoW+-} satisfies the constraints $W^\pm \geq 0$ and $W^+ \circ W^- = 0$. Under this last condition, we also have $W^+ + W^- = \abs{W}$. These facts show that \eqref{eqn:W+-map} preserves feasibility in both directions. Since $(W)_+ - (W)_- = W$, \eqref{eqn:WtoW+-} preserves the objective value, and clearly \eqref{eqn:W+-toW} does as well.
\end{proof}

We now show that the additional constraint $W^+ \circ W^- = 0$ in \eqref{eqn:prob+-Cons} does not change the optimal value, i.e.~there is no advantage from dropping it.
\begin{lemma}\label{lem:W+W-=0}
If $(W^+, W^-)$ is a feasible solution to problem \eqref{eqn:prob+-} and $W^+ \circ W^- \neq 0$, then there exists a feasible solution $\left(W_0^+, W_0^-\right)$ with the same objective value and satisfying $W_0^+ \circ W_0^- = 0$.
\end{lemma}
\begin{proof}
For any $(i,j)$ such that $W^+_{ij} W^-_{ij} > 0$, we can obtain another feasible solution by reducing each of $W^+_{ij}$, $W^-_{ij}$ by the same amount until $W^+_{ij} W^-_{ij} = 0$. Since the objective is a function of $W^+ - W^-$, its value is unchanged. At the same time, Lemma~\ref{lem:gradhNonNeg} ensures that $h(W^+ + W^-)$ cannot increase since it is a non-decreasing function, and thus the solution remains feasible.
\end{proof}
\noindent In particular, an optimal solution to \eqref{eqn:prob+-} not satisfying $W^+ \circ W^- = 0$ can be reduced to another optimal solution that does satisfy $W^+ \circ W^- = 0$. Hence it suffices to solve \eqref{eqn:prob+-Cons} in order to solve \eqref{eqn:prob+-}.

The combination of Lemmas~\ref{lem:prob+-Cons} and \ref{lem:W+W-=0} yields the following equivalence:
\begin{prop}
If $W^*$ is an optimal solution to problem \eqref{eqn:probOrig}, then applying mapping \eqref{eqn:WtoW+-} to $W^*$ yields an optimal solution $(W^{+*}, W^{-*})$ to \eqref{eqn:prob+-}. Conversely if $(W^{+*}, W^{-*})$ is an optimal solution to \eqref{eqn:prob+-}, then $W^* = W^{+*} - W^{-*}$ is an optimal solution to \eqref{eqn:probOrig}.
\end{prop}

\subsection{Proofs for Section~\ref{sec:theory:abs:KKT}}
\label{sec:proofs:abs}

\subsubsection{Proof of Proposition~\ref{prop:regular}}

To begin, we recall that a feasible solution to an inequality-constrained problem such as \eqref{eqn:prob+-} is said to be \emph{regular} if the gradients of the active (i.e.~tight) constraints are linearly independent \cite{bertsekas1999}. If a local minimum is regular, then the KKT conditions necessarily hold.

We first give expressions for the gradients of the constraints in \eqref{eqn:prob+-}. With $A = W^+ + W^-$, the gradient of $h(A)$ with respect to either $W^+$ or $W^-$ is given by $\nabla h(A)$ itself. Recalling that $W^\pm \geq 0$ is a collection of constraints $W^\pm_{ij} \geq 0$, the gradient of (say) constraint $W^+_{ij} \geq 0$ is a matrix $E^{ij}$ with entry $(i,j)$ equal to $1$ and $0$ elsewhere. A linear combination of these gradients with respect to $W^+$ (respectively $W^-$) can be represented as a matrix $M^+$ (respectively $M^-$). It will be seen shortly that we can take a non-negative linear combination of these gradients, so $M^+$, $M^-$ are non-negative and we reuse the symbol $M$ from \eqref{eqn:KKTgrad}.
 
If $(W^+, W^-)$ is feasible, then we must have $h(A) = 0$ so the constraint $h(A) \leq 0$ is active. Consider then the equation 
\begin{equation}\label{eqn:regular_1}
M^+ = M^- = \nabla h(A),
\end{equation}
which expresses the gradient of the constraint $h(A) \leq 0$ (with respect to $W^+$ or $W^-$) as a linear combination of gradients of the constraints $W^+_{ij} \geq 0$ or $W^-_{ij} \geq 0$. More specifically, $M^+$ and $M^-$ in \eqref{eqn:regular_1} are linear combinations only of those gradients $(i,j)$ for which $(\nabla h(A))_{ij} > 0$. By Lemma~\ref{lem:Agradh}, $h(A) = 0$ implies that
\[
\nabla h(A) \circ A = \nabla h(A) \circ \left(W^+ + W^-\right) = 0.
\]
In particular, if $(\nabla h(A))_{ij} > 0$, then $W^+_{ij} = W^-_{ij} = 0$, i.e.~these two constraints are active. Thus $M^+$, $M^-$ are linear combinations of active constraint gradients only, and \eqref{eqn:regular_1} equates these linear combinations to the gradient of active constraint $h(A) \leq 0$. We conclude that $(W^+, W^-)$ is not regular.

\subsubsection{Proof of Proposition~\ref{prop:quasinormal}}

Quasinormality is a weaker constraint qualification than regularity and is described in \cite[Sec.~3.3.5, p.~336]{bertsekas1999}.
We follow the framework therein. We let the convex set $\cX$ be $\mathbb{R}_+^{d\times d} \times \mathbb{R}_+^{d\times d}$, the set of pairs of non-negative matrices, to account for the constraints $W^\pm \geq 0$. Thus $h(A) \leq 0$ remains as a single inequality constraint, where again $A = W^+ - W^-$.

A feasible solution $(W^+, W^-)$ is \emph{not} quasinormal if it satisfies conditions (i)--(iv) in \cite[Sec.~3.3.5, p.~336]{bertsekas1999}. Translated to the current case of a single inequality constraint, these conditions are (i)
\begin{equation}\label{eqn:quasinormal_1}
\sum_{i,j} (\nabla h(A))_{ij} \left(W'^+_{ij} + W'^-_{ij} - W^+_{ij} - W^-_{ij}\right)  \geq 0 \quad \forall \left(W'^+, W'^-\right) \in \cX,
\end{equation}
and (iv) in every neighborhood around $(W^+, W^-)$ (e.g.~$\ell_2$ balls), there exists a $(W'^+, W'^-) \in \cX$ for which $h(W'^+ + W'^-) > 0$. Conditions (ii) and (iii) are easily satisfied by setting the single multiplier $\mu = 1$. 

To show that condition (i) \eqref{eqn:quasinormal_1} is satisfied, we consider the cases $(\nabla h(A))_{ij} > 0$ and $(\nabla h(A))_{ij} = 0$. In the former case, since $(W^+, W^-)$ is feasible, Lemma~\ref{lem:Agradh} requires that $A_{ij} = W^+_{ij} + W^-_{ij} = 0$. Hence the corresponding term in \eqref{eqn:quasinormal_1} becomes $(\nabla h(A))_{ij} \bigl(W'^+_{ij} + W'^-_{ij}\bigr)$ and is always non-negative. In the latter case $(\nabla h(A))_{ij} = 0$, the contribution to the sum is zero. Therefore \eqref{eqn:quasinormal_1} is satisfied.

Condition (iv) can be satisfied by choosing $W' = W'^+ - W'^-$ to be a fully dense matrix (corresponding to a complete graph) that is arbitrarily close to $W = W^+ - W^-$. Concretely, let $W'^- = W^-$, $W'^+_{ij} = \epsilon$ wherever $W^+_{ij} = W^-_{ij} = 0$, and $W'^+_{ij} = W^+_{ij}$ otherwise. Then $h(W'^+ + W'^-) > 0$ for all $\epsilon > 0$.

\subsubsection{Proof of Lemma~\ref{lem:simplePath}}
We provide a graphical proof by viewing $A$ as an adjacency matrix and $(\nabla h(A))_{ij} > 0$ as an indicator of a directed walk from node $j$ to $i$, the latter as ensured by Lemma~\ref{lem:gradh_ij}. If $(\nabla h(A))_{ij} > 0$, i.e.~there exists a directed walk from $j$ to $i$, then there also exists a directed path from $j$ to $i$. Since a directed path connects distinct vertices, it cannot contain an edge $(k,j)$. (Any directed walk from $j$ to $i$ that does have an edge $(k,j)$ must have a final subwalk from $j$ to $i$ that is a path.) Thus changing the values of $A_{kj}$, and specifically removing edges into $j$, cannot remove directed paths from $j$ to $i$ (and thereby set $(\nabla h(A))_{ij} = 0$). 

Similarly for the second statement, if $(\nabla h(A))_{ij} = 0$, then there is no directed walk from $j$ to $i$, including directed paths. Then changing the values of $A_{kj}$, and specifically adding edges into $j$, cannot create a directed walk from $j$ to $i$ because it would require a final subwalk from $j$ to $i$ that is a directed path, which was assumed not to exist. 

\subsubsection{Proof of Theorem~\ref{thm:necessary}}

By definition, $(W^+, W^-)$ is a feasible solution to \eqref{eqn:prob+-}. We prove that \eqref{eqn:KKTgrad} and \eqref{eqn:KKTcomplementary} can be satisfied. Again letting $A = W^+ + W^-$, we consider two cases for the entries of the constraint gradient $\nabla h(A)$.

\textbf{Case $(\nabla h(A))_{ij} = 0$}: In this case, the only way in which \eqref{eqn:KKTgrad} can be satisfied is if $(\nabla F(W^+ - W^-))_{ij} = 0$, and we show that this is indeed true. First we establish by a graphical argument that all $\bigl(\tilde{W}^+, \tilde{W}^-\bigr)$ of the form $\tilde{W}^+ = W^+ + w E^{ij}$ and $\tilde{W}^- = W^-$ are feasible solutions to \eqref{eqn:prob+-}, where $W^+_{ij} + w \geq 0$ to maintain non-negativity. The only potential obstacle is if $W^+_{ij} = W^-_{ij} = 0$ so that varying $w$ introduces an edge $(i,j)$. However, since $(\nabla h(A))_{ij} = 0$, there is no directed walk from $j$ to $i$, and Lemma~\ref{lem:simplePath} ensures that none can be created by varying $\tilde{W}^+_{ij}$. Therefore $\bigl(\tilde{W}^+, \tilde{W}^-\bigr)$ remains acyclic and feasible. The above argument can be repeated for $\tilde{W}^+ = W^+$ and $\tilde{W}^- = W^- + w E^{ij}$.

From the previous paragraph, we conclude that $W = W^+ - W^- + w E^{ij}$ is feasible for all $w \in \mathbb{R}$. Then if $(W^+, W^-)$ is a local minimum, we must have the partial derivative $(\nabla F(W^+ - W^-))_{ij} = 0$. Otherwise, entry $(i,j)$ could be increased or decreased ($w > 0$ or $w < 0$) to reduce the cost while remaining feasible.

Given that $(\nabla h(A))_{ij} = (\nabla F(W^+ - W^-))_{ij} = 0$, we take $M^+_{ij} = M^-_{ij} = 0$ to satisfy component $(i,j)$ of constraint \eqref{eqn:KKTcomplementary} as well as \eqref{eqn:KKTgrad}.

\textbf{Case $(\nabla h(A))_{ij} > 0$}: Since $(W^+, W^-)$ is feasible, $h(A) = 0$ and Lemma~\ref{lem:Agradh} implies that 
\[
(\nabla h(A))_{ij} A_{ij} = (\nabla h(A))_{ij} \bigl(W^+_{ij} + W^-_{ij}\bigr) = 0.
\]
Hence $W^+_{ij} = W^-_{ij} = 0$, satisfying \eqref{eqn:KKTcomplementary}.

To satisfy \eqref{eqn:KKTgrad}, we take 
\[
\lambda \geq \max_{(i,j): (\nabla h(A))_{ij} > 0} \frac{\abs*{\bigl(\nabla F(W^+ - W^-) \bigr)_{ij}}}{(\nabla h(A))_{ij}}
\]
and define $M^+_{ij}$, $M^-_{ij}$ to be the resulting slack in component $(i,j)$ of \eqref{eqn:KKTgrad}. This completes the proof.

The above proof is related to the idea discussed in \cite{bertsekas1999} that the directions of first-order feasible variations around $(W^+, W^-)$ do not include a direction of descent. The latter idea is used to prove existence of Lagrange multipliers in \cite[Prop.~3.3.14]{bertsekas1999}.

\subsection{Proofs for Section~\ref{sec:theory:abs:zeros}}

\subsubsection{Proof of Lemma~\ref{lem:KKTzeros}}
The proof follows from that of Theorem~\ref{thm:necessary}. Case $(\nabla h(A))_{ij} = 0$ in Theorem~\ref{thm:necessary} corresponds to $(i,j) \notin \cP$ and was shown to imply $(\nabla F(W^+ - W^-))_{ij} = (\nabla F(W^*))_{ij} = 0$. Case $(\nabla h(A))_{ij} > 0$ corresponds to $(i,j) \in \cP$ and implies $W^+_{ij} = W^-_{ij} = W^*_{ij} = 0$. If $F$ is convex, then conditions \eqref{eqn:optCondZeros} are also sufficient for optimality in \eqref{eqn:probZeros}. 

\subsubsection{Proof of Theorem~\ref{thm:KKTsuff}}
Let $W$ be a feasible solution to \eqref{eqn:probOrig} with $\norm{W - W^*}_F < \epsilon$ (the Frobenius norm is used for concreteness), $A = \abs{W}$, and $A^* = \abs{W^*}$. Since the gradient $\nabla h(A) = \sum_{p=1}^d p c_p \bigl(A^{p-1}\bigr)^T$ is a continuous function of $A$ and therefore of $W$, there exists a sufficiently small $\epsilon > 0$ such that $(\nabla h(A))_{ij} > 0$ wherever $(\nabla h(A^*))_{ij} > 0$, in other words for $(i,j)$ in the set $\cP$. Then for feasible $W$ within such an $\epsilon$-ball around $W^*$, it follows from Lemma~\ref{lem:Agradh} that $A_{ij} = W_{ij} = 0$ for $(i,j) \in \cP$. $W$ is therefore a feasible solution to \eqref{eqn:probZeros} for $\cZ = \cP$. By Lemma~\ref{lem:KKTzeros} and the convexity of $F$, we then have $F(W^*) \leq F(W)$ for all feasible $W$ such that $\norm{W - W^*}_F < \epsilon$.

\subsubsection{Proof of Theorem~\ref{thm:zerosKKT}}

By assumption, $W^*(\cZ)$ is feasible. 
For $(i,j) \in \cZ$, the constraint $W_{ij} = 0$ satisfies \eqref{eqn:KKTcomplementary}. Since $\cZ$ is irreducible, $\bigl(\nabla h(A^*(\cZ)) \bigr)_{ij} > 0$. We may then choose $\lambda$ large enough as in the proof of Theorem~\ref{thm:necessary} to satisfy \eqref{eqn:KKTgrad}. 
    
For $(i,j) \notin \cZ$, the optimality conditions \eqref{eqn:optCondZeros} imply $\bigl(\nabla F(W^*(\cZ))\bigr)_{ij} = 0$. If $(W^*(\cZ))_{ij} \neq 0$, then we must have $\bigl(\nabla h(A^*(\cZ)) \bigr)_{ij} = 0$ by the feasibility of $W^*(\cZ)$ and Lemma~\ref{lem:Agradh}. If $(W^*(\cZ))_{ij} = 0$, then the convention in defining $\cZ$ also ensures that $\bigl(\nabla h(A^*(\cZ)) \bigr)_{ij} = 0$. Letting $M^+_{ij} = M^-_{ij} = 0$ then satisfies \eqref{eqn:KKTgrad} and \eqref{eqn:KKTcomplementary}.

\subsubsection{Proof of Lemma~\ref{lem:reduceZ}}

Since $F(W)$ is separable and the pairs in $\cZ_0(j)$ have $j$ in common, removing the constraints $W_{ij} = 0$ for $(i,j) \in \cZ_0(j)$ affects only the subproblem of \eqref{eqn:probZeros} for node $j$. This subproblem is now given by 
\begin{equation}\label{eqn:reduceZ}
\argmin_{W_{\cdot j}} \; F_{j}\bigl(W_{\cdot j}\bigr) \quad \text{s.t.} \quad W_{ij} = 0, \quad (i,j) \in \cZ \backslash \cZ_0(j).
\end{equation}
By the definitions of $\cZ$ and $\cZ_0(j)$, we have $\bigl(\nabla h(A^*(\cZ)) \bigr)_{ij} = 0$ for $(i,j) \notin \cZ\backslash \cZ_0(j)$, i.e.~there are no directed walks from $j$ to such $i$. From Lemma~\ref{lem:simplePath}, it follows that re-optimizing the values of $W_{ij}$, $(i,j) \notin \cZ\backslash \cZ_0(j)$ in \eqref{eqn:reduceZ} cannot create directed walks from $j$ to $i$. For $(i,j) \in \cZ\backslash \cZ_0(j)$, $W_{ij}$ is constrained to zero. We conclude that re-solving \eqref{eqn:reduceZ} does not introduce new cycles. 

\subsubsection{Proof of Theorem~\ref{thm:algZC}}
The first while-loop adds more and more elements to $\cZ$, i.e.~constrains more and more edges to be absent, and is hence guaranteed to eventually produce a feasible (acyclic) solution $W^*(\cZ)$. If the resulting set $\cZ$ is not irreducible, then repeated application of Lemma~\ref{lem:reduceZ} in the second while-loop will make it so while maintaining feasibility. The algorithm thus yields a solution satisfying the conditions of Theorem~\ref{thm:zerosKKT}.

\section{Modified LARS algorithms}
\label{sec:algsAdd}

\subsection{Adding zero-value constraints}
\label{sec:algsAdd:setZero}

This appendix describes a modification of the LARS algorithm \cite{efron2004} to efficiently re-solve problem \eqref{eqn:probZeros} under the following conditions: a) the score function $F(W)$ is given by \eqref{eqn:scoreL1}, b) the loss function $\ell(W; \bX)$ is the least-squares loss, $\ell(W; \bX) = \frac{1}{2n} \norm{\bX - \bX W}_F^2$, and c) we have an optimal solution $W^*(\cZ)$ for the existing set $\cZ$ of zero-value constraints and a new pair $(i_0, j)$ is being added to $\cZ$. 

Given conditions a) and b), $F(W)$ is separable column-wise and hence we only have to re-solve the subproblem of \eqref{eqn:probZeros} for column $j$. Define $\cZ^c(j) = \{i: (i,j) \notin \cZ\}$ to be the set of rows in column $j$ that are not constrained to zero by $\cZ$. Then the subproblem for column $j$ can be written as 
\[
\min_{W_{\cZ^c(j), j}} \quad \frac{1}{2n} \norm*{\bX_{\cdot j} - \bX_{\cdot \cZ^c(j)} W_{\cZ^c(j), j}}_2^2 + \tau \norm*{W_{\cZ^c(j), j}}_1,
\]
to which we wish to add the constraint $W_{i_0 j} = 0$. To simplify notation, let $y = \bX_{\cdot j}$, $\tilde{\bX} = \bX_{\cdot \cZ^c(j)}$, and $w = W_{\cZ^c(j), j}$. Our approach is to add a penalty $\alpha \abs{w_{i_0}}$ to the objective function, giving
\begin{equation}\label{eqn:probSetZero}
\min_{w} \quad \frac{1}{2n} \norm[\big]{y - \tilde{\bX} w}_2^2 + \tau \norm{w}_1 + \alpha \abs{w_{i_0}},
\end{equation}
and increase $\alpha$ from zero until we obtain $w_{i_0} = 0$.

LARS is an active-set algorithm, where the active set $\cA$ corresponds to the set of non-zero $w_i$, i.e.~$\cA = \{i: w_i \neq 0\}$. The initial active set is given by the existing optimal solution $W_{\cZ^c(j), j}^*(\cZ)$. We assume that it includes $i_0$, as otherwise $w_{i_0} = 0$ and we are done.

In each iteration of LARS, the active elements of $w$ are updated as 
\begin{equation}\label{eqn:setZero_w}
w_{\cA} \gets w_{\cA} - \gamma d,
\end{equation}
where $\gamma$ is the step size and $d$ is an $\abs{\cA}$-dimensional direction vector determined below. The step size $\gamma$ will be made equal to the increase in $\alpha$ and is chosen to be the largest possible before a change in the active set occurs. 

One set of conditions on $\gamma$ and $d$ comes from maintaining the optimality of $w$. Define
\begin{equation}\label{eqn:g}
g = \frac{1}{n} \tilde{\bX}^T \bigl(y - \tilde{\bX} w\bigr) = \frac{1}{n} \tilde{\bX}^T \bigl(y - \tilde{\bX}_{\cdot\cA} w_{\cA}\bigr)
\end{equation}
to be the negative gradient of the least-squares term in \eqref{eqn:probSetZero}, where the second equality is due to $w_i$ being zero for $i \notin \cA$. The update equation for $w$ \eqref{eqn:setZero_w} implies that the gradient changes as 
\begin{equation}\label{eqn:setZero_g}
g \gets g + \gamma c,
\end{equation}
where
\begin{equation}\label{eqn:setZero_c}
c = \frac{1}{n} \tilde{\bX}^T \tilde{\bX}_{\cdot\cA} d.
\end{equation}
The optimality conditions of \eqref{eqn:probSetZero} for $i \in \cA$ require
\begin{equation}\label{eqn:setZeroOptCondA}
g_i + \gamma c_i = \begin{cases}
\sign(w_i) \tau, & i \in \cA, \; i \neq i_0,\\
\sign(w_i) (\tau + \alpha + \gamma), & i = i_0,
\end{cases}
\end{equation}
where $\alpha$ is increased by $\gamma$ as mentioned. Defining $e^{i_0}$ to be the $\abs{\cA}$-dimensional standard basis vector with $e^{i_0}_{i_0} = 1$ and $e^{i_0}_i = 0$ otherwise, we must have $c_{\cA} = \sign(w_{i_0}) e^{i_0}$ from \eqref{eqn:setZeroOptCondA}. This in combination with \eqref{eqn:setZero_c} determines the direction $d$:
\begin{equation}\label{eqn:setZero_d}
d = \left( \frac{1}{n} \tilde{\bX}_{\cA}^T \tilde{\bX}_{\cdot\cA} \right)^{-1} \sign(w_{i_0}) e^{i_0}.
\end{equation}

To determine the step size $\gamma$, we consider the optimality conditions for $i \notin \cA$, namely $\abs{g_i + \gamma c_i} \leq \tau$. By expanding the absolute value function and disregarding one of the cases because it is always satisfied, we obtain 
\begin{equation}\label{eqn:setZero_gammaAc}
\gamma \leq \frac{\tau - \sign(c_i) g_i}{\abs{c_i}}, \quad i \notin \cA.
\end{equation}
We also have the constraints $w_i - \gamma d_i \neq 0$ to maintain the current active set, which imply 
\begin{equation}\label{eqn:setZero_gammaA}
\gamma \leq \frac{w_i}{d_i}, \quad i \in \cA: \frac{w_i}{d_i} > 0,
\end{equation}
where the constraint is never binding if $w_i / d_i < 0$. Combining \eqref{eqn:setZero_gammaAc} and \eqref{eqn:setZero_gammaA} yields 
\begin{equation}\label{eqn:setZero_gamma}
\gamma = \min\left\{ \min_{i\in\cA: w_i/d_i > 0} \frac{w_i}{d_i}, \;\; \min_{i\notin\cA} \frac{\tau - \sign(c_i) g_i}{\abs{c_i}} \right\}.
\end{equation}
Let $i^*$ denote the minimizing index in \eqref{eqn:setZero_gamma}. The active set is updated as 
\begin{equation}\label{eqn:setZero_A}
\cA \gets \begin{cases}
\cA \backslash \{i^*\}, & i^* \in \cA,\\
\cA \cup \{i^*\}, & i^* \notin \cA.
\end{cases}
\end{equation}

Equations \eqref{eqn:setZero_w}, \eqref{eqn:setZero_g}, \eqref{eqn:setZero_c}, \eqref{eqn:setZero_d}, \eqref{eqn:setZero_gamma}, and \eqref{eqn:setZero_A} define one iteration of the LARS algorithm. The algorithm terminates with $w_{i_0} = 0$ when $i_0$ leaves the active set.

\subsection{Relaxing zero-value constraints}
\label{sec:algsAdd:relaxZero}

We now discuss the use of the LARS algorithm to re-solve problem \eqref{eqn:probZeros} after a pair $(i_0, j)$ is removed from the set $\cZ$ of zero-value constraints. Other assumptions remain as in Appendix~\ref{sec:algsAdd:setZero}, and thus we again only have to re-solve the subproblem of \eqref{eqn:probZeros} for column $j$. Recalling the definition of $\cZ^c(j)$ from Appendix~\ref{sec:algsAdd:setZero} and defining $\tilde{\cZ}^c(j) = \cZ^c(j) \cup \{i_0\}$, the subproblem for column $j$ can be expressed as  
\begin{equation}\label{eqn:probRelaxZero}
\min_{W_{\tilde{\cZ}^c(j), j}} \quad \frac{1}{2n} \norm*{\bX_{\cdot j} - \bX_{\cdot \tilde{\cZ}^c(j)} W_{\tilde{\cZ}^c(j), j}}_2^2 + \tau \norm*{W_{\tilde{\cZ}^c(j) j}}_1 \quad \text{s.t.} \quad W_{i_0 j} = 0,
\end{equation}
where we wish to relax the constraint $W_{i_0 j} = 0$.

To simplify notation as before, let $y = \bX_{\cdot j}$, $\tilde{\bX} = \bX_{\cdot \tilde{\cZ}^c(j)}$, and $w = W_{\tilde{\cZ}^c(j), j}$. We show that problem \eqref{eqn:probRelaxZero} is equivalent to \eqref{eqn:probSetZero} for a sufficiently large penalty $\alpha$. Let $w^* = W^*_{\tilde{\cZ}^c(j), j}(\cZ)$ denote the existing optimal solution of subproblem $j$, and $g^*$ be the corresponding negative loss gradient from \eqref{eqn:g}. Then the optimality conditions for \eqref{eqn:probSetZero} imply that $w_{i_0} = 0$ if the loss gradient satisfies $\abs{g^*_{i_0}} < \tau + \alpha$. Therefore $\alpha = \abs{g^*_{i_0}} - \tau$ is the first value at which $w_{i_0}$ becomes active. If $\abs{g^*_{i_0}} - \tau$ is non-positive, i.e.~$\abs{g^*_{i_0}} \leq \tau$, then relaxing the constraint $w_{i_0} = 0$ does not change $w_{i_0}$ as $w^*$ is still optimal without the constraint. In this case, we are done. Assuming therefore that $\abs{g^*_{i_0}} - \tau > 0$, we initialize $\alpha = \abs{g^*_{i_0}} - \tau$ and seek to decrease $\alpha$ to zero.

Given this initial value for $\alpha$, the modified LARS algorithm proceeds in the reverse direction of that in Appendix~\ref{sec:algsAdd:setZero}. In each iteration, we update 
\begin{align}
    w_{\cA} &\gets w_{\cA} + \gamma d,\label{eqn:relaxZero_w}\\
    g &\gets g - \gamma c,\label{eqn:relaxZero_g}\\
    \alpha &\gets \alpha - \gamma,\label{eqn:relaxZero_alpha}
\end{align}
where $d$ and $c$ are still given by \eqref{eqn:setZero_d} and \eqref{eqn:setZero_c}, except that when $w_{i_0}$ is still zero, we use $\sign(g^*_{i_0})$ in place of $\sign(w_{i_0})$ in \eqref{eqn:setZero_d} (\cite{efron2004} shows that these two signs must agree). The determination of the step size $\gamma$ is slightly modified from that in \eqref{eqn:setZero_gamma} because of the change in signs in \eqref{eqn:relaxZero_w}, \eqref{eqn:relaxZero_g} relative to \eqref{eqn:setZero_w}, \eqref{eqn:setZero_g}:
\begin{equation}\label{eqn:relaxZero_gamma}
\gamma = \min\left\{ \min_{i\in\cA: w_i/d_i < 0} -\frac{w_i}{d_i}, \;\; \min_{i\notin\cA} \frac{\tau + \sign(c_i) g_i}{\abs{c_i}} \right\}.
\end{equation}
The update for the active set $\cA$ remains as in \eqref{eqn:setZero_A}. 

In summary, each LARS iteration is defined by \eqref{eqn:relaxZero_w}--\eqref{eqn:relaxZero_gamma}, \eqref{eqn:setZero_d}, and \eqref{eqn:setZero_c}. As mentioned, the algorithm terminates when $\alpha$ decreases to zero.

\subsection{Solution path of \eqref{eqn:selectRemove}}
\label{sec:algsAdd:path}

The LARS algorithm can also be adapted to compute the solution path of problem \eqref{eqn:selectRemove} as the penalty parameter $\alpha$ increases from zero. This adaptation differs from the one in Appendix~\ref{sec:algsAdd:setZero} in two respects: First, \eqref{eqn:selectRemove} involves updates to the entire matrix $W$, with a common step size $\gamma$, and not just to a single column. At the same time, assumptions a) and b) in Appendix~\ref{sec:algsAdd:setZero} remain in effect, allowing the computation of update directions to be done in a separable manner. Second, \eqref{eqn:selectRemove} includes a weighted $\ell_1$ penalty with weight matrix $\nabla h(A^*(\cZ))$ instead of an unweighted $\ell_1$ penalty plus an additional penalty on a single element $w_{i_0}$. To ease notation, let $P = \nabla h(A^*(\cZ))$.

As in Appendix~\ref{sec:algsAdd:setZero}, in each iteration, $\alpha$ is increased by $\gamma$, 
\begin{equation}\label{eqn:selectRemove_alpha}
\alpha \gets \alpha + \gamma,
\end{equation}
and other quantities are updated accordingly. Equations~\eqref{eqn:setZero_w} and \eqref{eqn:setZero_g} are generalized to matrices as follows:
\begin{align}
W_{\cA} &\gets W_{\cA} - \gamma D_{\cA},\label{eqn:selectRemove_W}\\
G &\gets G + \gamma C,\label{eqn:selectRemove_G}
\end{align}
where the active set $\cA = \{(i,j): W_{ij} \neq 0\}$ is now a set of pairs, $D_{ij} = 0$ for $(i,j) \notin \cA$, and 
\begin{equation}\label{eqn:G}
G = \frac{1}{n} \bX^T (\bX - \bX W)
\end{equation}
is the negative loss gradient matrix. From \eqref{eqn:selectRemove_W}--\eqref{eqn:G}, it can be seen that 
\begin{equation}\label{eqn:selectRemove_C}
C = \frac{1}{n} \bX^T \bX D.
\end{equation}

To determine $D_{ij}$ for $(i,j) \in \cA$, we use the corresponding optimality conditions for \eqref{eqn:selectRemove}:
\begin{equation}\label{eqn:selectRemoveOptCondA}
G_{ij} + \gamma C_{ij} = \sign(W_{ij}) \left(\tau + (\alpha + \gamma) P_{ij} \right), \quad (i,j) \in \cA.
\end{equation}
Define $\cA(j)$ to be the set of active elements in column $j$. By combining \eqref{eqn:selectRemove_C} and \eqref{eqn:selectRemoveOptCondA} and considering each column $j$ separately, we obtain
\[
\sign\bigl(W_{\cA(j),j}\bigr) \circ P_{\cA(j),j} = C_{\cA(j),j} = \frac{1}{n} \bX_{\cdot\cA(j)}^T \bX D_{\cdot j} = \frac{1}{n} \bX_{\cdot\cA(j)}^T \bX_{\cdot\cA(j)} D_{\cA(j),j},
\]
where the last inequality follows because $D_{ij} = 0$, $(i,j) \notin \cA(j)$. Hence 
\begin{equation}\label{eqn:selectRemove_D}
D_{\cA(j),j} = \left(\frac{1}{n} \bX_{\cdot\cA(j)}^T \bX_{\cdot\cA(j)} \right)^{-1} \left( \sign\bigl(W_{\cA(j),j}\bigr) \circ P_{\cA(j),j} \right), \quad j \in \cV.
\end{equation}

To determine the step size $\gamma$, we consider the optimality conditions for $(i,j) \in \cZ^c \backslash \cA$, i.e.
\[
\abs[\big]{G_{ij} + \gamma C_{ij}} \leq \tau + (\alpha + \gamma) P_{ij}.
\]
Similar to Appendix~\ref{sec:algsAdd:setZero}, this can be reduced to the following upper bound on $\gamma$:
\begin{equation}\label{eqn:selectRemove_gammaAc}
\gamma \leq \frac{\tau + \alpha P_{ij} - \sign(C_{ij}) G_{ij}}{\abs{C_{ij}} - P_{ij}}, \quad (i,j) \in \cZ^c \backslash \cA: \abs{C_{ij}} > P_{ij},
\end{equation}
whereas no bound is imposed if $\abs{C_{ij}} \leq P_{ij}$. We also have the conditions $W_{ij} - \gamma D_{ij} \neq 0$ for $(i,j) \in \cA$. Define $\Gamma$ as the resulting matrix of upper bounds, 
\begin{equation}\label{eqn:selectRemove_Gamma1}
\Gamma_{ij} = \begin{cases}
\dfrac{W_{ij}}{D_{ij}}, & (i,j) \in \cA: \dfrac{W_{ij}}{D_{ij}} > 0,\\
\dfrac{\tau + \alpha P_{ij} - \sign(C_{ij}) G_{ij}}{\abs{C_{ij}} - P_{ij}}, & (i,j) \in \cZ^c \backslash \cA: \abs{C_{ij}} > P_{ij},\\
+\infty & \text{otherwise.}
\end{cases}
\end{equation}
Then we have 
\begin{equation}\label{eqn:selectRemove_gamma}
\gamma = \min_{i,j} \Gamma_{i,j},
\end{equation}
and given the minimizing pair $(i^*,j^*)$ from \eqref{eqn:selectRemove_gamma}, we update the active set as 
\begin{equation}\label{eqn:selectRemove_A}
\cA \gets \begin{cases}
\cA \backslash \{(i^*, j^*)\}, & (i^*, j^*) \in \cA,\\
\cA \cup \{(i^*, j^*)\}, & (i^*, j^*) \in \cZ^c \backslash \cA.
\end{cases}
\end{equation}

The update to the active set \eqref{eqn:selectRemove_A} affects only $\cA(j^*)$ in column $j^*$. We may take advantage of this by updating only column $j^*$ of $D$ and $C$, i.e.~computing \eqref{eqn:selectRemove_D} for $j = j^*$ and \eqref{eqn:selectRemove_C} only for column $j^*$. The other columns are unchanged. Similarly, the upper bounds $\Gamma_{ij}$ are recomputed using \eqref{eqn:selectRemove_Gamma1} only for $j = j^*$. For columns other than $j^*$, it suffices to subtract the previous step size:
\begin{equation}\label{eqn:selectRemove_Gamma2}
\Gamma_{ij} \gets \Gamma_{ij} - \gamma, \quad j \neq j^*.
\end{equation}

In summary, each iteration of the modified LARS algorithm is given by \eqref{eqn:selectRemove_alpha}--\eqref{eqn:selectRemove_G}, \eqref{eqn:selectRemove_D}, \eqref{eqn:selectRemove_C}, \eqref{eqn:selectRemove_Gamma1}--\eqref{eqn:selectRemove_Gamma2}, together with the simplification noted in the previous paragraph. The algorithm terminates as soon as $(i^*, j^*)$ coincides with an edge belonging to a cycle in the existing optimal solution $W^*(\cZ)$, i.e.~$(i^*, j^*)$ such that $W^*_{i^*j^*}(\cZ) \neq 0$ and $P_{i^*j^*} > 0$.

\section{Additional experimental details and results}
\label{sec:exptAdd}

\subsection{Algorithm parameter settings}
\label{sec:exptAdd:params}

Parameter settings for all algorithms are shown in Table~\ref{tab:params}. We use the least-squares loss $\ell(W; \bX) = \frac{1}{2n} \norm{\bX - \bX W}_F^2$ regardless of the noise type. We found the polynomial acyclicity penalty $h(A) = \tr\bigl((I + A/d)^d\bigr) - d$ from \cite{yu2019dag} to take less time and perform slightly better than the exponential penalty $h(A) = \tr\bigl(e^A\bigr) - d$ from \cite{zheng2018dags} (polynomial is now also the default in the NOTEARS code). Similarly, we preferred a tolerance on $h$ of $\epsilon = 10^{-10}$ compared to $\epsilon = 10^{-8}$ in \cite{zheng2018dags}. We did not attempt to tune other parameters.

\begin{table}[ht]
\caption{Algorithm parameter settings}
\label{tab:params}
\centering
\begin{tabular}{llll}
\toprule
parameter & symbol & value & applicable to \\
\midrule
threshold on $W$ & $\omega$ \cite{zheng2018dags} & $0.3$ & NOTEARS, Abs, \\
& & & KKTS before and after \\
loss function & $\ell(W; \bX)$ & $\frac{1}{2n} \norm{\bX - \bX W}_F^2$ & NOTEARS, Abs, KKTS \\
$\ell_1$ penalty parameter & $\tau$ & $0.1$ & NOTEARS, Abs, KKTS \\
acyclity penalty & $h(A)$ & $\tr\bigl((I + A/d)^d\bigr) - d$ & NOTEARS, Abs, KKTS \\
$h$ tolerance & $\epsilon$ \cite{zheng2018dags} & $10^{-10}$ & NOTEARS, Abs, KKTS \\
$h$ progress rate & $c$ \cite{zheng2018dags} & $0.25$ & NOTEARS, Abs \\
initial solution & $W_0$ \cite{zheng2018dags} & $0$ & NOTEARS, Abs \\
initial Lagrange multiplier & $\alpha_0$ \cite{zheng2018dags} & $0$ & NOTEARS, Abs \\
$\rho$ increase factor & & $10$ & NOTEARS, Abs \\
$\rho$ maximum & & $10^{16}$ & NOTEARS, Abs \\
variablesel & & True & CAM \cite{buhlmann2014cam} \\
pruning & & True & CAM \cite{buhlmann2014cam} \\
\bottomrule
\end{tabular}
\end{table}

For baseline method causal additive models (CAM), we use Causal Discovery Toolbox (CDT) \cite{kalainathan2019causal} in Python and only tuned two input parameters,
``variablesel" and ``pruning". We found with both turned on, the results are the best. 

For baseline method fast greedy equivalent search (FGS), we use py-causal package\footnote{\url{https://github.com/bd2kccd/py-causal}} in Python from Carnegie Mellon University. We use the default parameter settings and did not tune any. 

For PC, we also used CDT, and for MMHC, we used the bnlearn package \cite{scutari2010bnlearn} in R by adapting CDT's interface for calling other bnlearn algorithms. The main parameter for both PC and MMHC is the significance level $\alpha$ for the conditional independence tests that they conduct. While we considered the same range of $\alpha$ values as in \cite{aragam2015concave}, we found $\alpha = 0.01$ or $\alpha = 0.05$ to be the best in all cases. The differences between $\alpha = 0.01$ and $\alpha = 0.05$ are not large, and in any case, PC and MMHC are not the most competitive algorithms in our experiments.

\subsection{Computing environment}

Solution times were obtained using a single $2.0$ GHz core of a server with $64$ GB of memory (only a small fraction of which was used) running Ubuntu 16.04 (64-bit). The limitation to a single core was done to control for different multi-threading behavior of different algorithms and for different dimensions $d$.

\subsection{Effect of mean subtraction}
\label{sec:exptAdd:meanSub}

\begin{figure}[ht]
  \centering
  \includegraphics[width=0.8\linewidth]{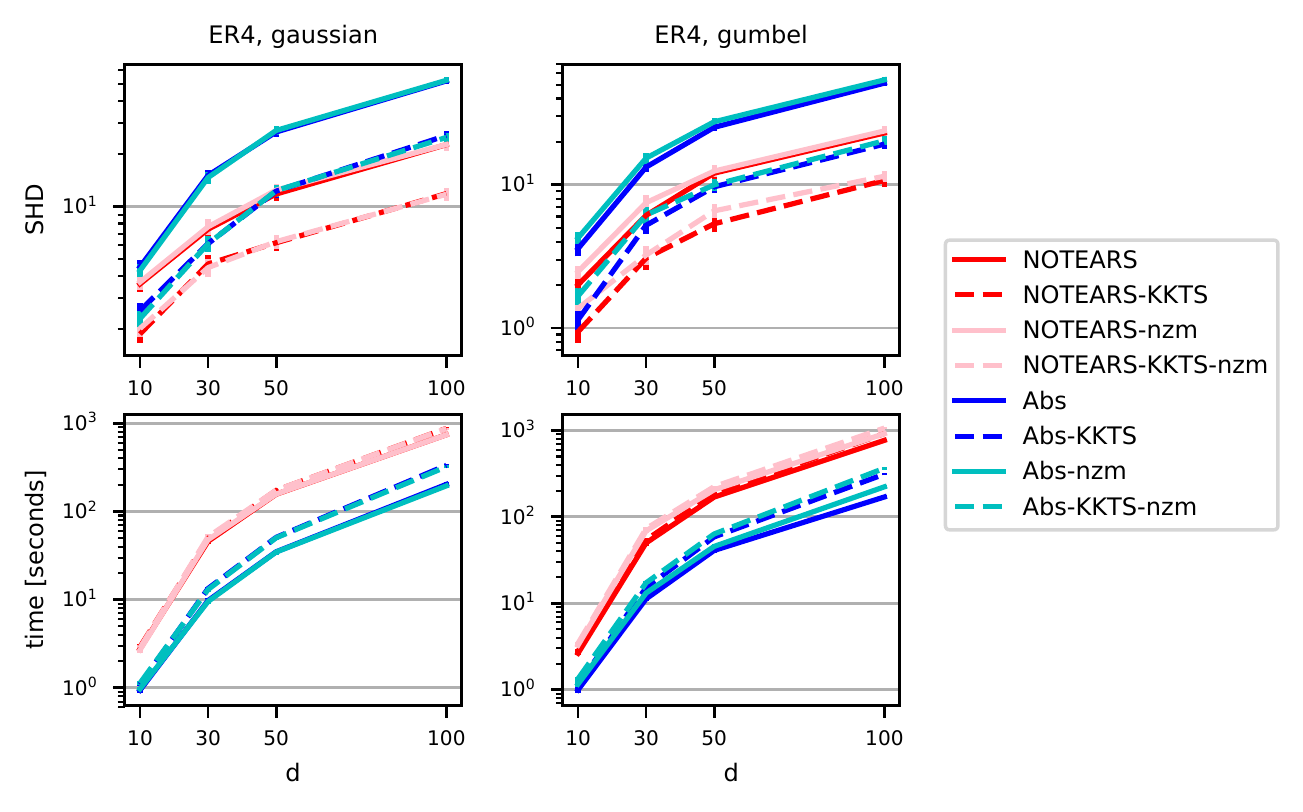}
  \caption{Effect of mean subtraction (`nzm' means nonzero mean) on SHD and solution times.}
  \label{fig:nzm}
\end{figure}

We show the effect of subtracting the mean from the data $\bX$ as a preprocessing step in Figure~\ref{fig:nzm}. Tables~\ref{tab:ER4_gaussian_nzm} and \ref{tab:ER4_gumbel_nzm} present the same results in tabular form. As one may see, subtracting the mean improves the SHD in the ER4 Gumbel case for all the methods shown and slightly decreases the running time. Mean subtraction has less effect in the Gaussian case. In our experience, subtracting the mean improves results or at least does not hurt in all the cases we studied, not just the ones shown in Figure~\ref{fig:nzm}.

\begin{table}[ht]
\caption{Effect of mean subtraction (`nzm' means nonzero mean) on ER4 graphs with Gaussian noise}
\label{tab:ER4_gaussian_nzm}
\centering
\small
\begin{tabular}{lrrrrrr}
\toprule
& \multicolumn{3}{c}{$d = 10$} & \multicolumn{3}{c}{$d = 30$} \\
\cmidrule(r){2-4} \cmidrule(l){5-7}
& SHD & nnz & time (sec) & SHD & nnz & time (sec) \\
\midrule
NOTEARS-nzm&3.70$\pm$0.36&18.05$\pm$0.29&2.7$\pm$0.2&7.66$\pm$0.81&58.11$\pm$0.76&47.9$\pm$3.5\\
NOTEARS&3.61$\pm$0.36&18.08$\pm$0.29&2.8$\pm$0.2&7.42$\pm$0.81&57.97$\pm$0.74&45.9$\pm$3.4\\
NOTEARS-KKTS-nzm&2.01$\pm$0.22&18.32$\pm$0.30&2.8$\pm$0.2&4.48$\pm$0.54&58.27$\pm$0.72&51.6$\pm$3.5\\
NOTEARS-KKTS&1.87$\pm$0.20&18.37$\pm$0.30&2.9$\pm$0.2&4.70$\pm$0.58&58.18$\pm$0.72&49.5$\pm$3.4\\
Abs-nzm&4.31$\pm$0.36&18.31$\pm$0.30&1.0$\pm$0.1&14.64$\pm$1.10&60.18$\pm$0.80&9.5$\pm$0.7\\
Abs&4.52$\pm$0.41&18.16$\pm$0.30&0.9$\pm$0.1&15.06$\pm$1.06&60.39$\pm$0.81&9.8$\pm$0.6\\
Abs-KKTS-nzm&2.29$\pm$0.25&18.29$\pm$0.31&1.1$\pm$0.1&6.13$\pm$0.64&58.31$\pm$0.71&13.0$\pm$0.7\\
Abs-KKTS&2.54$\pm$0.27&18.19$\pm$0.31&1.1$\pm$0.1&6.14$\pm$0.56&58.23$\pm$0.72&13.4$\pm$0.6\\
FGS-nzm&13.48$\pm$0.74&28.44$\pm$0.81&0.5$\pm$0.0&53.21$\pm$3.30&118.38$\pm$4.48&1.3$\pm$0.1\\
FGS&13.48$\pm$0.74&28.44$\pm$0.81&0.5$\pm$0.0&53.21$\pm$3.30&118.38$\pm$4.48&1.3$\pm$0.1\\
FGS-KKTS-nzm&5.26$\pm$0.55&17.57$\pm$0.30&0.7$\pm$0.0&15.50$\pm$1.34&59.73$\pm$0.86&4.4$\pm$0.1\\
FGS-KKTS&4.92$\pm$0.47&17.79$\pm$0.30&0.7$\pm$0.0&15.03$\pm$1.38&59.72$\pm$0.87&4.5$\pm$0.1\\
\midrule[\heavyrulewidth]
& \multicolumn{3}{c}{$d = 50$} & \multicolumn{3}{c}{$d = 100$} \\
\cmidrule(r){2-4} \cmidrule(l){5-7}
& SHD & nnz & time (sec) & SHD & nnz & time (sec) \\
\midrule
NOTEARS-nzm&12.41$\pm$1.13&99.59$\pm$1.06&157.5$\pm$7.7&22.61$\pm$1.73&199.49$\pm$1.53&739.9$\pm$23.2\\
NOTEARS&11.79$\pm$1.05&99.69$\pm$1.06&156.8$\pm$7.7&22.57$\pm$1.74&199.65$\pm$1.52&741.0$\pm$23.0\\
NOTEARS-KKTS-nzm&6.30$\pm$0.62&99.04$\pm$0.96&174.2$\pm$7.7&11.75$\pm$0.96&198.70$\pm$1.45&871.8$\pm$23.5\\
NOTEARS-KKTS&6.21$\pm$0.64&99.07$\pm$0.94&173.5$\pm$7.7&11.85$\pm$0.96&198.76$\pm$1.46&874.3$\pm$23.4\\
Abs-nzm&27.20$\pm$1.65&104.05$\pm$1.24&34.7$\pm$2.2&52.60$\pm$2.09&209.73$\pm$1.85&195.2$\pm$12.2\\
Abs&26.67$\pm$1.60&103.77$\pm$1.20&34.9$\pm$2.0&52.18$\pm$1.89&208.66$\pm$1.62&202.7$\pm$12.4\\
Abs-KKTS-nzm&12.33$\pm$1.03&99.78$\pm$0.96&50.3$\pm$2.2&24.82$\pm$1.53&201.44$\pm$1.57&321.5$\pm$12.5\\
Abs-KKTS&12.25$\pm$1.04&99.46$\pm$0.95&50.9$\pm$2.0&25.45$\pm$1.40&201.45$\pm$1.54&334.0$\pm$12.7\\
FGS-nzm&83.28$\pm$5.61&196.78$\pm$8.01&2.6$\pm$0.2&114.07$\pm$8.35&321.52$\pm$11.14&5.1$\pm$0.5\\
FGS&83.28$\pm$5.61&196.78$\pm$8.01&2.5$\pm$0.2&114.07$\pm$8.35&321.52$\pm$11.14&5.1$\pm$0.5\\
FGS-KKTS-nzm&19.58$\pm$1.78&102.32$\pm$1.23&16.6$\pm$0.3&36.74$\pm$3.62&208.68$\pm$2.22&124.7$\pm$2.4\\
FGS-KKTS&20.31$\pm$1.73&102.31$\pm$1.15&15.7$\pm$0.3&36.89$\pm$3.74&208.55$\pm$2.18&122.1$\pm$2.5\\
\bottomrule
\end{tabular}
\end{table}

\begin{table}[ht]
\caption{Effect of mean subtraction (`nzm' means nonzero mean) on ER4 graphs with Gumbel noise}
\label{tab:ER4_gumbel_nzm}
\centering
\small
\begin{tabular}{lrrrrrr}
\toprule
& \multicolumn{3}{c}{$d = 10$} & \multicolumn{3}{c}{$d = 30$} \\
\cmidrule(r){2-4} \cmidrule(l){5-7}
& SHD & nnz & time (sec) & SHD & nnz & time (sec) \\
\midrule
NOTEARS-nzm&2.47$\pm$0.26&19.11$\pm$0.31&3.3$\pm$0.1&7.49$\pm$0.99&60.47$\pm$0.79&68.9$\pm$3.2\\
NOTEARS&2.00$\pm$0.26&19.24$\pm$0.32&2.6$\pm$0.1&6.11$\pm$0.89&60.59$\pm$0.76&49.9$\pm$3.6\\
NOTEARS-KKTS-nzm&1.37$\pm$0.16&19.33$\pm$0.32&3.4$\pm$0.1&3.23$\pm$0.49&59.97$\pm$0.73&73.0$\pm$3.2\\
NOTEARS-KKTS&0.94$\pm$0.15&19.42$\pm$0.30&2.8$\pm$0.1&3.07$\pm$0.54&60.47$\pm$0.72&54.1$\pm$3.6\\
Abs-nzm&4.25$\pm$0.42&19.52$\pm$0.35&1.2$\pm$0.1&15.33$\pm$1.28&64.34$\pm$0.96&13.1$\pm$0.8\\
Abs&3.58$\pm$0.42&19.55$\pm$0.35&1.0$\pm$0.1&13.27$\pm$1.07&63.52$\pm$0.87&11.2$\pm$0.7\\
Abs-KKTS-nzm&1.67$\pm$0.22&19.30$\pm$0.32&1.3$\pm$0.1&6.15$\pm$0.83&60.90$\pm$0.75&17.2$\pm$0.8\\
Abs-KKTS&1.14$\pm$0.18&19.36$\pm$0.31&1.1$\pm$0.1&5.21$\pm$0.68&60.87$\pm$0.76&15.2$\pm$0.7\\
FGS-nzm&12.29$\pm$0.66&27.85$\pm$0.83&0.5$\pm$0.0&53.42$\pm$3.56&119.62$\pm$4.82&1.3$\pm$0.1\\
FGS&12.29$\pm$0.66&27.85$\pm$0.83&0.5$\pm$0.0&53.42$\pm$3.56&119.62$\pm$4.82&1.3$\pm$0.1\\
FGS-KKTS-nzm&3.97$\pm$0.46&18.75$\pm$0.33&0.7$\pm$0.0&14.75$\pm$1.48&62.64$\pm$0.92&4.9$\pm$0.1\\
FGS-KKTS&3.58$\pm$0.45&19.12$\pm$0.31&0.7$\pm$0.0&12.36$\pm$1.38&62.05$\pm$0.85&4.8$\pm$0.1\\
\midrule[\heavyrulewidth]
& \multicolumn{3}{c}{$d = 50$} & \multicolumn{3}{c}{$d = 100$} \\
\cmidrule(r){2-4} \cmidrule(l){5-7}
& SHD & nnz & time (sec) & SHD & nnz & time (sec) \\
\midrule
NOTEARS-nzm&12.45$\pm$1.19&100.58$\pm$1.16&205.7$\pm$6.4&23.73$\pm$1.90&201.65$\pm$1.63&918.6$\pm$18.2\\
NOTEARS&12.05$\pm$1.28&101.03$\pm$1.20&169.4$\pm$7.6&22.97$\pm$1.92&202.58$\pm$1.61&768.4$\pm$19.5\\
NOTEARS-KKTS-nzm&6.57$\pm$0.76&100.50$\pm$1.13&223.8$\pm$6.5&11.46$\pm$1.09&200.99$\pm$1.45&1066.1$\pm$18.5\\
NOTEARS-KKTS&5.34$\pm$0.69&100.16$\pm$1.09&187.8$\pm$7.7&10.71$\pm$1.08&201.68$\pm$1.41&922.2$\pm$19.6\\
Abs-nzm&27.49$\pm$1.67&108.09$\pm$1.43&45.5$\pm$3.1&53.92$\pm$2.23&217.77$\pm$1.86&222.6$\pm$12.9\\
Abs&25.16$\pm$1.62&107.75$\pm$1.41&40.9$\pm$2.7&51.18$\pm$2.19&216.54$\pm$1.88&169.8$\pm$9.7\\
Abs-KKTS-nzm&9.98$\pm$0.91&101.67$\pm$1.15&63.4$\pm$3.1&20.28$\pm$1.47&204.09$\pm$1.70&365.7$\pm$12.8\\
Abs-KKTS&9.67$\pm$0.92&101.62$\pm$1.15&58.6$\pm$2.7&19.20$\pm$1.44&204.36$\pm$1.61&312.2$\pm$9.9\\
FGS-nzm&76.40$\pm$4.77&184.04$\pm$6.73&2.2$\pm$0.1&110.21$\pm$6.39&312.84$\pm$8.48&4.2$\pm$0.2\\
FGS&76.40$\pm$4.77&184.04$\pm$6.73&2.2$\pm$0.1&110.21$\pm$6.39&312.84$\pm$8.48&4.2$\pm$0.2\\
FGS-KKTS-nzm&22.77$\pm$1.82&104.63$\pm$1.37&17.0$\pm$0.4&34.17$\pm$2.71&209.42$\pm$2.06&135.0$\pm$3.7\\
FGS-KKTS&19.48$\pm$1.62&104.23$\pm$1.26&18.5$\pm$0.4&31.32$\pm$2.64&209.63$\pm$1.98&136.3$\pm$3.9\\
\bottomrule
\end{tabular}
\end{table}

\clearpage
\subsection{Ablation study of KKT-informed local search}
\label{sec:exptAdd:ablation}

We also conduct an ablation study on KKT-informed local search by controlling which local search operations are performed. We test local search without reducing unnecessary constraints (`-noReduce'), without reversing edges (`-noReverse'), and full local search on  the ER4-Gumbel case. As shown in Figure~\ref{fig:minZRev} (with numerical values shown in Table~\ref{tab:ER4_gumbel_MinZRev}), NOTEARS-KKTS-noReverse outperforms NOTEARS-KKTS-noReduce in terms of SHD, while the opposite is true for Abs-KKTS-noReverse and Abs-KKTS-noReduce. Moreover, they are both worse than the full local search, showing that they are necessary and complement each other. In line with the discussion in Section~\ref{sec:expt}, we hypothesize that Abs benefits more from reversing edges than NOTEARS because Abs by itself suffers more from poorer local minima. Time-wise, the full local search takes only slightly longer than the other methods depicted. 

\begin{figure}[ht]
  \centering
  \includegraphics[width=0.7\linewidth]{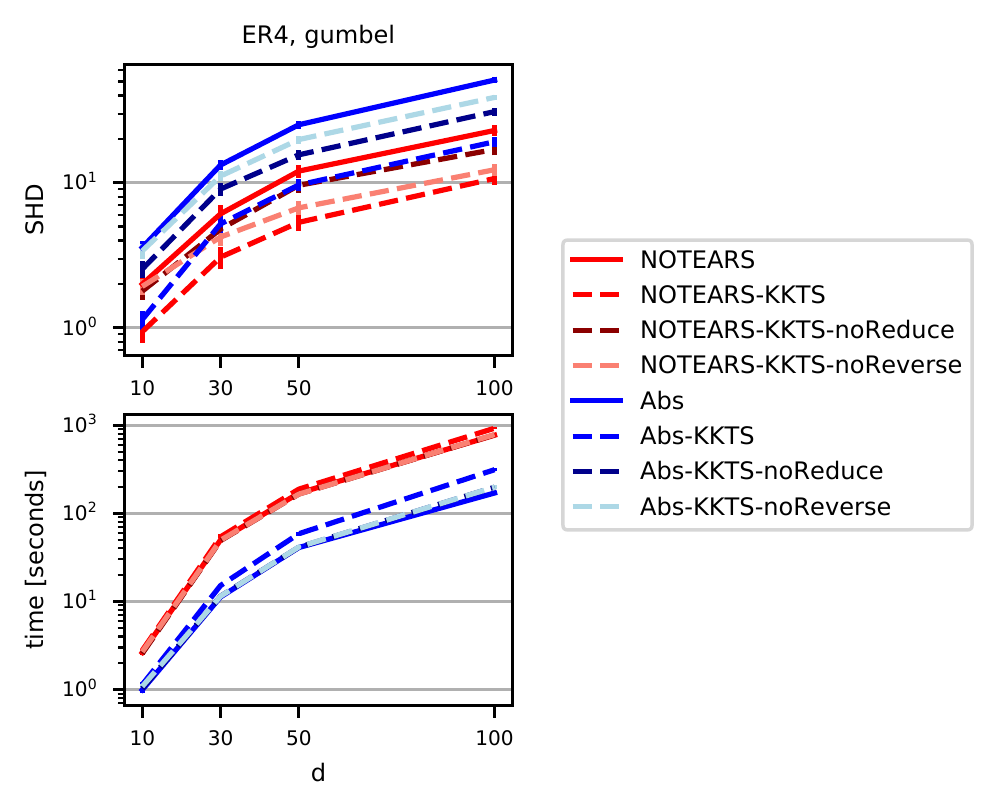}
  \caption{SHD and solution time of KKTS combinations without reducing unnecessary constraints (`-noReduce') and without reversing edges (`-noReverse').}
  \label{fig:minZRev}
\end{figure}

\begin{table}[ht]
\caption{Results of KKTS combinations without reducing unnecessary constraints (`-noReduce') and without reversing edges (`-noReverse') on ER4 graphs with Gumbel noise.}
\label{tab:ER4_gumbel_MinZRev}
\centering
\small
\begin{tabular}{lrrrrrr}
\toprule
& \multicolumn{3}{c}{$d = 10$} & \multicolumn{3}{c}{$d = 30$} \\
\cmidrule(r){2-4} \cmidrule(l){5-7}
& SHD & nnz & time (sec) & SHD & nnz & time (sec) \\
\midrule
NOTEARS&2.00$\pm$0.26&19.24$\pm$0.32&2.6$\pm$0.1&6.11$\pm$0.89&60.59$\pm$0.76&49.9$\pm$3.6\\
NOTEARS-KKTS&0.94$\pm$0.15&19.42$\pm$0.30&2.8$\pm$0.1&3.07$\pm$0.54&60.47$\pm$0.72&54.1$\pm$3.6\\
NOTEARS-KKTS-noReduce&1.79$\pm$0.23&19.02$\pm$0.30&2.6$\pm$0.1&4.80$\pm$0.77&59.26$\pm$0.69&49.0$\pm$3.5\\
NOTEARS-KKTS-noReverse&1.93$\pm$0.24&19.24$\pm$0.30&2.7$\pm$0.1&4.23$\pm$0.58&60.70$\pm$0.73&49.5$\pm$3.5\\
Abs&3.58$\pm$0.42&19.55$\pm$0.35&1.0$\pm$0.1&13.27$\pm$1.07&63.52$\pm$0.87&11.2$\pm$0.7\\
Abs-KKTS&1.14$\pm$0.18&19.36$\pm$0.31&1.1$\pm$0.1&5.21$\pm$0.68&60.87$\pm$0.76&15.2$\pm$0.7\\
Abs-KKTS-noReduce&2.54$\pm$0.35&18.82$\pm$0.31&1.0$\pm$0.1&9.06$\pm$0.97&59.59$\pm$0.75&11.5$\pm$0.7\\
Abs-KKTS-noReverse&3.39$\pm$0.38&19.51$\pm$0.34&1.1$\pm$0.1&11.13$\pm$0.92&63.74$\pm$0.83&11.4$\pm$0.7\\
\midrule[\heavyrulewidth]
& \multicolumn{3}{c}{$d = 50$} & \multicolumn{3}{c}{$d = 100$} \\
\cmidrule(r){2-4} \cmidrule(l){5-7}
& SHD & nnz & time (sec) & SHD & nnz & time (sec) \\
\midrule
NOTEARS&12.05$\pm$1.28&101.03$\pm$1.20&169.4$\pm$7.6&22.97$\pm$1.92&202.58$\pm$1.61&768.4$\pm$19.5\\
NOTEARS-KKTS&5.34$\pm$0.69&100.16$\pm$1.09&187.8$\pm$7.7&10.71$\pm$1.08&201.68$\pm$1.41&922.2$\pm$19.6\\
NOTEARS-KKTS-noReduce&9.62$\pm$1.08&98.05$\pm$1.03&165.3$\pm$7.4&16.99$\pm$1.52&196.47$\pm$1.39&776.2$\pm$19.9\\
NOTEARS-KKTS-noReverse&6.72$\pm$0.76&100.69$\pm$1.07&165.1$\pm$7.4&12.31$\pm$1.15&202.03$\pm$1.41&781.3$\pm$19.9\\
Abs&25.16$\pm$1.62&107.75$\pm$1.41&40.9$\pm$2.7&51.18$\pm$2.19&216.54$\pm$1.88&169.8$\pm$9.7\\
Abs-KKTS&9.67$\pm$0.92&101.62$\pm$1.15&58.6$\pm$2.7&19.20$\pm$1.44&204.36$\pm$1.61&312.2$\pm$9.9\\
Abs-KKTS-noReduce&15.60$\pm$1.27&98.73$\pm$1.12&41.2$\pm$2.7&30.99$\pm$1.88&198.78$\pm$1.60&197.6$\pm$11.2\\
Abs-KKTS-noReverse&19.88$\pm$1.22&107.66$\pm$1.25&41.4$\pm$2.7&38.84$\pm$1.65&216.96$\pm$1.73&197.9$\pm$11.1\\
\bottomrule
\end{tabular}
\end{table}

\clearpage
\subsection{Additional results}
\label{sec:exptAdd:results}

Figures~\ref{fig:gauss}--\ref{fig:exp} show SHD and running time results in the same manner as Figure~\ref{fig:main} for all tested combinations of SEM noise type (Gaussian, Gumbel, exponential), graph type (ER2, ER4, SF4), and $n = 1000$. The patterns discussed in Section~\ref{sec:expt} are quite similar across the three noise types.

\begin{figure}[ht]
  \centering
  \includegraphics[width=\linewidth]{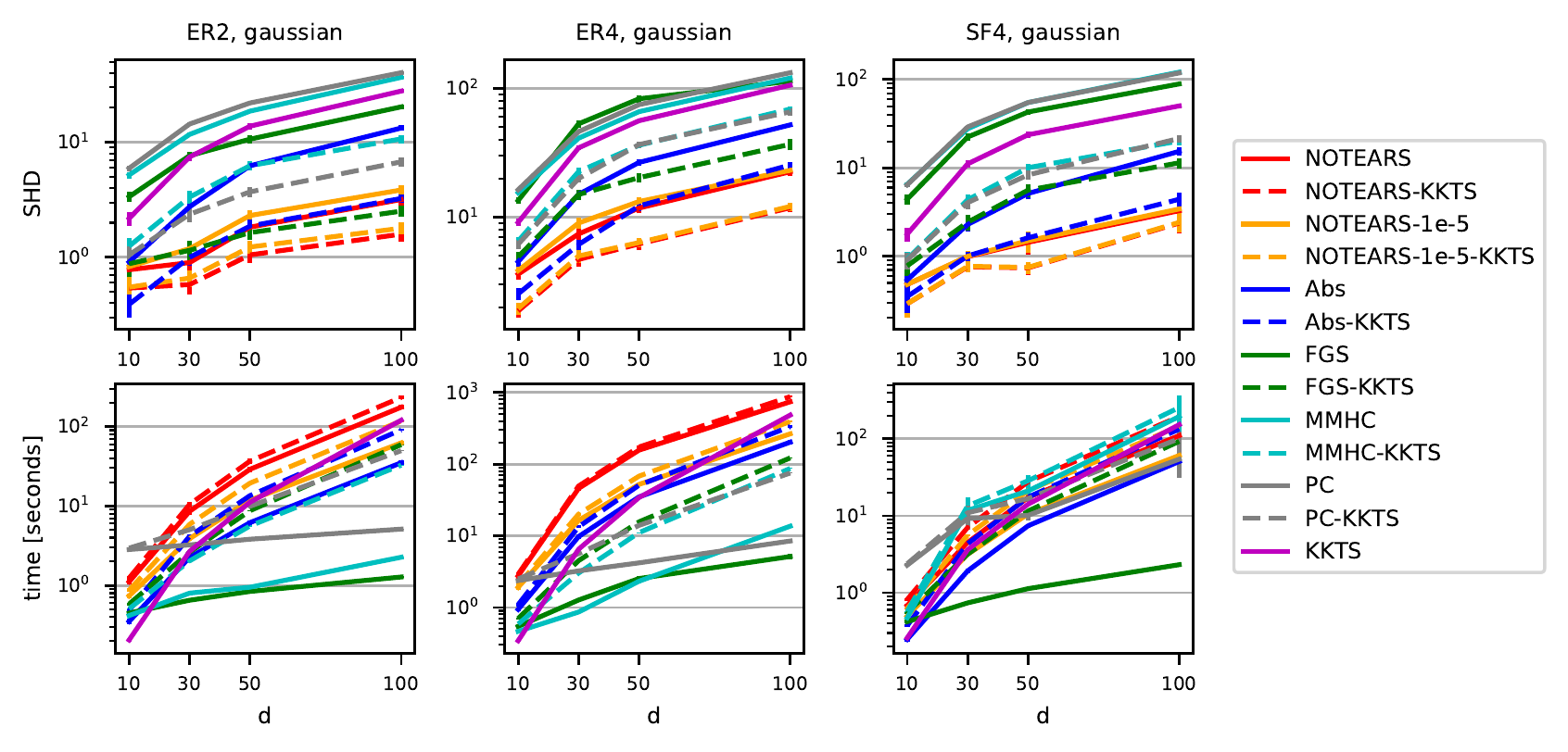}
  \caption{Structural Hamming distances (SHD) and solution times for SEMs with Gaussian noise and $n = 1000$. Red lines overlap with orange in the SF4 SHD plot.}
  \label{fig:gauss}
\end{figure}

\begin{figure}[ht]
  \centering
  \includegraphics[width=\linewidth]{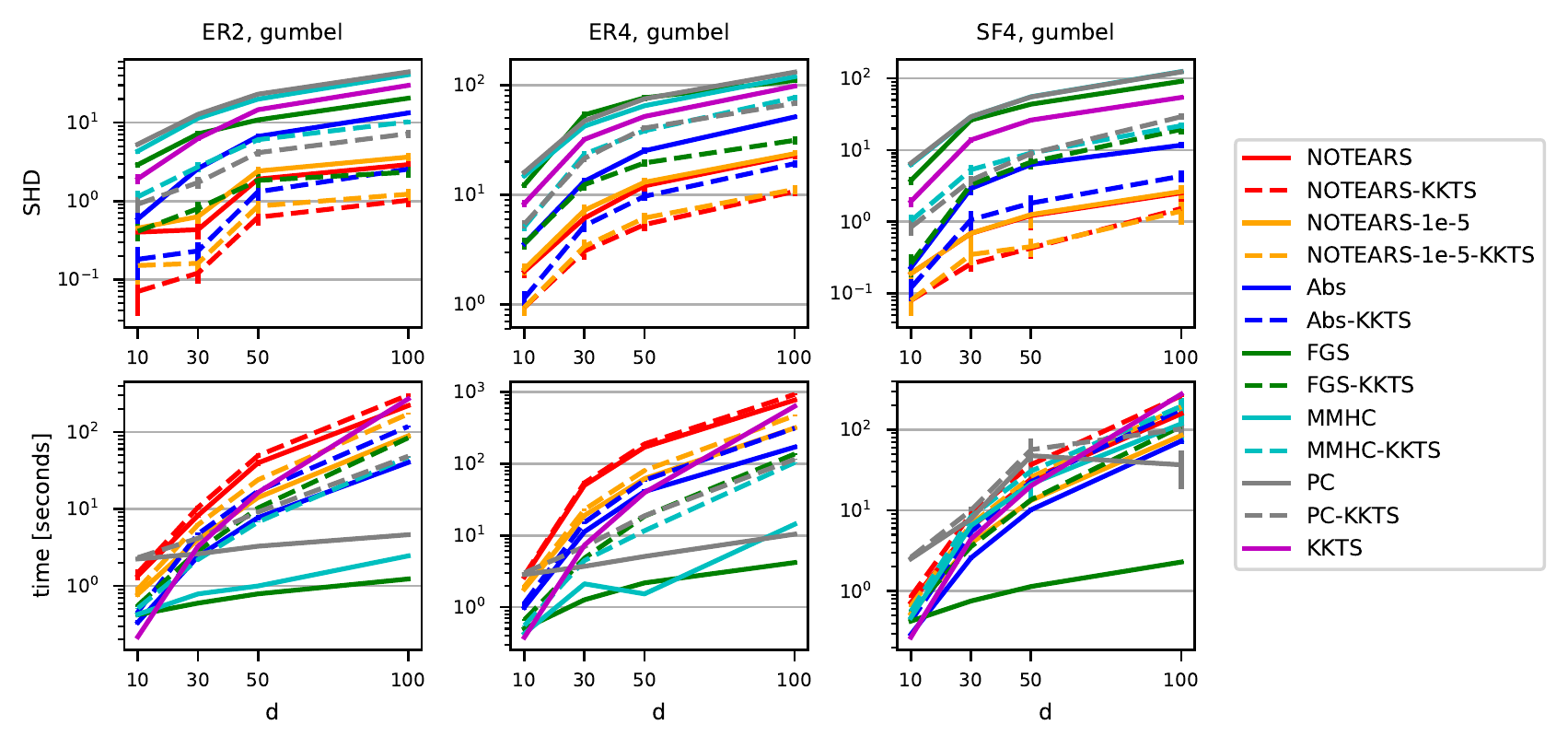}
  \caption{Structural Hamming distances (SHD) and solution times for SEMs with Gumbel noise and $n = 1000$. Red lines overlap with orange in the SF4 SHD plot.}
  \label{fig:gumbel}
\end{figure}

\begin{figure}[ht]
  \centering
  \includegraphics[width=\linewidth]{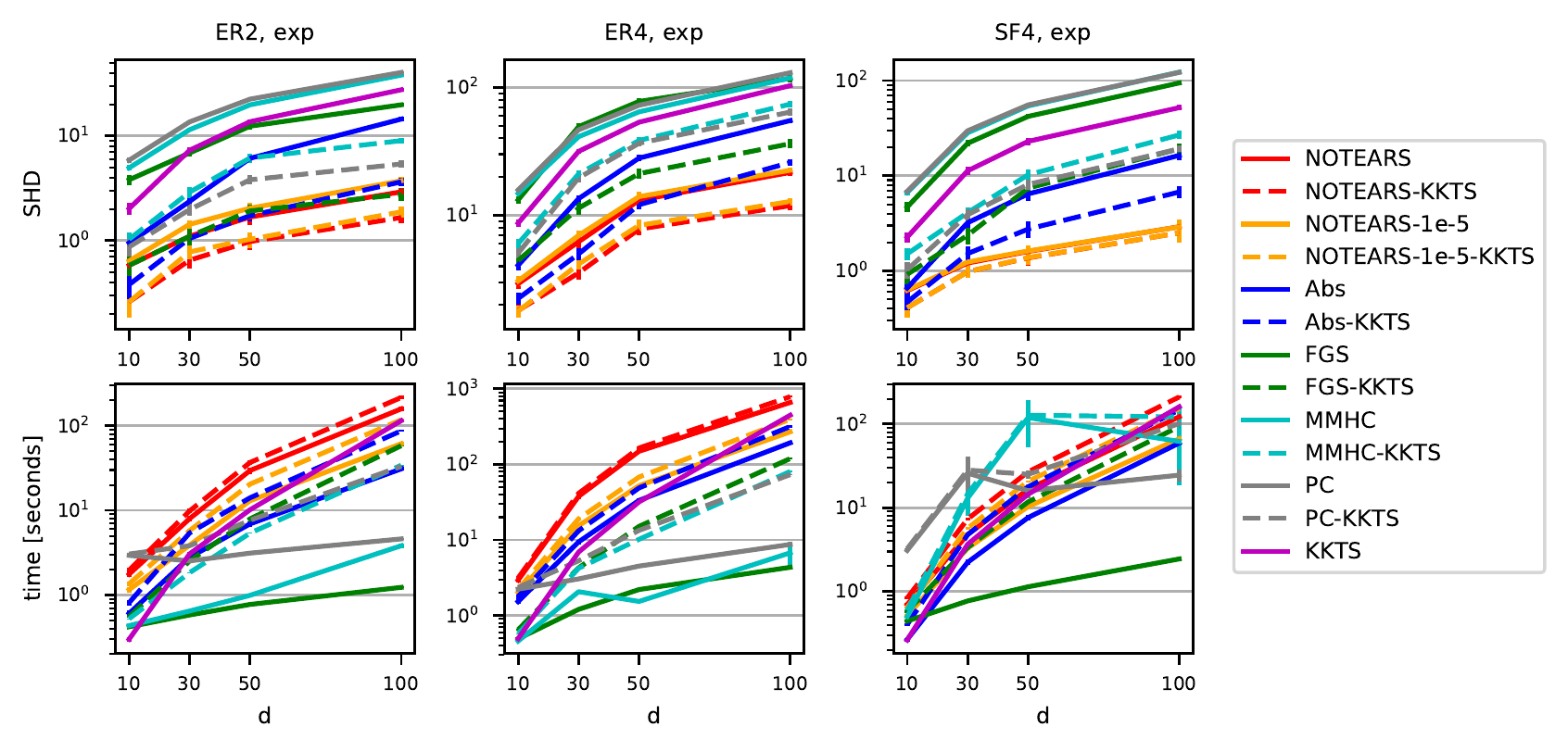}
  \caption{Structural Hamming distances (SHD) and solution times for SEMs with exponential noise and $n = 1000$. Red lines overlap with orange in the SF4 SHD plot.}
  \label{fig:exp}
\end{figure}

In response to a reviewer comment, we performed a quick comparison between the original GES algorithm \cite{chickering02} and its FGS implementation \cite{ramsey2017million}. As seen in Figure~\ref{fig:GES}, not only is FGS faster than GES as expected, but its SHD is also much better. After applying KKTS however, FGS-KKTS and GES-KKTS are similar.

\begin{figure}[ht]
  \centering
  \includegraphics[width=0.7\linewidth]{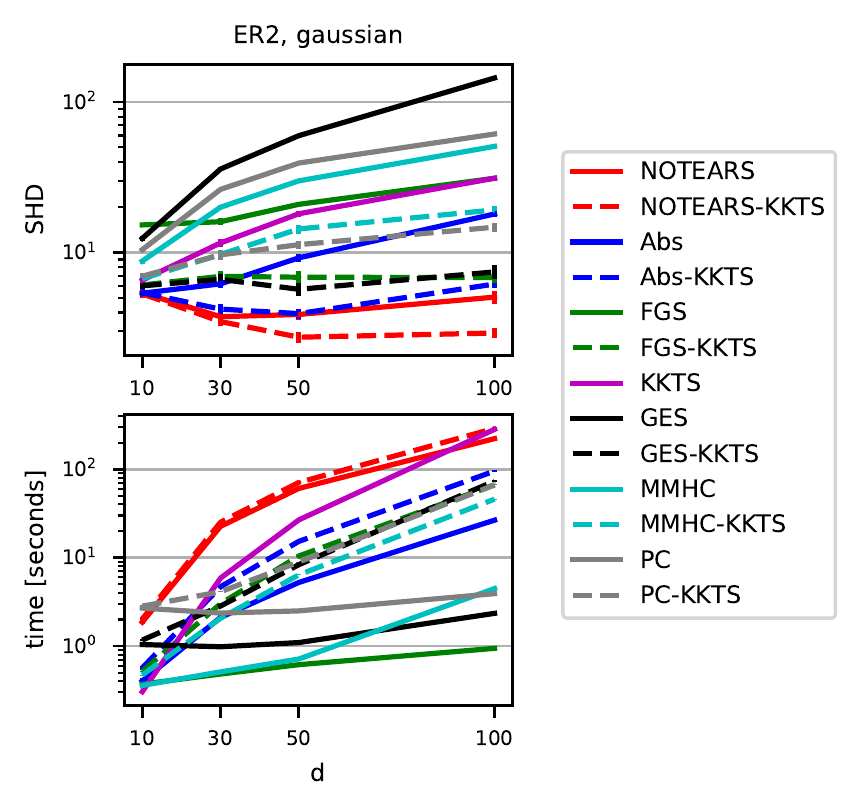}
  \caption{SHD and solution time for the same $n = 2d$ setting as in Figure~\ref{fig:main_n_2d}, showing that FGS is both faster and more accurate than GES.}
  \label{fig:GES}
\end{figure}

Tables~\ref{tab:ER2_gaussian}--\ref{tab:SF4_exp} show the same results as Figures~\ref{fig:gauss}--\ref{fig:exp} in tabular form. In addition, results for CAM \cite{buhlmann2014cam} are also shown and are seen to be less competitive in the linear SEM setting tested here. Nevertheless, like the other -KKTS combinations, CAM-KKTS succeeds in improving the SHDs of CAM, by large factors in some cases.

\begin{table}[ht]
\caption{Results (mean $\pm$ standard error over 100 trials) on ER2 graphs with Gaussian noise, $n = 1000$.}
\label{tab:ER2_gaussian}
\centering
\small
\begin{tabular}{lrrrrrr}
\toprule
& \multicolumn{3}{c}{$d = 10$} & \multicolumn{3}{c}{$d = 30$} \\
\cmidrule(r){2-4} \cmidrule(l){5-7}
& SHD & nnz & time (sec) & SHD & nnz & time (sec) \\
\midrule
NOTEARS&0.78$\pm$0.15&10.04$\pm$0.31&1.1$\pm$0.1&0.90$\pm$0.14&29.41$\pm$0.52&8.6$\pm$0.8\\
NOTEARS-KKTS&0.54$\pm$0.13&10.13$\pm$0.32&1.2$\pm$0.1&0.58$\pm$0.10&29.50$\pm$0.53&10.5$\pm$0.8\\
NOTEARS&0.83$\pm$0.15&10.00$\pm$0.31&0.7$\pm$0.0&1.19$\pm$0.16&29.18$\pm$0.52&3.9$\pm$0.2\\
NOTEARS-KKTS&0.55$\pm$0.13&10.12$\pm$0.32&0.9$\pm$0.0&0.66$\pm$0.11&29.47$\pm$0.53&5.9$\pm$0.2\\
Abs&0.91$\pm$0.17&10.12$\pm$0.32&0.4$\pm$0.0&2.75$\pm$0.34&29.54$\pm$0.55&2.3$\pm$0.2\\
Abs-KKTS&0.39$\pm$0.09&10.12$\pm$0.31&0.5$\pm$0.0&1.00$\pm$0.16&29.37$\pm$0.53&4.2$\pm$0.2\\
FGS&3.36$\pm$0.34&12.04$\pm$0.59&0.5$\pm$0.0&7.65$\pm$0.56&32.87$\pm$0.88&0.7$\pm$0.0\\
FGS-KKTS&0.88$\pm$0.21&10.15$\pm$0.31&0.6$\pm$0.0&1.15$\pm$0.24&29.49$\pm$0.53&2.6$\pm$0.1\\
MMHC&5.20$\pm$0.34&9.38$\pm$0.23&0.4$\pm$0.0&11.71$\pm$0.46&28.13$\pm$0.42&0.8$\pm$0.0\\
MMHC-KKTS&1.25$\pm$0.21&9.92$\pm$0.31&0.5$\pm$0.0&3.34$\pm$0.42&29.56$\pm$0.57&2.0$\pm$0.0\\
PC&5.94$\pm$0.36&12.21$\pm$0.27&2.8$\pm$0.1&14.42$\pm$0.53&36.83$\pm$0.47&3.2$\pm$0.0\\
PC-KKTS&1.05$\pm$0.20&10.05$\pm$0.31&3.0$\pm$0.1&2.37$\pm$0.36&29.46$\pm$0.54&5.0$\pm$0.1\\
Search&2.18$\pm$0.31&9.89$\pm$0.29&0.2$\pm$0.0&7.45$\pm$0.58&28.76$\pm$0.52&2.7$\pm$0.1\\
CAM&7.79$\pm$0.54&12.75$\pm$0.49&13.8$\pm$0.4&20.83$\pm$0.90&37.78$\pm$0.85&68.6$\pm$1.2\\
CAM-KKTS&1.41$\pm$0.22&10.19$\pm$0.33&14.0$\pm$0.4&3.81$\pm$0.49&29.49$\pm$0.54&71.0$\pm$1.2\\
\midrule[\heavyrulewidth]
& \multicolumn{3}{c}{$d = 50$} & \multicolumn{3}{c}{$d = 100$} \\
\cmidrule(r){2-4} \cmidrule(l){5-7}
& SHD & nnz & time (sec) & SHD & nnz & time (sec) \\
\midrule
NOTEARS&1.83$\pm$0.28&50.07$\pm$0.68&29.0$\pm$2.5&3.18$\pm$0.40&97.21$\pm$0.91&175.2$\pm$10.9\\
NOTEARS-KKTS&1.05$\pm$0.16&50.18$\pm$0.68&36.8$\pm$2.5&1.59$\pm$0.22&97.51$\pm$0.89&232.6$\pm$11.1\\
NOTEARS&2.31$\pm$0.26&49.79$\pm$0.69&11.8$\pm$0.7&3.85$\pm$0.35&96.64$\pm$0.90&62.3$\pm$3.1\\
NOTEARS-KKTS&1.23$\pm$0.16&50.21$\pm$0.69&19.4$\pm$0.8&1.79$\pm$0.21&97.48$\pm$0.90&117.6$\pm$3.4\\
Abs&6.25$\pm$0.53&50.90$\pm$0.77&6.2$\pm$0.4&13.31$\pm$0.71&98.53$\pm$1.07&35.0$\pm$2.2\\
Abs-KKTS&1.87$\pm$0.25&50.03$\pm$0.70&13.5$\pm$0.5&3.25$\pm$0.30&97.12$\pm$0.93&91.1$\pm$2.5\\
FGS&10.65$\pm$0.88&54.85$\pm$1.32&0.8$\pm$0.0&20.33$\pm$0.84&104.42$\pm$1.38&1.3$\pm$0.0\\
FGS-KKTS&1.64$\pm$0.23&50.34$\pm$0.70&8.7$\pm$0.1&2.52$\pm$0.28&97.88$\pm$0.91&59.2$\pm$0.9\\
MMHC&18.66$\pm$0.60&46.42$\pm$0.52&1.0$\pm$0.0&36.71$\pm$0.81&93.25$\pm$0.77&2.3$\pm$0.0\\
MMHC-KKTS&6.26$\pm$0.59&50.18$\pm$0.74&5.6$\pm$0.1&10.71$\pm$0.91&98.14$\pm$0.99&32.7$\pm$0.5\\
PC&21.99$\pm$0.65&61.79$\pm$0.64&3.8$\pm$0.1&40.27$\pm$0.86&123.45$\pm$0.92&5.1$\pm$0.1\\
PC-KKTS&3.71$\pm$0.37&50.38$\pm$0.70&9.9$\pm$0.1&6.76$\pm$0.58&97.74$\pm$0.94&49.7$\pm$0.8\\
Search&13.81$\pm$0.86&48.70$\pm$0.68&11.2$\pm$0.3&27.89$\pm$1.08&95.02$\pm$0.94&120.2$\pm$3.0\\
CAM&34.90$\pm$0.98&64.16$\pm$1.03&127.9$\pm$2.1&65.40$\pm$1.31&123.84$\pm$1.29&253.8$\pm$2.7\\
CAM-KKTS&7.24$\pm$0.66&50.71$\pm$0.75&142.5$\pm$2.0&10.80$\pm$0.84&98.61$\pm$0.97&360.8$\pm$4.1\\
\bottomrule
\end{tabular}
\end{table}

\begin{table}[ht]
\caption{Results (mean $\pm$ standard error over 100 trials) on ER4 graphs with Gaussian noise, $n = 1000$.}
\label{tab:ER4_gaussian}
\centering
\small
\begin{tabular}{lrrrrrr}
\toprule
& \multicolumn{3}{c}{$d = 10$} & \multicolumn{3}{c}{$d = 30$} \\
\cmidrule(r){2-4} \cmidrule(l){5-7}
& SHD & nnz & time (sec) & SHD & nnz & time (sec) \\
\midrule
NOTEARS&3.61$\pm$0.36&18.08$\pm$0.29&2.8$\pm$0.2&7.42$\pm$0.81&57.97$\pm$0.74&45.9$\pm$3.4\\
NOTEARS-KKTS&1.87$\pm$0.20&18.37$\pm$0.30&2.9$\pm$0.2&4.70$\pm$0.58&58.18$\pm$0.72&49.5$\pm$3.4\\
NOTEARS&3.85$\pm$0.37&17.89$\pm$0.30&1.9$\pm$0.1&9.02$\pm$0.87&57.52$\pm$0.76&16.4$\pm$1.0\\
NOTEARS-KKTS&1.95$\pm$0.22&18.35$\pm$0.30&2.1$\pm$0.1&5.00$\pm$0.57&58.08$\pm$0.70&20.1$\pm$1.0\\
Abs&4.52$\pm$0.41&18.16$\pm$0.30&0.9$\pm$0.1&15.06$\pm$1.06&60.39$\pm$0.81&9.8$\pm$0.6\\
Abs-KKTS&2.54$\pm$0.27&18.19$\pm$0.31&1.1$\pm$0.1&6.14$\pm$0.56&58.23$\pm$0.72&13.4$\pm$0.6\\
FGS&13.48$\pm$0.74&28.44$\pm$0.81&0.5$\pm$0.0&53.21$\pm$3.30&118.38$\pm$4.48&1.3$\pm$0.1\\
FGS-KKTS&4.92$\pm$0.47&17.79$\pm$0.30&0.7$\pm$0.0&15.03$\pm$1.38&59.72$\pm$0.87&4.5$\pm$0.1\\
MMHC&15.51$\pm$0.50&11.90$\pm$0.17&0.5$\pm$0.0&41.15$\pm$1.18&35.85$\pm$0.43&0.9$\pm$0.0\\
MMHC-KKTS&6.53$\pm$0.53&17.50$\pm$0.32&0.6$\pm$0.0&22.65$\pm$1.57&59.48$\pm$0.92&3.0$\pm$0.0\\
PC&16.35$\pm$0.50&15.26$\pm$0.23&2.4$\pm$0.0&46.28$\pm$1.23&45.46$\pm$0.46&3.2$\pm$0.0\\
PC-KKTS&6.17$\pm$0.52&17.26$\pm$0.32&2.5$\pm$0.0&20.00$\pm$1.44&59.15$\pm$0.85&5.6$\pm$0.1\\
Search&9.12$\pm$0.59&15.75$\pm$0.31&0.3$\pm$0.0&34.61$\pm$1.37&51.10$\pm$0.70&6.5$\pm$0.1\\
CAM&19.06$\pm$0.64&22.84$\pm$0.37&12.9$\pm$0.3&56.80$\pm$1.70&77.53$\pm$1.06&65.7$\pm$1.1\\
CAM-KKTS&6.64$\pm$0.57&17.42$\pm$0.30&13.1$\pm$0.3&19.55$\pm$1.45&59.48$\pm$0.84&70.8$\pm$1.2\\
\midrule[\heavyrulewidth]
& \multicolumn{3}{c}{$d = 50$} & \multicolumn{3}{c}{$d = 100$} \\
\cmidrule(r){2-4} \cmidrule(l){5-7}
& SHD & nnz & time (sec) & SHD & nnz & time (sec) \\
\midrule
NOTEARS&11.79$\pm$1.05&99.69$\pm$1.06&156.8$\pm$7.7&22.57$\pm$1.74&199.65$\pm$1.52&741.0$\pm$23.0\\
NOTEARS-KKTS&6.21$\pm$0.64&99.07$\pm$0.94&173.5$\pm$7.7&11.85$\pm$0.96&198.76$\pm$1.46&874.3$\pm$23.4\\
NOTEARS&13.25$\pm$0.99&98.51$\pm$1.06&52.2$\pm$2.2&23.16$\pm$1.57&197.89$\pm$1.48&264.9$\pm$10.1\\
NOTEARS-KKTS&6.36$\pm$0.55&98.66$\pm$0.95&68.1$\pm$2.2&12.13$\pm$0.92&198.56$\pm$1.46&391.8$\pm$10.5\\
Abs&26.67$\pm$1.60&103.77$\pm$1.20&34.9$\pm$2.0&52.18$\pm$1.89&208.66$\pm$1.62&202.7$\pm$12.4\\
Abs-KKTS&12.25$\pm$1.04&99.46$\pm$0.95&50.9$\pm$2.0&25.45$\pm$1.40&201.45$\pm$1.54&334.0$\pm$12.7\\
FGS&83.28$\pm$5.61&196.78$\pm$8.01&2.5$\pm$0.2&114.07$\pm$8.35&321.52$\pm$11.14&5.1$\pm$0.5\\
FGS-KKTS&20.31$\pm$1.73&102.31$\pm$1.15&15.7$\pm$0.3&36.89$\pm$3.74&208.55$\pm$2.18&122.1$\pm$2.5\\
MMHC&66.13$\pm$1.50&61.29$\pm$0.55&2.3$\pm$0.1&120.66$\pm$2.06&128.66$\pm$1.10&13.6$\pm$0.3\\
MMHC-KKTS&36.27$\pm$2.05&104.06$\pm$1.32&11.1$\pm$0.2&69.24$\pm$3.35&213.52$\pm$2.14&87.0$\pm$1.2\\
PC&74.99$\pm$1.62&77.16$\pm$0.56&4.2$\pm$0.1&132.77$\pm$2.52&152.49$\pm$0.93&8.5$\pm$0.2\\
PC-KKTS&36.73$\pm$2.14&103.28$\pm$1.23&14.0$\pm$0.2&65.88$\pm$3.99&212.44$\pm$2.27&75.2$\pm$1.1\\
Search&56.23$\pm$1.66&89.33$\pm$0.98&34.3$\pm$0.6&106.24$\pm$2.58&181.97$\pm$1.41&484.0$\pm$7.1\\
CAM&91.13$\pm$2.02&129.64$\pm$1.43&130.2$\pm$1.9&159.91$\pm$3.11&247.67$\pm$1.86&271.4$\pm$3.2\\
CAM-KKTS&34.76$\pm$1.98&104.17$\pm$1.23&146.1$\pm$2.1&68.13$\pm$3.71&212.93$\pm$2.19&388.3$\pm$5.7\\
\bottomrule
\end{tabular}
\end{table}

\begin{table}[ht]
\caption{Results (mean $\pm$ standard error over 100 trials) on SF4 graphs with Gaussian noise, $n = 1000$.}
\label{tab:SF4_gaussian}
\centering
\small
\begin{tabular}{lrrrrrr}
\toprule
& \multicolumn{3}{c}{$d = 10$} & \multicolumn{3}{c}{$d = 30$} \\
\cmidrule(r){2-4} \cmidrule(l){5-7}
& SHD & nnz & time (sec) & SHD & nnz & time (sec) \\
\midrule
NOTEARS&0.48$\pm$0.12&13.67$\pm$0.18&0.7$\pm$0.0&0.99$\pm$0.16&49.91$\pm$0.33&4.5$\pm$0.2\\
NOTEARS-KKTS&0.29$\pm$0.09&13.64$\pm$0.16&0.8$\pm$0.0&0.76$\pm$0.10&49.95$\pm$0.33&7.0$\pm$0.2\\
NOTEARS&0.48$\pm$0.12&13.67$\pm$0.18&0.5$\pm$0.0&1.00$\pm$0.16&49.90$\pm$0.33&3.2$\pm$0.1\\
NOTEARS-KKTS&0.29$\pm$0.09&13.64$\pm$0.16&0.6$\pm$0.0&0.77$\pm$0.10&49.94$\pm$0.33&5.6$\pm$0.1\\
Abs&0.54$\pm$0.15&13.62$\pm$0.16&0.2$\pm$0.0&2.29$\pm$0.39&50.03$\pm$0.35&1.9$\pm$0.1\\
Abs-KKTS&0.35$\pm$0.12&13.63$\pm$0.16&0.4$\pm$0.0&1.01$\pm$0.12&49.93$\pm$0.33&4.3$\pm$0.1\\
FGS&4.42$\pm$0.55&17.25$\pm$0.71&0.4$\pm$0.0&22.37$\pm$1.96&65.48$\pm$2.39&0.7$\pm$0.0\\
FGS-KKTS&0.78$\pm$0.17&13.58$\pm$0.16&0.6$\pm$0.0&2.43$\pm$0.49&50.05$\pm$0.36&3.2$\pm$0.0\\
MMHC&6.53$\pm$0.37&11.59$\pm$0.14&0.5$\pm$0.0&27.73$\pm$0.68&34.48$\pm$0.41&11.8$\pm$3.9\\
MMHC-KKTS&0.93$\pm$0.20&13.58$\pm$0.17&0.6$\pm$0.0&4.36$\pm$0.66&49.78$\pm$0.37&13.4$\pm$3.9\\
PC&6.44$\pm$0.33&12.35$\pm$0.17&2.3$\pm$0.0&29.09$\pm$0.58&36.11$\pm$0.39&9.2$\pm$2.5\\
PC-KKTS&0.91$\pm$0.20&13.54$\pm$0.16&2.4$\pm$0.0&4.05$\pm$0.65&49.83$\pm$0.36&10.9$\pm$2.5\\
Search&1.78$\pm$0.31&13.28$\pm$0.17&0.3$\pm$0.0&11.11$\pm$0.98&49.18$\pm$0.43&3.6$\pm$0.1\\
CAM&13.70$\pm$0.76&19.90$\pm$0.53&11.5$\pm$0.2&46.72$\pm$1.35&57.86$\pm$1.03&51.5$\pm$0.7\\
CAM-KKTS&1.36$\pm$0.27&13.46$\pm$0.18&11.6$\pm$0.2&4.09$\pm$0.63&49.92$\pm$0.36&53.1$\pm$0.7\\
\midrule[\heavyrulewidth]
& \multicolumn{3}{c}{$d = 50$} & \multicolumn{3}{c}{$d = 100$} \\
\cmidrule(r){2-4} \cmidrule(l){5-7}
& SHD & nnz & time (sec) & SHD & nnz & time (sec) \\
\midrule
NOTEARS&1.44$\pm$0.39&87.81$\pm$0.37&17.3$\pm$0.9&3.28$\pm$0.88&183.95$\pm$0.60&110.3$\pm$5.6\\
NOTEARS-KKTS&0.74$\pm$0.12&87.74$\pm$0.37&27.7$\pm$0.9&2.38$\pm$0.57&183.80$\pm$0.55&195.0$\pm$5.9\\
NOTEARS&1.50$\pm$0.39&87.80$\pm$0.37&10.6$\pm$0.4&3.43$\pm$0.89&183.93$\pm$0.60&60.8$\pm$3.6\\
NOTEARS-KKTS&0.75$\pm$0.12&87.74$\pm$0.37&21.2$\pm$0.4&2.42$\pm$0.58&183.78$\pm$0.55&146.2$\pm$3.9\\
Abs&5.14$\pm$0.73&88.45$\pm$0.47&7.4$\pm$0.4&15.40$\pm$1.74&186.68$\pm$0.74&50.7$\pm$4.1\\
Abs-KKTS&1.62$\pm$0.25&87.67$\pm$0.39&17.4$\pm$0.5&4.42$\pm$0.79&183.09$\pm$0.54&133.7$\pm$4.3\\
FGS&42.94$\pm$2.79&107.87$\pm$3.31&1.1$\pm$0.0&89.18$\pm$4.25&193.06$\pm$4.79&2.3$\pm$0.1\\
FGS-KKTS&5.61$\pm$0.95&87.94$\pm$0.50&11.5$\pm$0.1&11.42$\pm$1.51&181.64$\pm$0.79&92.7$\pm$1.1\\
MMHC&54.89$\pm$1.00&54.53$\pm$0.66&21.0$\pm$4.8&121.35$\pm$1.44&114.06$\pm$1.02&194.7$\pm$106.3\\
MMHC-KKTS&10.07$\pm$1.13&87.06$\pm$0.53&28.6$\pm$4.8&20.14$\pm$2.03&181.63$\pm$0.90&255.1$\pm$106.4\\
PC&54.98$\pm$0.80&58.14$\pm$0.68&10.1$\pm$1.4&118.74$\pm$1.12&120.68$\pm$0.98&54.4$\pm$23.1\\
PC-KKTS&8.34$\pm$0.93&86.42$\pm$0.56&16.7$\pm$1.4&21.56$\pm$2.08&180.15$\pm$0.93&100.7$\pm$23.1\\
Search&23.74$\pm$1.84&85.59$\pm$0.72&14.1$\pm$0.2&50.30$\pm$2.96&178.68$\pm$0.96&152.7$\pm$2.7\\
CAM&82.51$\pm$2.01&89.62$\pm$1.43&91.0$\pm$1.0&157.53$\pm$2.44&160.91$\pm$1.95&223.3$\pm$1.9\\
CAM-KKTS&11.91$\pm$1.36&88.47$\pm$0.65&98.8$\pm$1.0&24.74$\pm$2.16&184.40$\pm$0.97&288.3$\pm$2.4\\
\bottomrule
\end{tabular}
\end{table}

\begin{table}[ht]
\caption{Results (mean $\pm$ standard error over 100 trials) on ER2 graphs with Gumbel noise, $n = 1000$.}
\label{tab:ER2_gumbel}
\centering
\small
\begin{tabular}{lrrrrrr}
\toprule
& \multicolumn{3}{c}{$d = 10$} & \multicolumn{3}{c}{$d = 30$} \\
\cmidrule(r){2-4} \cmidrule(l){5-7}
& SHD & nnz & time (sec) & SHD & nnz & time (sec) \\
\midrule
NOTEARS&0.40$\pm$0.10&9.49$\pm$0.27&1.4$\pm$0.2&0.43$\pm$0.10&29.71$\pm$0.56&8.2$\pm$0.5\\
NOTEARS-KKTS&0.07$\pm$0.04&9.56$\pm$0.27&1.5$\pm$0.2&0.12$\pm$0.03&29.69$\pm$0.56&10.5$\pm$0.5\\
NOTEARS&0.45$\pm$0.10&9.45$\pm$0.27&0.8$\pm$0.1&0.63$\pm$0.12&29.52$\pm$0.55&4.0$\pm$0.2\\
NOTEARS-KKTS&0.15$\pm$0.06&9.55$\pm$0.27&0.9$\pm$0.1&0.16$\pm$0.04&29.69$\pm$0.56&6.2$\pm$0.2\\
Abs&0.59$\pm$0.13&9.58$\pm$0.26&0.3$\pm$0.0&2.59$\pm$0.28&30.72$\pm$0.60&2.4$\pm$0.2\\
Abs-KKTS&0.18$\pm$0.08&9.52$\pm$0.27&0.5$\pm$0.0&0.23$\pm$0.06&29.68$\pm$0.56&4.7$\pm$0.2\\
FGS&2.90$\pm$0.27&10.63$\pm$0.45&0.4$\pm$0.0&7.18$\pm$0.64&33.09$\pm$0.98&0.6$\pm$0.0\\
FGS-KKTS&0.41$\pm$0.10&9.58$\pm$0.28&0.6$\pm$0.0&0.80$\pm$0.18&29.72$\pm$0.57&2.8$\pm$0.1\\
MMHC&4.33$\pm$0.29&8.82$\pm$0.20&0.4$\pm$0.0&11.36$\pm$0.51&28.05$\pm$0.41&0.8$\pm$0.0\\
MMHC-KKTS&1.13$\pm$0.21&9.54$\pm$0.27&0.5$\pm$0.0&2.72$\pm$0.38&30.27$\pm$0.61&2.2$\pm$0.0\\
PC&5.29$\pm$0.30&12.18$\pm$0.26&2.2$\pm$0.0&12.75$\pm$0.49&35.72$\pm$0.41&2.6$\pm$0.0\\
PC-KKTS&0.91$\pm$0.20&9.51$\pm$0.27&2.3$\pm$0.0&1.72$\pm$0.29&30.09$\pm$0.61&4.2$\pm$0.1\\
Search&1.92$\pm$0.30&9.50$\pm$0.27&0.2$\pm$0.0&6.25$\pm$0.48&29.71$\pm$0.57&3.3$\pm$0.1\\
CAM&9.84$\pm$0.43&12.93$\pm$0.42&10.4$\pm$0.2&30.88$\pm$0.97&41.95$\pm$1.01&48.7$\pm$0.7\\
CAM-KKTS&1.23$\pm$0.21&9.70$\pm$0.29&10.5$\pm$0.2&4.66$\pm$0.58&30.57$\pm$0.65&50.1$\pm$0.7\\
\midrule[\heavyrulewidth]
& \multicolumn{3}{c}{$d = 50$} & \multicolumn{3}{c}{$d = 100$} \\
\cmidrule(r){2-4} \cmidrule(l){5-7}
& SHD & nnz & time (sec) & SHD & nnz & time (sec) \\
\midrule
NOTEARS&1.90$\pm$0.31&50.89$\pm$0.71&39.7$\pm$3.6&2.91$\pm$0.40&100.61$\pm$0.98&220.2$\pm$12.5\\
NOTEARS-KKTS&0.62$\pm$0.14&51.09$\pm$0.70&49.3$\pm$3.7&1.02$\pm$0.19&100.68$\pm$1.00&301.6$\pm$12.8\\
NOTEARS&2.41$\pm$0.33&50.72$\pm$0.73&14.1$\pm$0.8&3.63$\pm$0.43&100.32$\pm$0.98&89.0$\pm$4.7\\
NOTEARS-KKTS&0.86$\pm$0.18&51.11$\pm$0.71&23.9$\pm$0.9&1.22$\pm$0.21&100.69$\pm$0.99&171.6$\pm$5.1\\
Abs&6.67$\pm$0.57&52.60$\pm$0.78&7.7$\pm$0.6&13.27$\pm$0.98&103.49$\pm$1.30&40.2$\pm$2.5\\
Abs-KKTS&1.31$\pm$0.30&51.23$\pm$0.72&17.1$\pm$0.7&2.55$\pm$0.33&101.25$\pm$1.02&118.1$\pm$3.0\\
FGS&10.83$\pm$0.67&54.93$\pm$1.14&0.8$\pm$0.0&20.47$\pm$0.76&106.67$\pm$1.37&1.2$\pm$0.0\\
FGS-KKTS&1.83$\pm$0.35&51.27$\pm$0.74&10.3$\pm$0.2&2.33$\pm$0.37&101.12$\pm$1.03&84.0$\pm$1.1\\
MMHC&19.93$\pm$0.59&48.54$\pm$0.52&1.0$\pm$0.0&40.88$\pm$0.86&100.67$\pm$0.72&2.5$\pm$0.0\\
MMHC-KKTS&6.06$\pm$0.60&51.92$\pm$0.76&6.8$\pm$0.1&10.22$\pm$0.77&102.97$\pm$1.09&47.2$\pm$0.6\\
PC&23.07$\pm$0.68&62.33$\pm$0.60&3.3$\pm$0.1&44.31$\pm$0.88&125.32$\pm$0.87&4.6$\pm$0.1\\
PC-KKTS&4.13$\pm$0.49&51.58$\pm$0.78&9.1$\pm$0.2&7.22$\pm$0.85&102.45$\pm$1.18&48.4$\pm$0.7\\
Search&14.61$\pm$0.81&51.13$\pm$0.74&16.4$\pm$0.4&29.84$\pm$1.26&100.76$\pm$1.05&269.1$\pm$4.2\\
CAM&51.11$\pm$1.13&71.25$\pm$1.19&90.3$\pm$0.8&100.67$\pm$1.61&140.15$\pm$1.55&210.9$\pm$1.4\\
CAM-KKTS&8.51$\pm$0.77&52.80$\pm$0.82&96.0$\pm$0.8&22.19$\pm$1.62&106.30$\pm$1.25&258.2$\pm$1.5\\
\bottomrule
\end{tabular}
\end{table}

\begin{table}[ht]
\caption{Results (mean $\pm$ standard error over 100 trials) on ER4 graphs with Gumbel noise, $n = 1000$.}
\label{tab:ER4_gumbel}
\centering
\small
\begin{tabular}{lrrrrrr}
\toprule
& \multicolumn{3}{c}{$d = 10$} & \multicolumn{3}{c}{$d = 30$} \\
\cmidrule(r){2-4} \cmidrule(l){5-7}
& SHD & nnz & time (sec) & SHD & nnz & time (sec) \\
\midrule
NOTEARS&2.00$\pm$0.26&19.24$\pm$0.32&2.6$\pm$0.1&6.11$\pm$0.89&60.59$\pm$0.76&49.9$\pm$3.6\\
NOTEARS-KKTS&0.94$\pm$0.15&19.42$\pm$0.30&2.8$\pm$0.1&3.07$\pm$0.54&60.47$\pm$0.72&54.1$\pm$3.6\\
NOTEARS&2.10$\pm$0.26&19.15$\pm$0.32&1.8$\pm$0.1&7.21$\pm$0.88&59.97$\pm$0.76&18.9$\pm$1.2\\
NOTEARS-KKTS&0.94$\pm$0.15&19.42$\pm$0.30&1.9$\pm$0.1&3.38$\pm$0.54&60.42$\pm$0.72&22.9$\pm$1.1\\
Abs&3.58$\pm$0.42&19.55$\pm$0.35&1.0$\pm$0.1&13.27$\pm$1.07&63.52$\pm$0.87&11.2$\pm$0.7\\
Abs-KKTS&1.14$\pm$0.18&19.36$\pm$0.31&1.1$\pm$0.1&5.21$\pm$0.68&60.87$\pm$0.76&15.2$\pm$0.7\\
FGS&12.29$\pm$0.66&27.85$\pm$0.83&0.5$\pm$0.0&53.42$\pm$3.56&119.62$\pm$4.82&1.3$\pm$0.1\\
FGS-KKTS&3.58$\pm$0.45&19.12$\pm$0.31&0.7$\pm$0.0&12.36$\pm$1.38&62.05$\pm$0.85&4.8$\pm$0.1\\
MMHC&14.88$\pm$0.52&12.33$\pm$0.17&0.4$\pm$0.0&42.18$\pm$1.04&35.87$\pm$0.38&2.1$\pm$0.0\\
MMHC-KKTS&5.11$\pm$0.49&18.85$\pm$0.29&0.6$\pm$0.0&23.45$\pm$1.55&64.07$\pm$0.90&4.5$\pm$0.1\\
PC&16.01$\pm$0.52&15.25$\pm$0.24&2.8$\pm$0.1&47.03$\pm$1.16&45.27$\pm$0.47&3.7$\pm$0.0\\
PC-KKTS&5.41$\pm$0.50&18.62$\pm$0.29&3.0$\pm$0.1&21.45$\pm$1.56&64.64$\pm$0.98&7.0$\pm$0.1\\
Search&8.31$\pm$0.58&16.44$\pm$0.33&0.4$\pm$0.0&31.97$\pm$1.15&54.64$\pm$0.69&7.3$\pm$0.1\\
CAM&20.66$\pm$0.47&23.68$\pm$0.36&9.9$\pm$0.1&65.42$\pm$1.36&81.94$\pm$0.96&46.3$\pm$0.5\\
CAM-KKTS&5.62$\pm$0.49&18.66$\pm$0.31&10.0$\pm$0.1&22.34$\pm$1.42&63.18$\pm$0.92&48.5$\pm$0.5\\
\midrule[\heavyrulewidth]
& \multicolumn{3}{c}{$d = 50$} & \multicolumn{3}{c}{$d = 100$} \\
\cmidrule(r){2-4} \cmidrule(l){5-7}
& SHD & nnz & time (sec) & SHD & nnz & time (sec) \\
\midrule
NOTEARS&12.05$\pm$1.28&101.03$\pm$1.20&169.4$\pm$7.6&22.97$\pm$1.92&202.58$\pm$1.61&768.4$\pm$19.5\\
NOTEARS-KKTS&5.34$\pm$0.69&100.16$\pm$1.09&187.8$\pm$7.7&10.71$\pm$1.08&201.68$\pm$1.41&922.2$\pm$19.6\\
NOTEARS&12.99$\pm$1.19&100.38$\pm$1.18&61.8$\pm$2.9&23.66$\pm$1.81&201.62$\pm$1.60&313.9$\pm$11.2\\
NOTEARS-KKTS&6.14$\pm$0.72&100.39$\pm$1.10&80.1$\pm$2.9&11.19$\pm$1.06&201.45$\pm$1.37&466.3$\pm$11.3\\
Abs&25.16$\pm$1.62&107.75$\pm$1.41&40.9$\pm$2.7&51.18$\pm$2.19&216.54$\pm$1.88&169.8$\pm$9.7\\
Abs-KKTS&9.67$\pm$0.92&101.62$\pm$1.15&58.6$\pm$2.7&19.20$\pm$1.44&204.36$\pm$1.61&312.2$\pm$9.9\\
FGS&76.40$\pm$4.77&184.04$\pm$6.73&2.2$\pm$0.1&110.21$\pm$6.39&312.84$\pm$8.48&4.2$\pm$0.2\\
FGS-KKTS&19.48$\pm$1.62&104.23$\pm$1.26&18.5$\pm$0.4&31.32$\pm$2.64&209.63$\pm$1.98&136.3$\pm$3.9\\
MMHC&64.74$\pm$1.59&60.97$\pm$0.68&1.5$\pm$0.0&119.00$\pm$2.10&129.00$\pm$0.96&14.3$\pm$0.2\\
MMHC-KKTS&38.28$\pm$2.48&108.72$\pm$1.65&11.6$\pm$0.2&76.75$\pm$3.87&224.75$\pm$2.57&103.7$\pm$1.2\\
PC&74.85$\pm$1.64&75.99$\pm$0.61&5.1$\pm$0.1&131.01$\pm$2.41&153.06$\pm$0.89&10.4$\pm$0.1\\
PC-KKTS&40.53$\pm$2.58&108.84$\pm$1.56&18.7$\pm$0.3&67.96$\pm$3.91&221.17$\pm$2.49&115.1$\pm$1.7\\
Search&51.67$\pm$1.96&92.10$\pm$0.98&39.4$\pm$0.7&97.72$\pm$2.46&189.95$\pm$1.61&630.7$\pm$7.9\\
CAM&105.83$\pm$2.02&134.77$\pm$1.41&91.7$\pm$1.0&187.69$\pm$2.68&256.63$\pm$1.77&216.3$\pm$1.9\\
CAM-KKTS&47.10$\pm$2.35&109.91$\pm$1.54&100.6$\pm$1.0&86.06$\pm$3.72&224.16$\pm$2.58&306.3$\pm$2.9\\
\bottomrule
\end{tabular}
\end{table}

\begin{table}[ht]
\caption{Results (mean $\pm$ standard error over 100 trials) on SF4 graphs with Gumbel noise, $n = 1000$.}
\label{tab:SF4_gumbel}
\centering
\small
\begin{tabular}{lrrrrrr}
\toprule
& \multicolumn{3}{c}{$d = 10$} & \multicolumn{3}{c}{$d = 30$} \\
\cmidrule(r){2-4} \cmidrule(l){5-7}
& SHD & nnz & time (sec) & SHD & nnz & time (sec) \\
\midrule
NOTEARS&0.19$\pm$0.07&13.68$\pm$0.16&0.7$\pm$0.0&0.69$\pm$0.25&50.12$\pm$0.31&6.2$\pm$0.3\\
NOTEARS-KKTS&0.08$\pm$0.03&13.72$\pm$0.16&0.9$\pm$0.0&0.26$\pm$0.06&50.08$\pm$0.30&9.0$\pm$0.3\\
NOTEARS&0.19$\pm$0.07&13.68$\pm$0.16&0.5$\pm$0.0&0.69$\pm$0.25&50.11$\pm$0.31&3.9$\pm$0.2\\
NOTEARS-KKTS&0.08$\pm$0.03&13.72$\pm$0.16&0.7$\pm$0.0&0.35$\pm$0.10&50.12$\pm$0.31&6.7$\pm$0.2\\
Abs&0.23$\pm$0.06&13.74$\pm$0.16&0.3$\pm$0.0&2.92$\pm$0.60&50.96$\pm$0.46&2.6$\pm$0.2\\
Abs-KKTS&0.12$\pm$0.04&13.72$\pm$0.16&0.4$\pm$0.0&1.07$\pm$0.33&50.34$\pm$0.36&5.4$\pm$0.2\\
FGS&3.76$\pm$0.52&16.17$\pm$0.59&0.4$\pm$0.0&26.25$\pm$2.10&67.17$\pm$2.42&0.8$\pm$0.0\\
FGS-KKTS&0.26$\pm$0.09&13.70$\pm$0.15&0.6$\pm$0.0&3.31$\pm$0.67&50.51$\pm$0.39&3.6$\pm$0.0\\
MMHC&6.28$\pm$0.39&11.66$\pm$0.14&0.5$\pm$0.0&29.08$\pm$0.80&33.04$\pm$0.47&6.5$\pm$2.5\\
MMHC-KKTS&1.03$\pm$0.23&13.90$\pm$0.18&0.6$\pm$0.0&5.32$\pm$0.86&51.29$\pm$0.48&8.4$\pm$2.5\\
PC&6.55$\pm$0.28&12.56$\pm$0.16&2.5$\pm$0.1&29.18$\pm$0.56&35.36$\pm$0.39&7.8$\pm$1.2\\
PC-KKTS&0.84$\pm$0.21&14.05$\pm$0.16&2.6$\pm$0.1&3.79$\pm$0.63&50.45$\pm$0.42&10.0$\pm$1.2\\
Search&1.92$\pm$0.30&13.53$\pm$0.17&0.3$\pm$0.0&13.86$\pm$1.37&50.33$\pm$0.47&4.4$\pm$0.1\\
CAM&19.95$\pm$0.63&22.73$\pm$0.42&11.3$\pm$0.2&64.41$\pm$1.28&65.37$\pm$1.02&50.8$\pm$0.7\\
CAM-KKTS&1.32$\pm$0.32&13.81$\pm$0.18&11.4$\pm$0.2&7.74$\pm$1.04&51.65$\pm$0.59&52.7$\pm$0.7\\
\midrule[\heavyrulewidth]
& \multicolumn{3}{c}{$d = 50$} & \multicolumn{3}{c}{$d = 100$} \\
\cmidrule(r){2-4} \cmidrule(l){5-7}
& SHD & nnz & time (sec) & SHD & nnz & time (sec) \\
\midrule
NOTEARS&1.22$\pm$0.45&88.35$\pm$0.40&24.4$\pm$1.3&2.53$\pm$0.57&184.08$\pm$0.65&157.4$\pm$7.4\\
NOTEARS-KKTS&0.43$\pm$0.12&88.34$\pm$0.38&36.9$\pm$1.4&1.53$\pm$0.50&183.66$\pm$0.59&262.6$\pm$7.8\\
NOTEARS&1.26$\pm$0.46&88.35$\pm$0.40&13.3$\pm$0.7&2.64$\pm$0.59&184.04$\pm$0.65&84.6$\pm$5.2\\
NOTEARS-KKTS&0.45$\pm$0.13&88.35$\pm$0.39&25.8$\pm$0.7&1.41$\pm$0.49&183.60$\pm$0.59&191.6$\pm$5.7\\
Abs&6.29$\pm$0.84&90.68$\pm$0.61&10.2$\pm$0.7&11.69$\pm$1.22&187.94$\pm$0.94&72.3$\pm$5.8\\
Abs-KKTS&1.82$\pm$0.52&88.73$\pm$0.45&22.7$\pm$0.8&4.33$\pm$0.82&184.23$\pm$0.64&179.3$\pm$6.4\\
FGS&43.64$\pm$2.76&107.99$\pm$3.43&1.1$\pm$0.0&90.96$\pm$3.91&188.82$\pm$4.27&2.3$\pm$0.1\\
FGS-KKTS&6.79$\pm$1.36&89.74$\pm$0.72&13.5$\pm$0.2&19.44$\pm$2.90&187.95$\pm$1.55&109.9$\pm$1.7\\
MMHC&55.35$\pm$0.95&55.43$\pm$0.63&21.3$\pm$7.8&124.12$\pm$1.43&112.37$\pm$0.96&118.1$\pm$49.3\\
MMHC-KKTS&9.09$\pm$1.25&89.84$\pm$0.65&30.7$\pm$7.8&21.86$\pm$2.21&186.37$\pm$1.13&194.8$\pm$49.4\\
PC&55.07$\pm$0.98&58.65$\pm$0.67&47.6$\pm$21.6&123.02$\pm$1.30&118.71$\pm$0.93&36.7$\pm$18.3\\
PC-KKTS&8.96$\pm$1.28&88.93$\pm$0.66&56.4$\pm$21.6&29.16$\pm$2.99&186.16$\pm$1.18&102.1$\pm$18.4\\
Search&26.20$\pm$1.82&88.97$\pm$0.73&20.2$\pm$0.4&54.18$\pm$3.33&183.98$\pm$1.10&274.7$\pm$3.5\\
CAM&104.84$\pm$1.34&99.24$\pm$1.14&89.6$\pm$1.0&212.09$\pm$2.31&182.76$\pm$2.08&211.1$\pm$1.4\\
CAM-KKTS&16.81$\pm$1.92&92.20$\pm$0.87&98.4$\pm$1.0&35.73$\pm$3.32&190.62$\pm$1.37&284.3$\pm$2.0\\
\bottomrule
\end{tabular}
\end{table}

\begin{table}[ht]
\caption{Results (mean $\pm$ standard error over 100 trials) on ER2 graphs with exponential noise, $n = 1000$.}
\label{tab:ER2_exp}
\centering
\small
\begin{tabular}{lrrrrrr}
\toprule
& \multicolumn{3}{c}{$d = 10$} & \multicolumn{3}{c}{$d = 30$} \\
\cmidrule(r){2-4} \cmidrule(l){5-7}
& SHD & nnz & time (sec) & SHD & nnz & time (sec) \\
\midrule
NOTEARS&0.58$\pm$0.12&9.87$\pm$0.30&1.8$\pm$0.1&1.10$\pm$0.12&28.84$\pm$0.49&8.0$\pm$0.6\\
NOTEARS-KKTS&0.26$\pm$0.07&9.89$\pm$0.29&2.0$\pm$0.1&0.65$\pm$0.11&28.94$\pm$0.49&9.8$\pm$0.7\\
NOTEARS&0.64$\pm$0.12&9.80$\pm$0.29&1.2$\pm$0.1&1.43$\pm$0.14&28.51$\pm$0.49&3.9$\pm$0.2\\
NOTEARS-KKTS&0.26$\pm$0.07&9.89$\pm$0.29&1.4$\pm$0.1&0.78$\pm$0.12&28.87$\pm$0.49&5.8$\pm$0.2\\
Abs&0.95$\pm$0.18&9.86$\pm$0.29&0.6$\pm$0.1&2.37$\pm$0.26&28.97$\pm$0.50&2.8$\pm$0.2\\
Abs-KKTS&0.38$\pm$0.10&9.83$\pm$0.29&0.8$\pm$0.1&1.05$\pm$0.16&28.88$\pm$0.50&5.4$\pm$0.3\\
FGS&3.81$\pm$0.41&11.61$\pm$0.59&0.4$\pm$0.0&6.87$\pm$0.45&31.59$\pm$0.75&0.6$\pm$0.0\\
FGS-KKTS&0.58$\pm$0.12&9.81$\pm$0.29&0.6$\pm$0.0&1.10$\pm$0.20&29.00$\pm$0.50&2.5$\pm$0.0\\
MMHC&4.95$\pm$0.30&9.06$\pm$0.21&0.4$\pm$0.0&11.46$\pm$0.46&27.73$\pm$0.41&0.6$\pm$0.0\\
MMHC-KKTS&1.03$\pm$0.18&9.82$\pm$0.30&0.5$\pm$0.0&2.90$\pm$0.35&29.00$\pm$0.50&1.8$\pm$0.0\\
PC&5.86$\pm$0.31&12.34$\pm$0.29&2.9$\pm$0.1&13.61$\pm$0.43&35.86$\pm$0.46&2.5$\pm$0.0\\
PC-KKTS&0.87$\pm$0.17&9.69$\pm$0.29&3.0$\pm$0.1&1.98$\pm$0.25&28.92$\pm$0.49&3.8$\pm$0.0\\
Search&2.02$\pm$0.27&9.48$\pm$0.28&0.3$\pm$0.0&7.29$\pm$0.55&28.18$\pm$0.50&3.1$\pm$0.2\\
CAM&11.48$\pm$0.43&14.25$\pm$0.46&9.5$\pm$0.1&32.43$\pm$0.79&42.58$\pm$0.78&45.4$\pm$0.5\\
CAM-KKTS&1.69$\pm$0.25&9.76$\pm$0.29&9.6$\pm$0.1&5.50$\pm$0.53&29.16$\pm$0.52&46.5$\pm$0.5\\
\midrule[\heavyrulewidth]
& \multicolumn{3}{c}{$d = 50$} & \multicolumn{3}{c}{$d = 100$} \\
\cmidrule(r){2-4} \cmidrule(l){5-7}
& SHD & nnz & time (sec) & SHD & nnz & time (sec) \\
\midrule
NOTEARS&1.68$\pm$0.27&49.26$\pm$0.66&29.6$\pm$2.7&2.92$\pm$0.28&98.12$\pm$0.82&158.4$\pm$10.1\\
NOTEARS-KKTS&0.98$\pm$0.17&49.45$\pm$0.66&36.8$\pm$2.7&1.66$\pm$0.20&98.39$\pm$0.82&215.7$\pm$10.2\\
NOTEARS&2.04$\pm$0.25&49.01$\pm$0.65&13.0$\pm$0.9&3.73$\pm$0.32&97.56$\pm$0.81&61.3$\pm$2.8\\
NOTEARS-KKTS&1.03$\pm$0.18&49.40$\pm$0.65&20.5$\pm$0.9&1.88$\pm$0.22&98.29$\pm$0.82&117.8$\pm$3.0\\
Abs&6.11$\pm$0.47&50.24$\pm$0.74&6.9$\pm$0.5&14.47$\pm$0.66&100.20$\pm$1.02&30.9$\pm$1.5\\
Abs-KKTS&1.73$\pm$0.21&49.15$\pm$0.69&14.1$\pm$0.5&3.64$\pm$0.34&98.02$\pm$0.84&87.1$\pm$1.8\\
FGS&12.38$\pm$0.87&55.58$\pm$1.42&0.8$\pm$0.0&19.81$\pm$0.75&105.59$\pm$1.17&1.2$\pm$0.0\\
FGS-KKTS&1.92$\pm$0.29&49.48$\pm$0.65&7.9$\pm$0.1&2.77$\pm$0.37&98.76$\pm$0.83&58.1$\pm$0.8\\
MMHC&19.81$\pm$0.55&47.86$\pm$0.52&1.0$\pm$0.0&38.07$\pm$0.75&101.11$\pm$0.67&3.8$\pm$0.2\\
MMHC-KKTS&6.15$\pm$0.61&49.92$\pm$0.70&5.4$\pm$0.1&9.01$\pm$0.59&99.21$\pm$0.87&34.2$\pm$0.5\\
PC&22.46$\pm$0.67&61.68$\pm$0.63&3.1$\pm$0.1&40.36$\pm$0.73&124.45$\pm$0.82&4.6$\pm$0.1\\
PC-KKTS&3.81$\pm$0.39&49.57$\pm$0.69&7.6$\pm$0.1&5.39$\pm$0.46&98.62$\pm$0.84&33.1$\pm$0.5\\
Search&13.67$\pm$0.69&47.63$\pm$0.63&10.2$\pm$0.3&27.59$\pm$1.05&95.95$\pm$0.93&114.8$\pm$2.4\\
CAM&54.15$\pm$1.08&72.17$\pm$1.05&84.1$\pm$0.7&114.26$\pm$1.28&148.78$\pm$1.25&197.6$\pm$1.2\\
CAM-KKTS&9.71$\pm$0.85&50.35$\pm$0.72&88.4$\pm$0.7&20.48$\pm$1.21&101.37$\pm$0.94&229.1$\pm$1.2\\
\bottomrule
\end{tabular}
\end{table}

\begin{table}[ht]
\caption{Results (mean $\pm$ standard error over 100 trials) on ER4 graphs with exponential noise, $n = 1000$.}
\label{tab:ER4_exp}
\centering
\small
\begin{tabular}{lrrrrrr}
\toprule
& \multicolumn{3}{c}{$d = 10$} & \multicolumn{3}{c}{$d = 30$} \\
\cmidrule(r){2-4} \cmidrule(l){5-7}
& SHD & nnz & time (sec) & SHD & nnz & time (sec) \\
\midrule
NOTEARS&2.92$\pm$0.27&18.05$\pm$0.35&2.9$\pm$0.2&6.21$\pm$0.64&58.41$\pm$0.77&37.8$\pm$2.6\\
NOTEARS-KKTS&1.80$\pm$0.19&18.28$\pm$0.34&3.1$\pm$0.2&3.55$\pm$0.43&58.11$\pm$0.76&41.5$\pm$2.6\\
NOTEARS&3.07$\pm$0.28&17.82$\pm$0.35&2.0$\pm$0.1&6.97$\pm$0.61&57.67$\pm$0.77&15.4$\pm$0.8\\
NOTEARS-KKTS&1.78$\pm$0.18&18.26$\pm$0.34&2.2$\pm$0.1&4.16$\pm$0.41&58.16$\pm$0.76&19.1$\pm$0.8\\
Abs&4.07$\pm$0.34&18.17$\pm$0.38&1.5$\pm$0.1&13.49$\pm$0.93&60.12$\pm$0.89&9.4$\pm$0.6\\
Abs-KKTS&2.24$\pm$0.25&18.29$\pm$0.34&1.8$\pm$0.1&4.99$\pm$0.53&57.96$\pm$0.77&13.2$\pm$0.7\\
FGS&13.20$\pm$0.81&29.51$\pm$0.98&0.5$\pm$0.0&49.28$\pm$3.29&113.40$\pm$4.46&1.2$\pm$0.1\\
FGS-KKTS&4.43$\pm$0.44&17.85$\pm$0.34&0.7$\pm$0.0&11.38$\pm$1.27&58.88$\pm$0.82&4.3$\pm$0.1\\
MMHC&14.81$\pm$0.53&12.04$\pm$0.19&0.5$\pm$0.0&41.11$\pm$1.13&36.41$\pm$0.42&2.1$\pm$0.0\\
MMHC-KKTS&6.01$\pm$0.52&17.35$\pm$0.34&0.6$\pm$0.0&21.05$\pm$1.39&59.09$\pm$0.92&4.2$\pm$0.0\\
PC&15.77$\pm$0.53&15.21$\pm$0.23&2.2$\pm$0.0&46.57$\pm$1.19&45.90$\pm$0.41&3.1$\pm$0.0\\
PC-KKTS&5.06$\pm$0.43&17.91$\pm$0.34&2.4$\pm$0.0&19.61$\pm$1.42&59.16$\pm$0.88&5.3$\pm$0.1\\
Search&8.68$\pm$0.60&15.94$\pm$0.35&0.5$\pm$0.0&31.58$\pm$1.39&52.62$\pm$0.70&6.9$\pm$0.2\\
CAM&20.12$\pm$0.52&23.58$\pm$0.41&10.0$\pm$0.1&67.48$\pm$1.65&82.05$\pm$1.14&45.0$\pm$0.6\\
CAM-KKTS&7.08$\pm$0.59&17.08$\pm$0.32&10.1$\pm$0.1&24.71$\pm$1.60&58.63$\pm$0.82&46.9$\pm$0.6\\
\midrule[\heavyrulewidth]
& \multicolumn{3}{c}{$d = 50$} & \multicolumn{3}{c}{$d = 100$} \\
\cmidrule(r){2-4} \cmidrule(l){5-7}
& SHD & nnz & time (sec) & SHD & nnz & time (sec) \\
\midrule
NOTEARS&13.10$\pm$1.30&97.14$\pm$0.94&149.3$\pm$7.7&21.62$\pm$1.31&196.95$\pm$1.37&660.6$\pm$20.2\\
NOTEARS-KKTS&7.80$\pm$0.81&97.17$\pm$0.88&165.0$\pm$7.8&11.89$\pm$0.89&197.80$\pm$1.32&787.8$\pm$20.5\\
NOTEARS&14.00$\pm$1.25&96.64$\pm$0.95&52.4$\pm$2.7&22.49$\pm$1.25&195.75$\pm$1.40&270.1$\pm$9.3\\
NOTEARS-KKTS&8.35$\pm$0.85&97.24$\pm$0.88&68.1$\pm$2.8&12.81$\pm$0.88&197.48$\pm$1.33&396.6$\pm$9.6\\
Abs&28.00$\pm$1.61&101.50$\pm$1.19&33.4$\pm$2.1&54.99$\pm$2.39&208.17$\pm$1.99&193.2$\pm$9.6\\
Abs-KKTS&12.10$\pm$1.03&97.58$\pm$0.92&48.9$\pm$2.1&25.86$\pm$1.67&199.68$\pm$1.59&317.3$\pm$9.7\\
FGS&77.55$\pm$5.13&184.44$\pm$6.80&2.2$\pm$0.1&117.42$\pm$7.33&321.85$\pm$9.72&4.4$\pm$0.2\\
FGS-KKTS&21.17$\pm$1.90&99.95$\pm$1.12&15.1$\pm$0.3&36.27$\pm$3.03&206.12$\pm$1.99&119.4$\pm$2.5\\
MMHC&64.34$\pm$1.47&60.67$\pm$0.60&1.5$\pm$0.0&118.32$\pm$2.12&128.96$\pm$1.10&6.8$\pm$2.7\\
MMHC-KKTS&38.44$\pm$2.22&101.47$\pm$1.21&10.2$\pm$0.1&73.86$\pm$4.18&213.01$\pm$2.28&80.6$\pm$2.9\\
PC&72.39$\pm$1.54&75.24$\pm$0.59&4.5$\pm$0.1&130.04$\pm$2.49&154.14$\pm$0.92&8.7$\pm$0.1\\
PC-KKTS&36.55$\pm$2.16&102.24$\pm$1.26&13.4$\pm$0.2&64.06$\pm$3.88&210.39$\pm$2.26&73.3$\pm$1.0\\
Search&53.39$\pm$1.73&86.99$\pm$1.03&31.9$\pm$0.6&103.06$\pm$2.46&179.18$\pm$1.36&448.1$\pm$5.9\\
CAM&104.86$\pm$1.89&134.03$\pm$1.41&86.2$\pm$1.0&199.41$\pm$2.53&261.46$\pm$1.67&203.4$\pm$1.4\\
CAM-KKTS&41.92$\pm$2.29&102.43$\pm$1.15&94.2$\pm$1.0&85.55$\pm$3.39&213.69$\pm$1.91&273.9$\pm$2.0\\
\bottomrule
\end{tabular}
\end{table}

\begin{table}[ht]
\caption{Results (mean $\pm$ standard error over 100 trials) on SF4 graphs with exponential noise, $n = 1000$.}
\label{tab:SF4_exp}
\centering
\small
\begin{tabular}{lrrrrrr}
\toprule
& \multicolumn{3}{c}{$d = 10$} & \multicolumn{3}{c}{$d = 30$} \\
\cmidrule(r){2-4} \cmidrule(l){5-7}
& SHD & nnz & time (sec) & SHD & nnz & time (sec) \\
\midrule
NOTEARS&0.62$\pm$0.12&13.68$\pm$0.17&0.7$\pm$0.0&1.21$\pm$0.21&49.70$\pm$0.38&4.8$\pm$0.2\\
NOTEARS-KKTS&0.41$\pm$0.08&13.66$\pm$0.16&0.8$\pm$0.0&0.98$\pm$0.14&49.61$\pm$0.35&7.3$\pm$0.2\\
NOTEARS&0.63$\pm$0.12&13.67$\pm$0.17&0.5$\pm$0.0&1.24$\pm$0.21&49.67$\pm$0.38&3.3$\pm$0.2\\
NOTEARS-KKTS&0.41$\pm$0.08&13.66$\pm$0.16&0.6$\pm$0.0&0.98$\pm$0.14&49.61$\pm$0.35&5.7$\pm$0.2\\
Abs&0.66$\pm$0.11&13.66$\pm$0.17&0.3$\pm$0.0&3.21$\pm$0.44&49.90$\pm$0.43&2.2$\pm$0.2\\
Abs-KKTS&0.47$\pm$0.08&13.63$\pm$0.16&0.4$\pm$0.0&1.52$\pm$0.28&49.52$\pm$0.36&4.8$\pm$0.2\\
FGS&4.69$\pm$0.57&17.21$\pm$0.66&0.4$\pm$0.0&21.93$\pm$1.75&64.47$\pm$2.22&0.8$\pm$0.0\\
FGS-KKTS&0.92$\pm$0.20&13.63$\pm$0.17&0.6$\pm$0.0&2.37$\pm$0.46&49.72$\pm$0.37&3.3$\pm$0.0\\
MMHC&6.66$\pm$0.38&11.28$\pm$0.12&0.5$\pm$0.0&28.25$\pm$0.67&33.41$\pm$0.43&13.0$\pm$5.0\\
MMHC-KKTS&1.51$\pm$0.23&13.47$\pm$0.17&0.6$\pm$0.0&4.10$\pm$0.56&49.46$\pm$0.40&14.7$\pm$5.0\\
PC&6.84$\pm$0.33&12.52$\pm$0.16&3.1$\pm$0.1&29.71$\pm$0.73&35.75$\pm$0.43&25.8$\pm$12.2\\
PC-KKTS&1.03$\pm$0.20&13.88$\pm$0.17&3.2$\pm$0.1&3.90$\pm$0.57&48.88$\pm$0.38&28.2$\pm$12.2\\
Search&2.26$\pm$0.28&13.13$\pm$0.20&0.3$\pm$0.0&11.28$\pm$1.06&48.86$\pm$0.41&3.7$\pm$0.1\\
CAM&20.58$\pm$0.62&23.29$\pm$0.42&10.5$\pm$0.2&66.45$\pm$1.33&65.58$\pm$1.03&47.3$\pm$0.7\\
CAM-KKTS&2.02$\pm$0.28&13.26$\pm$0.18&10.6$\pm$0.2&7.03$\pm$0.81&49.63$\pm$0.43&49.0$\pm$0.7\\
\midrule[\heavyrulewidth]
& \multicolumn{3}{c}{$d = 50$} & \multicolumn{3}{c}{$d = 100$} \\
\cmidrule(r){2-4} \cmidrule(l){5-7}
& SHD & nnz & time (sec) & SHD & nnz & time (sec) \\
\midrule
NOTEARS&1.60$\pm$0.23&87.37$\pm$0.42&16.6$\pm$0.8&2.91$\pm$0.57&183.84$\pm$0.58&122.3$\pm$6.1\\
NOTEARS-KKTS&1.37$\pm$0.23&87.44$\pm$0.41&26.8$\pm$0.9&2.51$\pm$0.52&183.83$\pm$0.58&210.7$\pm$6.5\\
NOTEARS&1.62$\pm$0.23&87.33$\pm$0.41&10.3$\pm$0.5&2.90$\pm$0.56&183.78$\pm$0.57&66.8$\pm$4.0\\
NOTEARS-KKTS&1.38$\pm$0.23&87.43$\pm$0.41&20.7$\pm$0.6&2.51$\pm$0.52&183.83$\pm$0.58&152.7$\pm$4.5\\
Abs&6.44$\pm$1.06&88.63$\pm$0.71&7.6$\pm$0.6&16.37$\pm$1.68&187.11$\pm$0.98&59.2$\pm$5.4\\
Abs-KKTS&2.75$\pm$0.54&87.66$\pm$0.51&17.8$\pm$0.7&6.75$\pm$0.98&183.86$\pm$0.74&147.5$\pm$5.7\\
FGS&42.04$\pm$2.52&105.22$\pm$2.87&1.1$\pm$0.0&94.46$\pm$4.56&197.61$\pm$5.14&2.4$\pm$0.1\\
FGS-KKTS&7.41$\pm$1.28&87.88$\pm$0.53&11.7$\pm$0.1&19.20$\pm$2.68&186.21$\pm$1.43&93.5$\pm$1.0\\
MMHC&54.17$\pm$1.01&54.60$\pm$0.71&119.4$\pm$66.9&122.83$\pm$1.35&99.62$\pm$0.97&61.9$\pm$43.5\\
MMHC-KKTS&10.23$\pm$1.45&86.82$\pm$0.61&127.2$\pm$66.9&26.85$\pm$2.69&181.25$\pm$1.08&121.9$\pm$43.5\\
PC&55.54$\pm$0.81&59.18$\pm$0.67&16.0$\pm$2.9&121.91$\pm$1.22&120.34$\pm$0.85&24.5$\pm$3.8\\
PC-KKTS&8.11$\pm$0.98&86.61$\pm$0.54&25.1$\pm$2.8&19.18$\pm$1.51&179.10$\pm$0.75&101.3$\pm$3.9\\
Search&22.98$\pm$2.00&85.41$\pm$0.73&14.5$\pm$0.3&51.99$\pm$3.33&181.09$\pm$1.14&160.7$\pm$2.9\\
CAM&112.37$\pm$1.77&104.55$\pm$1.53&89.8$\pm$1.0&230.42$\pm$2.11&195.68$\pm$1.83&212.5$\pm$2.1\\
CAM-KKTS&16.02$\pm$1.66&88.90$\pm$0.69&97.2$\pm$1.0&33.75$\pm$2.47&185.03$\pm$1.30&275.2$\pm$2.5\\
\bottomrule
\end{tabular}
\end{table}

Tables~\ref{tab:ER2_gaussian_n_2d}--\ref{tab:SF4_exp_n_2d} show the same results as Figure~\ref{fig:main_n_2d} in tabular form.

\begin{table}[ht]
\caption{Results (mean $\pm$ standard error over 100 trials) on ER2 graphs with Gaussian noise, $n = 2d$.}
\label{tab:ER2_gaussian_n_2d}
\centering
\small
\begin{tabular}{lrrrrrr}
\toprule
& \multicolumn{3}{c}{$d = 10$} & \multicolumn{3}{c}{$d = 30$} \\
\cmidrule(r){2-4} \cmidrule(l){5-7}
& SHD & nnz & time (sec) & SHD & nnz & time (sec) \\
\midrule
NOTEARS&5.40$\pm$0.31&12.40$\pm$0.31&1.9$\pm$0.1&3.75$\pm$0.24&29.65$\pm$0.53&22.6$\pm$0.8\\
NOTEARS-KKTS&5.32$\pm$0.29&12.86$\pm$0.32&2.0$\pm$0.1&3.48$\pm$0.25&29.96$\pm$0.54&25.2$\pm$0.8\\
NOTEARS&5.40$\pm$0.32&12.33$\pm$0.31&0.9$\pm$0.0&3.94$\pm$0.29&29.45$\pm$0.53&4.8$\pm$0.3\\
NOTEARS-KKTS&5.29$\pm$0.29&12.85$\pm$0.32&1.0$\pm$0.0&3.52$\pm$0.24&29.77$\pm$0.52&7.4$\pm$0.3\\
Abs&5.37$\pm$0.30&12.47$\pm$0.31&0.4$\pm$0.0&6.20$\pm$0.39&29.84$\pm$0.63&2.1$\pm$0.1\\
Abs-KKTS&5.42$\pm$0.28&12.85$\pm$0.31&0.6$\pm$0.0&4.21$\pm$0.31&29.76$\pm$0.56&4.7$\pm$0.1\\
FGS&15.27$\pm$0.44&25.78$\pm$0.60&0.4$\pm$0.0&16.04$\pm$0.89&30.28$\pm$0.82&0.5$\pm$0.0\\
FGS-KKTS&6.09$\pm$0.32&12.37$\pm$0.29&0.5$\pm$0.0&6.95$\pm$0.54&30.28$\pm$0.58&3.0$\pm$0.0\\
MMHC&8.78$\pm$0.33&5.01$\pm$0.12&0.4$\pm$0.0&20.03$\pm$0.54&21.90$\pm$0.28&0.5$\pm$0.0\\
MMHC-KKTS&6.71$\pm$0.37&11.98$\pm$0.35&0.5$\pm$0.0&9.72$\pm$0.70&30.35$\pm$0.62&2.1$\pm$0.0\\
PC&10.45$\pm$0.33&6.62$\pm$0.21&2.7$\pm$0.1&26.29$\pm$0.58&30.87$\pm$0.42&2.4$\pm$0.0\\
PC-KKTS&6.97$\pm$0.38&11.73$\pm$0.31&2.8$\pm$0.1&9.66$\pm$0.71&29.66$\pm$0.60&4.1$\pm$0.1\\
Search&6.56$\pm$0.37&11.77$\pm$0.28&0.3$\pm$0.0&11.59$\pm$0.70&29.04$\pm$0.55&5.9$\pm$0.1\\
\midrule[\heavyrulewidth]
& \multicolumn{3}{c}{$d = 50$} & \multicolumn{3}{c}{$d = 100$} \\
\cmidrule(r){2-4} \cmidrule(l){5-7}
& SHD & nnz & time (sec) & SHD & nnz & time (sec) \\
\midrule
NOTEARS&3.88$\pm$0.32&49.32$\pm$0.69&60.7$\pm$1.9&5.06$\pm$0.49&96.05$\pm$0.90&222.0$\pm$7.2\\
NOTEARS-KKTS&2.74$\pm$0.22&49.25$\pm$0.66&70.9$\pm$1.9&2.92$\pm$0.22&96.41$\pm$0.90&289.8$\pm$7.4\\
NOTEARS&4.42$\pm$0.33&48.94$\pm$0.67&12.4$\pm$0.6&5.28$\pm$0.50&95.84$\pm$0.92&66.3$\pm$2.8\\
NOTEARS-KKTS&2.95$\pm$0.22&49.15$\pm$0.67&22.1$\pm$0.6&2.88$\pm$0.20&96.36$\pm$0.89&131.8$\pm$2.9\\
Abs&9.23$\pm$0.56&49.88$\pm$0.77&5.2$\pm$0.3&18.05$\pm$0.76&95.52$\pm$1.14&26.5$\pm$1.4\\
Abs-KKTS&3.93$\pm$0.30&49.00$\pm$0.65&15.3$\pm$0.3&6.18$\pm$0.40&95.84$\pm$0.93&95.4$\pm$1.6\\
FGS&20.92$\pm$0.70&52.16$\pm$0.89&0.6$\pm$0.0&31.17$\pm$0.82&103.36$\pm$1.16&0.9$\pm$0.0\\
FGS-KKTS&6.86$\pm$0.65&49.64$\pm$0.76&10.4$\pm$0.1&6.82$\pm$0.56&96.95$\pm$0.93&66.7$\pm$0.7\\
MMHC&29.96$\pm$0.70&34.56$\pm$0.37&0.7$\pm$0.0&50.80$\pm$0.87&79.48$\pm$0.63&4.5$\pm$0.2\\
MMHC-KKTS&14.35$\pm$0.96&50.08$\pm$0.78&6.4$\pm$0.1&19.22$\pm$1.16&97.61$\pm$1.05&45.9$\pm$0.5\\
PC&39.35$\pm$0.79&55.57$\pm$0.62&2.5$\pm$0.0&61.35$\pm$1.08&117.41$\pm$0.88&3.9$\pm$0.1\\
PC-KKTS&11.29$\pm$0.71&49.58$\pm$0.71&8.9$\pm$0.1&14.75$\pm$1.00&96.78$\pm$1.03&66.6$\pm$1.0\\
Search&18.11$\pm$0.82&48.58$\pm$0.67&26.8$\pm$0.2&31.22$\pm$1.22&94.18$\pm$0.99&282.7$\pm$3.0\\
\bottomrule
\end{tabular}
\end{table}

\begin{table}[ht]
\caption{Results (mean $\pm$ standard error over 100 trials) on ER4 graphs with Gumbel noise, $n = 2d$.}
\label{tab:ER4_gumbel_n_2d}
\centering
\small
\begin{tabular}{lrrrrrr}
\toprule
& \multicolumn{3}{c}{$d = 10$} & \multicolumn{3}{c}{$d = 30$} \\
\cmidrule(r){2-4} \cmidrule(l){5-7}
& SHD & nnz & time (sec) & SHD & nnz & time (sec) \\
\midrule
NOTEARS&9.49$\pm$0.40&21.11$\pm$0.35&2.9$\pm$0.2&17.19$\pm$1.01&64.34$\pm$0.90&50.3$\pm$2.4\\
NOTEARS-KKTS&9.16$\pm$0.37&21.84$\pm$0.32&3.1$\pm$0.2&13.63$\pm$0.63&63.82$\pm$0.77&53.9$\pm$2.5\\
NOTEARS&9.43$\pm$0.40&20.95$\pm$0.35&2.1$\pm$0.1&18.56$\pm$1.05&64.10$\pm$0.90&18.5$\pm$0.8\\
NOTEARS-KKTS&9.12$\pm$0.38&21.84$\pm$0.31&2.2$\pm$0.1&14.00$\pm$0.70&64.12$\pm$0.86&21.9$\pm$0.8\\
Abs&9.37$\pm$0.41&21.11$\pm$0.33&0.9$\pm$0.1&23.71$\pm$1.20&67.86$\pm$0.98&8.2$\pm$0.5\\
Abs-KKTS&8.68$\pm$0.32&21.82$\pm$0.32&1.1$\pm$0.1&16.51$\pm$0.90&64.77$\pm$0.84&11.4$\pm$0.5\\
FGS&20.34$\pm$0.47&32.84$\pm$0.56&0.4$\pm$0.0&50.57$\pm$1.54&64.90$\pm$1.33&0.6$\pm$0.0\\
FGS-KKTS&11.25$\pm$0.43&20.44$\pm$0.33&0.5$\pm$0.0&32.61$\pm$1.66&67.68$\pm$1.05&3.7$\pm$0.1\\
MMHC&17.81$\pm$0.37&5.74$\pm$0.12&0.4$\pm$0.0&50.29$\pm$0.86&23.78$\pm$0.28&0.5$\pm$0.0\\
MMHC-KKTS&12.20$\pm$0.51&19.90$\pm$0.35&0.5$\pm$0.0&43.91$\pm$1.73&69.74$\pm$1.12&2.4$\pm$0.0\\
PC&20.10$\pm$0.34&7.22$\pm$0.26&2.2$\pm$0.0&58.42$\pm$0.86&35.84$\pm$0.45&2.3$\pm$0.0\\
PC-KKTS&13.24$\pm$0.47&19.56$\pm$0.40&2.4$\pm$0.0&42.51$\pm$1.65&68.63$\pm$1.15&4.2$\pm$0.0\\
Search&13.37$\pm$0.54&17.81$\pm$0.40&0.4$\pm$0.0&42.55$\pm$1.16&55.10$\pm$0.73&8.4$\pm$0.1\\
\midrule[\heavyrulewidth]
& \multicolumn{3}{c}{$d = 50$} & \multicolumn{3}{c}{$d = 100$} \\
\cmidrule(r){2-4} \cmidrule(l){5-7}
& SHD & nnz & time (sec) & SHD & nnz & time (sec) \\
\midrule
NOTEARS&20.49$\pm$1.39&102.65$\pm$1.26&156.3$\pm$3.8&30.60$\pm$1.89&202.11$\pm$1.60&660.5$\pm$10.3\\
NOTEARS-KKTS&11.38$\pm$0.76&100.56$\pm$1.09&171.1$\pm$3.8&16.17$\pm$1.29&200.32$\pm$1.52&771.9$\pm$10.4\\
NOTEARS&21.26$\pm$1.29&101.84$\pm$1.19&52.3$\pm$1.7&30.92$\pm$1.89&201.33$\pm$1.63&265.0$\pm$9.0\\
NOTEARS-KKTS&13.37$\pm$0.95&101.33$\pm$1.15&66.2$\pm$1.7&16.47$\pm$1.27&200.17$\pm$1.50&375.2$\pm$9.2\\
Abs&33.02$\pm$1.93&109.21$\pm$1.56&28.9$\pm$1.8&55.70$\pm$2.25&214.18$\pm$1.93&142.7$\pm$7.7\\
Abs-KKTS&17.00$\pm$1.14&102.61$\pm$1.20&43.2$\pm$1.9&25.22$\pm$1.55&203.79$\pm$1.71&253.4$\pm$7.7\\
FGS&76.30$\pm$2.29&121.86$\pm$2.14&1.0$\pm$0.0&108.31$\pm$3.42&246.38$\pm$3.90&2.0$\pm$0.0\\
FGS-KKTS&46.87$\pm$2.65&113.11$\pm$1.74&13.4$\pm$0.3&57.28$\pm$3.40&214.57$\pm$2.30&104.5$\pm$1.9\\
MMHC&78.41$\pm$1.34&44.90$\pm$0.49&1.6$\pm$0.0&141.90$\pm$1.91&106.56$\pm$0.73&8.4$\pm$0.1\\
MMHC-KKTS&69.17$\pm$3.13&118.96$\pm$1.79&9.0$\pm$0.1&118.61$\pm$4.77&238.61$\pm$3.20&84.0$\pm$1.4\\
PC&93.50$\pm$1.33&66.67$\pm$0.64&2.5$\pm$0.0&163.85$\pm$2.18&140.41$\pm$1.03&4.3$\pm$0.1\\
PC-KKTS&74.18$\pm$3.03&120.62$\pm$1.94&9.8$\pm$0.1&124.97$\pm$5.04&239.79$\pm$3.10&63.0$\pm$0.7\\
Search&65.70$\pm$2.02&94.82$\pm$1.07&47.1$\pm$0.5&111.35$\pm$2.53&190.97$\pm$1.42&700.1$\pm$4.3\\
\bottomrule
\end{tabular}
\end{table}

\begin{table}[ht]
\caption{Results (mean $\pm$ standard error over 100 trials) on SF4 graphs with exponential noise, $n = 2d$.}
\label{tab:SF4_exp_n_2d}
\centering
\small
\begin{tabular}{lrrrrrr}
\toprule
& \multicolumn{3}{c}{$d = 10$} & \multicolumn{3}{c}{$d = 30$} \\
\cmidrule(r){2-4} \cmidrule(l){5-7}
& SHD & nnz & time (sec) & SHD & nnz & time (sec) \\
\midrule
NOTEARS&7.02$\pm$0.35&15.39$\pm$0.27&1.3$\pm$0.1&7.63$\pm$0.37&50.74$\pm$0.49&17.5$\pm$0.3\\
NOTEARS-KKTS&6.85$\pm$0.34&15.54$\pm$0.26&1.5$\pm$0.1&7.27$\pm$0.34&50.44$\pm$0.45&20.1$\pm$0.4\\
NOTEARS&6.95$\pm$0.35&15.26$\pm$0.27&0.8$\pm$0.0&7.50$\pm$0.36&50.48$\pm$0.49&4.8$\pm$0.1\\
NOTEARS-KKTS&6.82$\pm$0.34&15.53$\pm$0.26&0.9$\pm$0.0&7.25$\pm$0.34&50.30$\pm$0.44&7.4$\pm$0.2\\
Abs&6.30$\pm$0.32&15.18$\pm$0.26&0.3$\pm$0.0&9.37$\pm$0.52&50.30$\pm$0.57&2.3$\pm$0.1\\
Abs-KKTS&6.70$\pm$0.34&15.49$\pm$0.26&0.5$\pm$0.0&7.83$\pm$0.41&50.14$\pm$0.45&4.9$\pm$0.1\\
FGS&17.57$\pm$0.55&29.69$\pm$0.70&0.3$\pm$0.0&37.53$\pm$0.55&28.26$\pm$0.56&0.5$\pm$0.0\\
FGS-KKTS&7.11$\pm$0.34&15.13$\pm$0.27&0.4$\pm$0.0&14.93$\pm$0.89&50.26$\pm$0.58&3.2$\pm$0.0\\
MMHC&12.49$\pm$0.22&4.16$\pm$0.13&0.4$\pm$0.0&43.18$\pm$0.42&15.28$\pm$0.25&0.5$\pm$0.0\\
MMHC-KKTS&7.67$\pm$0.34&14.64$\pm$0.29&0.5$\pm$0.0&18.58$\pm$1.19&50.29$\pm$0.71&2.1$\pm$0.0\\
PC&13.96$\pm$0.16&5.82$\pm$0.21&2.2$\pm$0.0&47.17$\pm$0.44&24.56$\pm$0.41&2.2$\pm$0.0\\
PC-KKTS&7.19$\pm$0.30&14.74$\pm$0.25&2.3$\pm$0.0&17.10$\pm$0.90&51.20$\pm$0.54&3.8$\pm$0.0\\
Search&7.95$\pm$0.35&13.95$\pm$0.27&0.3$\pm$0.0&19.55$\pm$1.09&47.54$\pm$0.53&4.8$\pm$0.1\\
\midrule[\heavyrulewidth]
& \multicolumn{3}{c}{$d = 50$} & \multicolumn{3}{c}{$d = 100$} \\
\cmidrule(r){2-4} \cmidrule(l){5-7}
& SHD & nnz & time (sec) & SHD & nnz & time (sec) \\
\midrule
NOTEARS&5.94$\pm$0.29&86.62$\pm$0.48&49.7$\pm$0.9&6.91$\pm$0.65&182.18$\pm$0.65&203.0$\pm$3.6\\
NOTEARS-KKTS&5.98$\pm$0.33&86.69$\pm$0.49&60.1$\pm$0.9&6.12$\pm$0.61&182.02$\pm$0.64&281.6$\pm$4.0\\
NOTEARS&6.04$\pm$0.29&86.52$\pm$0.48&12.5$\pm$0.5&6.97$\pm$0.65&182.15$\pm$0.65&73.7$\pm$2.7\\
NOTEARS-KKTS&5.94$\pm$0.33&86.56$\pm$0.49&22.7$\pm$0.6&6.15$\pm$0.61&181.99$\pm$0.64&150.5$\pm$3.1\\
Abs&13.46$\pm$1.33&87.19$\pm$0.81&7.0$\pm$0.5&19.81$\pm$1.44&182.82$\pm$1.03&48.4$\pm$3.4\\
Abs-KKTS&9.41$\pm$0.86&86.80$\pm$0.60&17.2$\pm$0.5&11.04$\pm$1.12&181.49$\pm$0.69&123.3$\pm$3.7\\
FGS&63.93$\pm$0.80&50.98$\pm$0.90&0.6$\pm$0.0&121.68$\pm$1.82&119.36$\pm$1.80&1.0$\pm$0.0\\
FGS-KKTS&22.54$\pm$1.70&86.05$\pm$0.82&10.9$\pm$0.1&50.50$\pm$3.58&180.72$\pm$1.70&75.4$\pm$1.0\\
MMHC&73.67$\pm$0.56&29.46$\pm$0.34&1.5$\pm$0.0&145.51$\pm$1.12&74.83$\pm$0.62&5.4$\pm$0.0\\
MMHC-KKTS&31.46$\pm$2.11&86.52$\pm$0.89&7.3$\pm$0.1&75.14$\pm$4.37&181.41$\pm$1.49&55.4$\pm$0.6\\
PC&77.16$\pm$0.57&45.53$\pm$0.53&2.3$\pm$0.0&146.92$\pm$1.05&100.22$\pm$0.71&3.4$\pm$0.1\\
PC-KKTS&35.87$\pm$2.20&88.69$\pm$0.85&8.2$\pm$0.1&66.27$\pm$4.40&177.25$\pm$1.60&43.5$\pm$0.5\\
Search&28.94$\pm$1.77&81.59$\pm$0.67&23.5$\pm$0.3&57.61$\pm$3.44&173.72$\pm$1.00&253.1$\pm$2.6\\
\bottomrule
\end{tabular}
\end{table}

\end{document}